\documentclass[twoside]{article}

\usepackage[preprint]{aistats2026} 

\usepackage{amssymb,amsmath,amsthm}
\usepackage{algorithm}
\usepackage{algorithmic}

\newtheorem{theorem}{Theorem}

\newtheorem{definition}{Definition}
\newtheorem{lemma}{Lemma}
\newtheorem{remark}{Remark}
\newtheorem{example}{Example}
\newtheorem{corollary}{Corollary}

\usepackage{accents}
\newcommand{\dbtilde}[1]{\accentset{\approx}{#1}}

\usepackage{hyperref}
\usepackage{graphicx}
\usepackage{dsfont}
\graphicspath{ {./images/} }
\newif\ifTR
\TRfalse
\begin{document}

%

%

\twocolumn[

\aistatstitle{On the Intrinsic Dimensions of Data in Kernel Learning}

\aistatsauthor{ Rustem Takhanov }

\aistatsaddress{ Nazarbayev University } ]

\begin{abstract}
The manifold hypothesis suggests that the generalization performance of machine learning methods improves significantly when the intrinsic dimension of the input distribution's support is low. In the context of Kernel Ridge Regression (KRR), we investigate two alternative notions of intrinsic dimension. The first, denoted $d_\varrho$, is the upper Minkowski dimension defined with respect to the canonical metric induced by a kernel function $K$ on a domain $\Omega$. The second, denoted $d_K$, is the effective dimension, derived from the decay rate of Kolmogorov $n$-widths associated with $K$ on $\Omega$.
Given a probability measure $\mu$ on $\Omega$, we analyze the relationship between these $n$-widths and eigenvalues of the integral operator $\phi \mapsto \int_\Omega K(\cdot,x)\phi(x)\,d\mu(x)$. We show that, for a fixed domain $\Omega$, the Kolmogorov $n$-widths characterize the worst-case eigenvalue decay across all probability measures $\mu$ supported on $\Omega$. These eigenvalues are central to understanding the generalization behavior of constrained KRR, enabling us to derive an excess error bound of order
$\mathcal{O}(n^{-\frac{2+d_K}{2+2d_K} + \varepsilon})$
for any $\varepsilon > 0$, when the training set size $n$ is large.
We also propose an algorithm that estimates upper bounds on the $n$-widths using only a finite sample from $\mu$. For distributions close to uniform, we prove that $\varepsilon$-accurate upper bounds on all $n$-widths can be computed with high probability using at most
$\mathcal{O}\left(\varepsilon^{-d_\varrho}\log\frac{1}{\varepsilon}\right)$
samples, with fewer required for small $n$.

Finally, we compute the effective dimension $d_K$ for various fractal sets and present additional numerical experiments. Our results show that, for kernels such as the Laplace kernel, the effective dimension $d_K$ can be significantly smaller than the Minkowski dimension $d_\varrho$, even though $d_K = d_\varrho$ provably holds on regular domains.
\end{abstract}

\section{Introduction}
The manifold hypothesis posits that, within high-dimensional spaces, probability distributions tend to concentrate near lower-dimensional structures, suggesting that data primarily lies along smooth manifolds whose intrinsic dimension is significantly smaller than that of the ambient space. It is highly desirable for the sample complexity of a machine learning (ML) algorithm to depend on the intrinsic dimension of the data, rather than the ambient dimension of the space in which the data resides. When sample complexity scales with the ambient dimension instead, the algorithm suffers from the well-known ``curse of dimensionality''. Fortunately, many widely used algorithms avoid the curse of dimensionality. Notable examples include kernel regression~\cite{e24a4c46} and $k$-nearest neighbors ($k$-NN) regression~\cite{10.5555/2986459.2986541}.

A general theory of machine learning algorithms that adapt to intrinsic dimension has yet to be developed. A natural move in this direction is to establish generalization bounds for kernel methods that explicitly reflect the intrinsic dimension of the data. Some results of this kind rely explicitly on the assumption that the data lies on a low-dimensional Riemannian manifold and make full use of the manifold’s geometric structure~\cite{NIPS2009_2ac2406e,NEURIPS2020_977f8b33,10.3150/23-BEJ1685}.

In this paper, we study excess risk bounds for Kernel Ridge Regression. Unlike the previously cited works, we do not impose strong regularity assumptions on the underlying low-dimensional structures. Instead, we employ concepts from geometric measure theory to characterize the support of the data distribution. This approach enables us to handle less regular supports --- those that are not only theoretically plausible but also observed in real-world datasets --- such as sets with fractional Hausdorff dimension (e.g., fractals). The main contribution of this paper is the identification of two alternative notions of intrinsic dimension, both defined solely in terms of the support of the data distribution, independent of the distribution itself. The first notion is the upper Minkowski dimension, defined with respect to an appropriate metric. The second is based on the behavior of Kolmogorov $n$-widths of the support, viewed as a subset of the corresponding reproducing kernel Hilbert space (RKHS). The first definition is proportional to the Hausdorff dimension of the support and does not take into account the global structure of the kernel. In contrast, the second notion --- appearing in our excess risk bounds --- captures the interaction between the kernel and the support. We argue that this parameter governs the generalization behavior of constrained KRR.

The paper is {\em organized} as follows. Section~\ref{dim-intro} introduces the reproducing kernel Hilbert space (RKHS) and the metric structure induced by the kernel $K$ on the domain $\Omega$, which supports the input distribution $\mu$. We then define the Kolmogorov $n$-widths of $\Omega$ under its embedding into the RKHS, following conventional definitions.
This framework allows us to define two notions of intrinsic dimension for $\Omega$: the upper Minkowski dimension, denoted $d_\varrho$, and the effective dimension, denoted $d_K$. We prove that $d_K \leq d_\varrho$, and, through illustrative examples, justify particular attention to the case where equality holds, i.e., $d_K = d_\varrho$. In Section~\ref{eigenvalues}, we show that the Kolmogorov $n$-width characterizes the worst-case decay rate of the $n$th eigenvalue of the associated integral operator on $L_2(\Omega, \mu)$. This result may be seen as a counterpart to the classical Ismagilov's theorem in approximation theory and forms the foundation for our main excess risk bound for constrained KRR, presented in Section~\ref{KRR-dimension}. The bound exhibits the asymptotic rate $\mathcal{O}\big(n^{-\frac{2 + d_K}{2 + 2d_K} + \varepsilon}\big)$ for any $\varepsilon > 0$. 
In the second, empirically focused part of the paper, we present an algorithm for estimating $n$-widths from a sample, which may be either drawn from the distribution $\mu$ or deterministically constructed. In Section~\ref{upper-bound-algo}, we analyze this algorithm and show that, under the so-called $C$-uniformness assumption on $\mu$, it produces $\varepsilon$-accurate upper bounds on the $n$-widths using a sample of size $\mathcal{O}\big(\varepsilon^{-d_\varrho} \log \frac{1}{\varepsilon}\big)$. In Section~\ref{expr} we report the results of numerical experiments with the calculation of $n$-widths and effective dimensions $d_K$ for various kernels and domains. Some of these experiments give rise to new conjectures concerning the gap between $d_K$ and $d_\varrho$ for such kernels as the Laplace kernel, given on fractal sets. Proofs of most statements, additional remarks on the tables, and details of the numerical experiments are provided in the Appendix.

{\bf Related work.} 
It is well known that the leading term in the Shannon entropy of a discretized $n$-dimensional random vector is proportional to a parameter that can be interpreted as the information dimension of its support~\cite{Renyi1959}. Several other classical notions of a distribution's dimension are discussed in standard texts on geometric measure theory, such as~\cite{Mattila_1995}.
The concept of effective dimension in the context of kernel methods for supervised learning is well established. A distribution-dependent definition based on the eigenvalues of the kernel operator was introduced in~\cite{10.5555/2968618.2968677} and further developed in~\cite{10.5555/945365.964294} and~\cite{Caponnetto2007}.
In this paper, we consider two alternative notions of intrinsic dimension. The first is closely related to the definition used by~\cite{doi:10.1137/21M1435690}, who employed it to derive an improved excess risk bound for regularized ERMs over RKHSs.

Our second notion of intrinsic dimension is based on the behavior of Kolmogorov $n$-widths. While $n$-widths have a long history in approximation theory~\cite{Pinkus1985}, their application in machine learning is a relatively recent development~\cite{STEINWART201713}.
Notably, \cite{10.5555/3586589.3586708} and~\cite{Siegel2024} demonstrated that the asymptotic behavior of $n$-widths characterizes the gap in approximation power between two-layer neural networks and random feature models~\cite{10.5555/2981562.2981710}.
In certain settings --- such as the unit sphere $\mathbb{S}^{d-1}$ --- the computation of $n$-widths reduces to the analysis of kernel eigenvalues. For many commonly used kernels, these eigenvalues have been explicitly computed in~\cite{JMLR:v18:14-546} and~\cite{10.5555/3454287.3455442}.


In the second part of the paper, we develop a numerical algorithm for the approximate computation of $n$-widths from a sample drawn over a given domain. This algorithm is intended to serve as a new numerical tool for studying kernels on domains, analogous to the widely used approach of estimating eigenvalues from empirical kernel matrices~\cite{bj/1082665383, NIPS2002_f516dfb8}.

{\bf Notations.} Given $f:\,\mathbb{N}\to\mathbb{R}$ and $g:\,\mathbb{N}\to\mathbb{R}_+$ (or, alternatively, given two functions of a small argument $f:\,\mathbb{R}_+\to\mathbb{R}, g:\,\mathbb{R}_+\to\mathbb{R}_+$), we write $f(x)\ll g(x)$  if there exist constants $c_1, c_2\in\mathbb{R}_+$ such that for all natural numbers $x>c_2$ (correspondingly, all positive reals $x < c_2$) we have $|f(x)|\leq c_1 g(x)$.  We write $f\asymp g$  if $f\ll g$ and $g\ll f$. The space of measurable functions $f:\Omega\to {\mathbb R}$ such that $\int_{\Omega}|f|^pd\mu < \infty$ is denoted by $L_p(\Omega, \mu)$.  Accordingly, $\|f\|_{L_p(\Omega, \mu)} = (\int_{\Omega}|f|^pd\mu )^{\frac{1}{p}}$. A $d-1$-sphere is the set ${\mathbb S}^{d-1} = \{{\mathbf x}\in {\mathbb R}^d\mid \|{\mathbf x}\|=1\}$ where $\|{\mathbf x}\|$ denotes the canonical euclidean norm. 

\section{Kernel-based upper metric and effective dimensions}\label{dim-intro}
Let $\Omega\subseteq {\mathbb R}^d$ be a compact set and $K: \Omega\times \Omega\to {\mathbb R}$ a continuous positive-semidefinite kernel. Let $\mathcal{H}_K$ denote the reproducing kernel Hilbert space (RKHS) associated with $K$. The Mercer kernel $K$ induces a canonical pseudometric on $\Omega$ defined by
\begin{equation}\label{metric}
\begin{split}
&\varrho(x,y)=\|K(x,\cdot)-K(y,\cdot)\|_{\mathcal{H}_K} = \\
&\sqrt{K(x,x)+K(y,y)-2K(x,y)}.
\end{split}
\end{equation}
This pseudometric, sometimes referred to as the ``canonical metric'', is widely used in kernel-based analysis~\cite{Adler2007}. 

Let $B(z,r)$ denote the closed ball of radius $r$ centered at $z$ in $(\Omega, \varrho)$. By $\mathcal{N}(\varepsilon,\Omega,\varrho)$ we denote the $\varepsilon$-covering number of $(\Omega, \varrho)$, i.e.
\begin{equation}
\begin{split}
&\mathcal{N}(\varepsilon,\Omega,\varrho) = \min\big\{M\in {\mathbb N}\mid \\
&\exists z_1, \cdots, z_M\in \Omega\,\,{\rm s.t.} \,\,
\bigcup_{i=1}^M B(z_i,\varepsilon)\supseteq \Omega\big\}.
\end{split}
\end{equation}
It was shown in~\cite{takhanov2024nonasymptotic} that the Dudley-type integral
$\int_0^\infty \sqrt{\log \mathcal{N}(\varepsilon,\Omega,\varrho)}d\varepsilon$
plays a central role in bounding the generalization error of kernel methods.
The function inverse to $\mathcal{N}(\varepsilon,\Omega,\varrho)$ is denoted by $\varepsilon(n)$, i.e. $\varepsilon(n) = \min\big\{\varepsilon>0\mid \mathcal{N}(\varepsilon,\Omega,\varrho)\leq n\big\}$. 

Then, 
\begin{equation}\label{minkowski}
\begin{split}
&d_\varrho = \limsup_{\varepsilon\to +0}\frac{\log \mathcal{N}(\varepsilon,\Omega,\varrho)}{\log \frac{1}{\varepsilon}}=\\
&\limsup_{n\to+\infty}\frac{\log n}{\log (1/\varepsilon(n))}
\end{split}
\end{equation}
is called the {\em kernel-based upper metric dimension} of $\Omega$~\cite{Mattila_1995}. For many kernels of interest (e.g., translation-invariant kernels, polynomial kernels, and neural tangent kernels), one typically has $\varrho(x,y)= \|x-y\|^a(b+o(1)), a,b>0$ as $\|x-y\|\to 0$, implying $d_\varrho = \frac{d_{H}}{a}$, where $d_{H}$ is the upper metric dimension of $\Omega$ equipped with the euclidean metric. If $\Omega$ is a $k$-dimensional smooth manifold embedded in ${\mathbb R}^d$, then $d_{H}=k$. For less regular sets, such as fractals, $d_H$ coincides with the Hausdorff dimension for all interesting domains $\Omega$. For readers unfamiliar with these notions, we refer to the popular book by~\cite{peitgen2004chaos}, which contains numerous illustrations clarifying different definitions of dimension.

Since the map $x\in \Omega\mapsto K(x,\cdot)\in \mathcal{H}_K$ defines an embedding of $\Omega$ into $\mathcal{H}_K$, it is natural to study the image of $\Omega$ under this mapping as a subset of $\mathcal{H}_K$. The Kolmogorov $n$-width of this image is defined by 
\begin{equation}\label{n-width-def}
\begin{split}
w_K(n)=\inf_{L_n}\sup_{x\in \Omega}\inf_{f\in L_n}\|K(x,\cdot)-f\|_{\mathcal{H}_K},
\end{split}
\end{equation}
where the infimum is taken over all $n$-dimensional subspaces $L_n$ of $\mathcal{H}_K$~\cite{2ad9b75f-e5d2-3a8c-8211-9d37aa110eb2}.

We will call the value
\begin{equation}\label{kolm-dim}
\begin{split}
d_{K} = \limsup_{n\to+\infty}\frac{\log n}{\log (1/w_K(n))}
\end{split}
\end{equation}
the {\em kernel-based effective dimension} of $\Omega$. 

In summary, the metric dimension defined by the kernel-induced pseudometric (cf. Equation~\eqref{minkowski}) depends on both the domain $\Omega$ and the local (typically linear) behavior of the kernel --- e.g., its Taylor expansion in the translation-invariant case. The dimension~\eqref{kolm-dim} depends on the domain, and also, on the global structure of the kernel.  
We now state a basic but useful result:
\begin{theorem} \label{entropy-number} We have $d_K\leq d_\varrho$.
\end{theorem}
\begin{proof}
It is sufficient to prove that $$w_K(n)\leq \varepsilon(n),$$ 
for any $n\in {\mathbb N}$. Indeed, let $\varepsilon=\varepsilon(n)$ and $z_1, \cdots, z_n\in \Omega$ be such that $\bigcup_{i=1}^n B(z_i,\varepsilon)\supseteq \Omega$. 
Given $x\in \Omega$, we define $i(x)\in\arg\min\limits_{i:1\leq i\leq n} \varrho(x,z_i)$ and set $c(i(x),x) = 1$, and $c(j,x) = 0$ if $j\ne i(x)$. Let $L_n$ be the span of $\{K(z_i,\cdot)\}_1^n$. Then,
\begin{equation*}
\begin{split}
&w_K(n)\leq \sup_{x\in \Omega}\min_{f\in L_n}\|K(x,\cdot)-f\|_{\mathcal{H}_K}\leq \\
&\sup_{x\in \Omega}\|K(x,\cdot)-\sum_{i=1}^n c(i,x) K(z_i, \cdot)\|_{\mathcal{H}_K}\leq \varepsilon.
\end{split}
\end{equation*}
Theorem proved.
\end{proof}

To illustrate these definitions, let us estimate $n$-widths for certain kernels on regular domains and exactly calculate both $d_\varrho$ and $d_K$.
In particular, we identify cases where the metric dimension and the effective dimension coincide. This situation is quite typical, as the next theorem demonstrates for several important kernels, including the Laplace kernel, when the domain $\Omega$ is sufficiently regular. 
\begin{theorem}\label{width-cases} In the following cases we have $d_K=d_\varrho$:
\begin{itemize}
\item $K(x,y)=e^{-\gamma \|x-y\|^a}$, $a\in (0,1]$ and $\Omega\subseteq {\mathbb R}^d$ is Riemann measurable with non-zero volume;
\item Mat{\'e}rn kernel: $K(x,y)=k_{\rm M}(\|x-y\|)$ where $k_{\rm M}(r)  = \frac{2^{1-\nu}}{\Gamma(\nu)}(\frac{\sqrt{2\nu}r}{l})^\nu K_\nu (\frac{\sqrt{2\nu}r}{l})$, $\nu\in (0,1),l>0$
and $K_\nu$ is a modified Bessel function, $\Omega\subseteq {\mathbb R}^d$ is Riemann measurable with non-zero volume;
\item $K(x,y)=e^{-\gamma \|x-y\|^a}$, $a\in (0,1]$ and $\Omega={\mathbb S}^{d-1}$.
\end{itemize}
\end{theorem}
One of the key empirical findings of the experimental section (see Section~\ref{expr}) is that even for the mentioned kernels, the relation $d_K < d_\varrho$ often holds when the domain $\Omega$ exhibits fractal structure. A possible explanation for this phenomenon can be inferred from Table~\ref{rates}, which compares the kernel-based upper metric and effective dimensions for several kernels on ${\mathbb S}^{d-1}$.  

As the table shows, for the Gaussian kernel we have $d_K=0$. This reflects the fact that Kolmogorov $n$-widths decay more rapidly for smoother kernels. In this context, kernels for which $d_K=d_\varrho$ can be interpreted as the {\em least smooth}, since they yield the slowest possible decay rates of $n$-widths for a given domain geometry ($w_K(n)\asymp n^{-\frac{1}{d_{\varrho}}}$). The observed gap between $d_K$ and $d_\varrho$ for highly fractional domains suggests that kernel-based methods (even with least smooth kernels) treat fractals in such a way as if they were effectively lower-dimensional than their true metric complexity. In other words, the effective dimension captures a tradeoff between the smoothness of the kernel and the regularity of the domain.

\begin{table*}
\begin{center}
\begin{tabular}{ |p{4cm}|p{2cm}|p{1.5cm}|p{1.5cm}|p{3.5cm}|}
\hline
$K$ & $w_K(n)$ & $d_{\varrho}$ & $d_{K}$ &  Using the source\\
\hline
$e^{-\gamma\|x-y\|}$ &  $\asymp n^{-\frac{1}{2d-2}}$ & $2d-2$ & $2d-2$ & \cite{NEURIPS2020_1006ff12}\\
\hline
${\rm NTK}_{\max(0,x)}$ &  $\asymp n^{-\frac{1}{2d-2}}$ & $2d-2$ & $2d-2$ & \cite{NEURIPS2020_1006ff12}\\
\hline
$e^{-\frac{\|{\mathbf x}-{\mathbf y}\|^2}{\sigma^2}},\sigma>\sqrt{\frac{2}{d}}$ &  $\searrow$ exp. fast & $d-1$ & $0$ & \cite{10.1007/11776420_14} \\
\hline
${\rm NNGP}_{\cos}, {\rm NNGP}_{\sin}$ &  $\ll n^{-\frac{1}{2}}$ & $d-1$ & $\leq 2$ & \cite{10.5555/3586589.3586708}\\
\hline
${\rm NNGP}_{\max(0,x)^\alpha}, \alpha\geq 0$ &  $\gg n^{-\frac{1+2\alpha}{2d-2}}$ & $\frac{2d-2}{1+\alpha}(?)$ & $\geq \frac{2d-2}{1+2\alpha}$ & \cite{10.5555/3586589.3586708}\\
\hline
${\rm NNGP}_{\max(0,x)^\alpha}, \alpha\in \{0,1\}$ &  $\asymp n^{-\frac{1+2\alpha}{2d-2}}$ & $\frac{2d-2}{1+\alpha}$ & $\frac{2d-2}{1+2\alpha}$ & \cite{JMLR:v18:14-546}\\
\hline
\end{tabular}
\end{center} 
\caption{\small The Kolmogorov $n$-widths for various $K$ on $\Omega={\mathbb S}^{d-1}$}\label{rates}
\end{table*}

\section{$n$-widths and eigenvalues}\label{eigenvalues}
Recall that a kernel $K$ and a Borel probability measure $\mu$ on a compact domain $\Omega$ together define a compact integral operator ${\rm O}_{K,\mu}: L_2(\Omega, \mu)\to L_2(\Omega, \mu)$  given by ${\rm O}_{K,\mu}\phi (x)=\int_\Omega K(x,y)\phi(y)d\mu(y)$. Let $\lambda_1({\rm O}_{K,\mu})\geq \lambda_2({\rm O}_{K,\mu})\geq \cdots$ denote its eigenvalues, counted with multiplicities and ordered in decreasing order. These eigenvalues play a central role in understanding the generalization ability of kernel methods~\cite{Cucker2001OnTM,945262}. This makes it natural to analyze the relationship between eigenvalues and $n$-widths. The key intuition is as follows: given the fixed domain $\Omega$, the behaviour of $n$-widths captures the information on the worst case behaviour of eigenvalues. Our main result formalizing this relationship is stated below.
\begin{theorem}\label{from-isma} For any probabilistic Borel measure $\mu$ on $\Omega$, we have $$\lambda_{2n}({\rm O}_{K,\mu})\leq \frac{w_{K}(n)^2}{n}.$$
Moreover, 
$$\limsup\limits_{n\to+\infty}\sup\limits_{\mu} \frac{n\lambda_n({\rm O}_{K,\mu})}{w_K(n)^2}\geq \frac{1}{e},$$ where the supremum is taken over all Borel probability measures $\mu$ supported on $\Omega$.
\end{theorem}
\begin{remark} This theorem addresses the question: if the domain $\Omega$ is fixed but the distribution $\mu$ varies, how poorly can the eigenvalues of ${\rm O}_{K,\mu}$ decay? The first inequality shows that the Kolmogorov $n$-widths provide an upper bound for the eigenvalue decay for any $\mu$. The second inequality confirms that this bound is essentially optimal (up to a constant) and cannot be improved uniformly over all $\mu$, under the condition that domain is fixed and $w_K(n)\asymp w_K(2n)$. It highlights that the Kolmogorov $n$-widths reflect the worst-case spectral behavior induced by arbitrary distributions.
\end{remark}

The proof of Theorem~\ref{from-isma} proceeds in two steps: establishing the upper and lower bounds.
A major tool in the first part of the proof is the following remarkable result known as Ismagilov's theorem. 
\begin{theorem}[\cite{Ismagilov1968}] Let $\mu$ be a probabilistic Borel, nondegenerate measure on $\Omega$. Let $\lambda_1\geq \lambda_2\geq \cdots$ be  positive eigenvalues of ${\rm O}_{K,\mu}$ (counting multiplicities) with corresponding orthogonal unit eigenvectors $\{\psi_i\}_{i=1}^\infty$. Then, 
\begin{equation*}
\begin{split}
\sqrt{\sum_{i=n+1}^\infty \lambda_i}\leq w_K(n)\leq \sup\limits_{x\in \Omega}\sqrt{\sum_{i=n+1}^\infty \lambda_i\psi_i(x)^2}.
\end{split}
\end{equation*}
\end{theorem}

If the measure $\mu$ in Ismagilov's theorem is chosen appropriately --- typically as the uniform distribution on $\Omega$ --- and the corresponding eigenfunctions of the integral operator ${\rm O}_{K,\mu}$ grow moderately (e.g., satisfy a bound of the form $\sup_{n\in {\mathbb N}, x\in \Omega}\frac{1}{n}\sum_{i=1}^n \psi_i(x)^2<+\infty$), then we obtain sharp estimates of the form $w_K(n)\asymp \sqrt{\sum_{i=n+1}^\infty \lambda_i}$. In practice, Ismagilov's theorem is most effective when applied with such carefully selected measures $\mu$, which yield well-behaved eigenfunctions and allow accurate estimation of Kolmogorov widths via spectral data.

\begin{proof}[Proof of the first part of Theorem~\ref{from-isma}] Let us first assume that $\mu$ is non-degenerate on $\Omega$, meaning its support equals $\Omega$.
From Ismagilov's theorem we obtain
\begin{equation*}
\begin{split}
n\lambda_{2n}({\rm O}_{K,\mu})\leq \sum_{i=n+1}^\infty \lambda_i({\rm O}_{K,\mu})\leq  w_{K}(n)^2.
\end{split}
\end{equation*}
Thus, $\lambda_{2n}({\rm O}_{K,\mu})\leq \frac{w_{K}(n)^2}{n}$. Any probabilistic Borel measure $\mu$ can be approximated by a non-degenerate $(1-\varepsilon)\mu+\varepsilon \nu$, where $\nu$ is a fixed non-degenerate probabilistic Borel measure on $\Omega$. Using Weyl's inequality we have $\lambda_{2n}({\rm O}_{K,(1-\varepsilon)\mu+\varepsilon \nu})\overset{\varepsilon\to +0}{\to} \lambda_{2n}({\rm O}_{K,\mu})$ and, therefore, $\lambda_{2n}({\rm O}_{K,\mu})\leq \frac{w_{K}(n)^2}{n}$ holds for any probabilistic Borel measure $\mu$. This concludes the proof of the first part of the theorem.
\end{proof}

The proof of the second part of Theorem~\ref{from-isma} is technical and provided in the Appendix~\ref{lower-ismagilov}. Its central idea is to construct a sequence of points $x_1, x_2, \dots \in \Omega$ such that the corresponding sequence of empirical measures $\mu_m := \frac{1}{m} \sum_{i=1}^m \delta_{x_i}, \quad m \in \mathbb{N}$, induces integral operators ${\rm O}_{K,\mu_m}$ whose eigenvalues exhibit the worst-case decay behavior relative to the Kolmogorov $n$-widths. Here, $\delta_x$ denotes the probabilistic measure concentrated at point $x$. Specifically, we prove that
$\limsup_{n \to \infty} \sup_{m \in \mathbb{N}} \frac{n \lambda_n({\rm O}_{K,\mu_m})}{w_K(n)^2} \geq \frac{1}{e}$.

The construction of the sequence $\{x_i\}_{i=1}^\infty \subset \Omega$ is not only essential to the proof but also serves as the foundation for Algorithm~\ref{empirical-width}, which we present in Section~\ref{upper-bound-algo}. That algorithm provides a practical method for approximating upper bounds on Kolmogorov $n$-widths. Below, we briefly describe the core idea behind this construction.  

Given $X_n = (x_1, \cdots, x_n)$ and $Y_m = (y_1, \cdots, y_m)$, let $K[X_n,Y_m]$ denote the matrix $[K(x_i,y_j)]_{i=1}^{n}{}_{j=1}^m\in {\mathbb R}^{n\times m}$.  For any $n\in {\mathbb N}$ and any $X_n = (x_1, \cdots, x_n)\in \Omega^n$, let us denote $f_n: \Omega^n\to \Omega$ by
\begin{equation}
\begin{split}
f_n(X_n) = \arg\max_{x\in \Omega} {\rm det}(K[(X_n,x),(X_n,x)]).
\end{split}
\end{equation}
Given any $x_1\in \arg\max_{x\in\Omega} K(x,x)$, the collection of functions $\{f_n\}$ induces a sequence $\{x_i\}_1^\infty\subseteq \Omega$ defined by the rule
\begin{equation}\label{sequence}
\begin{split}
x_{n+1} = f_n(x_1, \cdots, x_n).
\end{split}
\end{equation}

Recall that the Schur complement $(M\mid A)$ of the block matrix $M = \begin{bmatrix} A & B \\ C & D \end{bmatrix}$, where $A, B,C,D$ are matrices with dimensions $n\times n, n\times m, m\times n, m\times m$ respectively, is defined as the $D - CA^{-1}B$. A key property of the Schur complement is that ${\rm det}(M) = {\rm det}(A) {\rm det}((M|A))$ if ${\rm det}(A)\ne 0$~\cite{Horn2005}. The key lemma for the proof is formulated below.

\begin{lemma}\label{schur} For any $n\in {\mathbb N}$ and any $X_n \in \Omega^n$ such that ${\rm det}(K[X_n,X_n])>0$, $X_{n+1} = (X_n,f_n(X_n))$ satisfies
\begin{equation*}
\begin{split}
(K[X_{n+1},X_{n+1}]\mid K[X_n,X_n])\geq w_K(n)^2.
\end{split}
\end{equation*}
\end{lemma}

\section{An application to ERM}\label{KRR-dimension}
Let $\mathcal{Z} = (\Omega\times {\mathbb R}, \mathcal{B}(\Omega\times {\mathbb R}), P)$ be a probability space, where $\mathcal{B}(\Omega\times {\mathbb R})$ is a family of Borel subsets of $\Omega\times {\mathbb R}$. We denote the probability function over inputs by $\mu$, i.e. $\mu(B) = P(B\times {\mathbb R})$. Let $\mathcal{Z}^{n}$ denote the Cartesian product of $n$ copies of $\mathcal{Z}$ and $\mathcal{T}=\{(X_{i}, Y_{i})\}_{i=1}^n$ denote a random variable distributed according to $\mathcal{Z}^{n}$. 
The hypothesis class is defined by $\mathcal{F}=B_{\mathcal{H}_K}$ where $B_{\mathcal{H}_K}$ is a unit ball of $\mathcal{H}_K$ centered at origin. We are also given a loss function $l: {\mathbb R}\times {\mathbb R}\to {\mathbb R}_+$. We consider the empirical risk minimization (ERM) algorithm for regression, i.e. the algorithm that, given $\mathcal{T}$, selects a function $\hat{f}\in B_{\mathcal{H}_K}$ such that
\begin{equation*}
\begin{split}
\hat{f}=\arg\min\limits_{f\in B_{\mathcal{H}_K}} \frac{1}{n}\sum_{i=1}^n l(f(X_i), Y_i).
\end{split}
\end{equation*}
The optimal regression function in $B_{\mathcal{H}_K}$, minimizing the true risk, is denoted by $f^\ast$, i.e. $f^\ast  =\arg\min\limits_{f\in B_{\mathcal{H}_K}} {\mathbb E}[l(f(X),Y)]$.  

Following the framework of~\cite{Rademacher}, we assume that the loss function satisfies conditions: (1) $f^\ast$ is well-defined for any probability function $P$; (2) there is $L>0$ such that $|l(y_1,y)-l(y_2,y)| \leq L|y_1-y_2|$; (3) there is $B>0$ such that $B({\mathbb E}[l(f(X), Y)]-{\mathbb E}[l(f^\ast(X), Y)])\geq \|f-f^\ast\|_{L_2(\Omega, \mu)}^2$ for any $f\in B_{\mathcal{H}_K}$. Note that for the square loss case, we have $B=1$ due to convexity of $B_{\mathcal{H}_K}$.
Our main result provides a high-probability bound on the excess risk and is stated below.

\begin{theorem}\label{krr-main} Let $\varepsilon>0$ and $N_\varepsilon = \min\{N\in {\mathbb N}\mid \sup_{i > N} w_K(i)i^{\frac{1}{d_K+\varepsilon}}\leq 1\}$. 
For any $x>0$, with probability at least $1-e^{-x}$ (over randomness in $\mathcal{T}$), we have 
\begin{equation*}
\begin{split}
&{\mathbb E}[l(\hat{f}(X), Y)]-{\mathbb E}[l(f^\ast(X), Y)]\leq \\
&1995 B^{\frac{1}{1+d_K+\varepsilon}}L^{1-\frac{1}{1+d_K+\varepsilon}} n^{-\frac{2+d_K+\varepsilon}{2(1+d_K+\varepsilon)}}+\frac{(11L+27B)x}{n},
\end{split}
\end{equation*}
provided that $n\geq \frac{B^2(3+d_K+\varepsilon)N_\varepsilon^{\frac{2(1+d_K+\varepsilon)}{d_K+\varepsilon}}}{2L^2}$.
\end{theorem}
Note that excess risk rate $\mathcal{O}(\frac{1}{n})$ represents the best possible convergence achievable in well-specified parametric settings or under strong regularity assumptions~\cite{klochkov2021stability}.
For kernel methods, a typical and often minimax-optimal convergence rate is $\mathcal{O}(\frac{1}{\sqrt{n}})$. As the latter theorem shows, $d_K=0$, which we have for, e.g., the Gaussian kernel, leads to the excess risk rate arbitrarily close to $\mathcal{O}(\frac{1}{n})$. The larger the dimension $d_K$ the closer we approach to $\mathcal{O}(\frac{1}{\sqrt{n}})$ rate. For kernels such as the Laplace kernel, we have $d_K=d_\varrho$ on regular domains, and the excess risk convergence behaves like $\mathcal{O}(n^{-\frac{2+d_\varrho}{2+2d_\varrho}+\varepsilon})$ which aligns with the manifold hypothesis.

Analogously, Theorem~\ref{from-isma} allows to apply Kolmogorov $n$-widths to other kernel methods, if their statistical properties are defined by the eigenvalue decay of the associated integral operator. E.g., one could estimate the effective dimension in the sense of ~\cite{Caponnetto2007} or the Signal Capture Threshold~\cite{10.55553495724.3497030}, and then derive implications for unconstrained KRR. 

\section{The estimation of the $n$-width's upper bound}\label{upper-bound-algo}
In the previous section, we demonstrated that the Kolmogorov $n$-width, $w_K(n)$, plays a fundamental role in understanding how the geometry of the domain $\Omega$ influences the generalization ability of kernel ridge regression.
Given the domain $\Omega$ and the kernel $K$, let us now address the problem of the estimation of an upper bound on $w_K(n)$. We begin by describing an algorithm under the assumption that we can efficiently perform function maximization over $\Omega$. Afterwards, we extend the approach to a more practical setting in which we are given only a finite set of points $\tilde{\Omega}\subset \Omega$, either computed directly or sampled from some distribution over $\Omega$.

The Algorithm~\ref{empirical-width} is inspired by the construction of a sequence $\{x_t\}$ defined by the greedy rule~\eqref{sequence}, which arises naturally in the proof of Theorem~\ref{from-isma} (see also Lemma~\ref{schur}). The key idea behind that proof is that, for any $X_t = (x_1, \cdots, x_t)\in \Omega^t$, the quantity $\sup_{x\in \Omega}(K[(X_t,x),(X_t,x)]\mid K[X_t,X_t])$ serves as a valid upper bound on the squared Kolmogorov width $w_K(t)^2$. Based on this, we compute a sequence of values $\{w_t\}_{t=0}^{T-1}$ such that $ w_K(t) \leq w_t$ for $t=0,1,\cdots, T-1$. Note that $\forall x\in\Omega, (K[(X_{t-1},x),(X_{t-1},x)]\mid K[X_{t-1},X_{t-1}])=0$ implies  $\mathcal{H}_K\subseteq {\rm span}(\{K(x_t,\cdot)\}_{t=1}^{t-1})$. Thus, Algorithm~\ref{empirical-width} is well-defined if $T\geq {\rm dim}(\mathcal{H}_K)$, that is we have ${\rm det}(K[X_{t-1},X_{t-1}])\ne 0$ at each step. In all our applications, we have ${\rm dim}(\mathcal{H}_K)=+\infty$.
\begin{algorithm}
\begin{algorithmic}
\caption{The empirical $n$-width upper bound algorithm. Parameter: $T\in {\mathbb N}$}\label{empirical-width}
\STATE \textbf{Input:} $\tilde{\Omega}$. Either $\tilde{\Omega}=\Omega$ (ideal case) or $\tilde{\Omega}\subset \Omega$ is finite (empirical case)
\STATE $x_1\longleftarrow \arg\max_{x\in \tilde{\Omega}} K(x,x)$
\STATE $w_0\longleftarrow K(x_1,x_1)^{1/2}$
\STATE $X_1 = \{x_1\}$
\FOR{$t = 2, \cdots, T$}
\STATE $x_t \longleftarrow \arg\max_{x\in \tilde{\Omega}} K(x,x)-K[X_{t-1},x]^\top K[X_{t-1},X_{t-1}]^{-1}K[X_{t-1},x]$
\STATE $w_{t-1} \longleftarrow  (K(x_t,x_t)-K[X_{t-1},x_t]^\top K[X_{t-1},X_{t-1}]^{-1}K[X_{t-1},x_t])^{1/2}$
\STATE $X_t = X_{t-1}\cup \{x_t\}$
\ENDFOR
\STATE \textbf{Output:} $\{w_t\}_{t=0}^{T-1}$
\end{algorithmic}
\end{algorithm}

We now consider the setting where $\tilde{\Omega}\subset \Omega$ is a finite set, and all maximization steps in Algorithm~\ref{empirical-width} are performed over $\tilde{\Omega}$. We begin with the case in which $\tilde{\Omega}$ is computed deterministically and serves as a finite approximation of the entire domain $\Omega$.
Since the maximization is now restricted to a subset of $\Omega$, we no longer can guarantee $\sup_{x\in \tilde{\Omega}}(K[(X_t,x),(X_t,x)]\mid K[X_t,X_t])\geq w_K(t)^2$ unless the sample $\tilde{\Omega}$  is sufficiently representative of $\Omega$.

Any sequence $\{x_t\}_{t=1}^\infty \subseteq \Omega$ defines a sequence of functions $\{S_t: \Omega\to {\mathbb R}\}_{t=0}^\infty$ by $S_0(x) = K(x,x)$ and $S_t(x) = K(x,x)- K[X_t,x]^\top K[X_t,X_t]^{-1}K[X_t,x]$ if ${\rm det}(K[X_t,X_t])\ne 0$ where $X_t = (x_1, \cdots, x_t)$. If ${\rm det}(K[X_t,X_t])= 0$, we define $S_t(x)=0$. The key observation underlying our analysis of Algorithm~\ref{empirical-width} is the following lemma, which states that functions in $\{\sqrt{S_t}\}$ are all 1-Lipschitz w.r.t. the metric $\varrho$.

\begin{lemma}\label{lipshits} For any $\{x_t\}_{t=1}^\infty \subseteq \Omega$ and any $n\in {\mathbb N}\cup \{0\}$, we have $\sqrt{S_n(x)}-\sqrt{S_n(y)}\leq \varrho(x,y)$.
\end{lemma}

As the following theorem shows, a slightly adjusted output of the empirical Algorithm~\ref{empirical-width} can serve as a valid upper bound on the Kolmogorov $n$-widths, provided that the finite set $\tilde{\Omega}$ is $\varepsilon$-dense in $(\Omega, \varrho)$.
\begin{theorem}\label{deterministic} Let $\{w_t\}_{t=0}^{T-1}$ be the sequence of approximate upper bounds produced by Algorithm~\ref{empirical-width} using the finite set $\tilde{\Omega}=\Omega_1\subset \Omega$. Let us assume that $\Omega_1$ is an $\varepsilon$-net in $(\Omega, \varrho)$. Then, we have
\begin{equation*}
\begin{split}
w_t \geq w_K(t)-\varepsilon, t=0,\cdots, T-1.
\end{split}
\end{equation*}
\end{theorem}
\begin{proof} Let $\{x_t\}_{t=1}^T\subseteq \Omega_1$ be the sequence of elements computed in Algorithm~\ref{empirical-width}, given the input $\tilde{\Omega}=\Omega_1$. This sequence defines the sequence of functions $\{S_t\}_{t=0}^T$ and, by construction of the algorithm, we have $w_{t-1} = \sqrt{S_{t-1}(x_t)} = \sqrt{\max_{x\in \Omega_1} S_{t-1}(x)}$. 

Let $\tilde{x}\in \arg\max_{x\in \Omega}S_{t-1}(x)$. By assumption, there exists $\dbtilde{x}\in \Omega_1$ such that $\varrho(\tilde{x},\dbtilde{x})\leq \varepsilon$. From Lemma~\ref{lipshits} we conclude $\sqrt{S_{t-1}(\dbtilde{x})}\geq \sqrt{S_{t-1}(\tilde{x})}-\varepsilon$, which implies $\sqrt{\max_{x\in \Omega_1}S_{t-1}(x)}\geq \sqrt{\max_{x\in \Omega}S_{t-1}(x)}-\varepsilon$. Therefore,
\begin{equation*}
\begin{split}
&w_{t-1}+\varepsilon\geq \max_{x\in \Omega}\sqrt{S_{t-1}(x)} = \\
&\max_{x\in \Omega} \min_{c_{1:t-1}}\|K(x,\cdot)-\sum_{i=1}^{t-1}c_i K(x_i, \cdot)\|_{\mathcal{H}_K}\geq w_K(t-1).
\end{split}
\end{equation*}
This completes the proof.
\end{proof}

Sampling points from a distribution $\mu$ whose support is the full domain $\Omega$ is another practically important scenario. That is, we assume $\tilde{\Omega} = \{Z_1, \cdots, Z_N\}$ and $Z_1, \cdots, Z_N\sim^{\rm iid} \mu$. 
If $\mu$ is the uniform distribution over $\Omega$, then with high probability the sample $\tilde{\Omega}$ can serve as a good approximation of $\Omega$.
We now introduce a generalization of this natural uniformity assumption.

\begin{definition}\label{c-uniform} Let $C>0$. The measure $\mu$ is called $C$-uniform over $\Omega$ if there is $\varepsilon_0>0$ such that $\mu(B(x,\varepsilon))\geq C\varepsilon^{d_\varrho}$ for any $x\in \Omega$ and $\varepsilon\in (0,\varepsilon_0)$.
\end{definition}
\begin{remark} In geometric measure theory, the quantity 
$$
\Theta_\ast(x,\mu) = \liminf_{\varepsilon\to 0+}(2\varepsilon)^{-d_\varrho}\mu(B(x,\varepsilon))
$$
is known as the lower $d_\varrho$-density of $\mu$ at the point $x$~\cite{Mattila_1995}. This notion generalizes the concept of a probability density function to sets $\Omega$ of potentially fractal nature, and to distributions supported on such sets. In the case of regular domains --- e.g., compact sets with smooth boundaries --- and continuous measures $\mu$, this density behaves proportionally to the standard probability density function (provided that $\varrho(x,y)= \|x-y\|^a(b+o(1)), a,b>0$ as $\|x-y\|\to 0$). 
So, in those cases, the condition in Definition~\ref{c-uniform} is equivalent to requiring a positive lower bound on the pdf of $\mu$.
\end{remark}

The following theorem shows how the difference between the ideal setting ($\tilde{\Omega}=\Omega$) and the empirical setting ($\tilde{\Omega}=\{Z_1, \cdots, Z_N\}$) vanishes as the sample size $N$ increases. To ensure that Algorithm~\ref{empirical-width} is almost surely well-defined, we additionally assume that for $X_1, \cdots, X_n\sim^{\rm iid}\mu$ the probability of the event ${\rm det}([K(X_i,X_j)]_{i,j=1}^n)=0$ is zero.

\begin{theorem}\label{finite-omega} Suppose that the measure $\mu$ is $C$-uniform over $\Omega$ and $Z_1, \cdots, Z_N\sim^{\rm iid} \mu$.  Let $\{w_t\}_{t=0}^{T-1}$ be an output of Algorithm~\ref{empirical-width} for $\tilde{\Omega}=\{Z_1, \cdots, Z_N\}$. For $\delta\in (0,1)$, we have
\begin{equation*}
\begin{split}
{\mathbb P}[w_t \geq 
w_K(t)-2\varepsilon, t=0,\cdots, T-1]\geq 1-\delta,
\end{split}
\end{equation*}
provided that $N\geq \frac{\log \mathcal{N}(\varepsilon,\Omega,\varrho)+\log \frac{1}{\delta}}{C\varepsilon^{d_\varrho}}$.
\end{theorem}

\begin{remark}\label{dim-eps}  As $\varepsilon\to +0$, $\mathcal{N}(\varepsilon,\Omega,\varrho)\ll (\frac{1}{\varepsilon})^{d_\varrho+\zeta}$ for any $\zeta>0$. Thus, we need a sample size of $N=\frac{(d_\varrho+\zeta)\log \frac{1}{\varepsilon}+\log \frac{1}{\delta}+O(1)}{C\varepsilon^{d_\varrho}}$ to have for $t=0,\cdots, T-1$, $w_t \geq 
w_K(t)-2\varepsilon$  with probability at least $1-\delta$. The Appendix~\ref{proofs-smaller-t} contains a proof showing that the condition $T\leq \frac{NC\varepsilon^{d_\varrho}+\log \delta}{\log N}$ is sufficient to ensure that $w_t \geq w_K(t) - \varepsilon$ holds for all $t = 0, \ldots, T - 1$ for a target confidence level $1 - \delta$. This result implies that fewer samples are required for smaller values of $t$.
\end{remark}

In Appendix~\ref{implementation-complexity}, we present two implementations of Algorithm~\ref{empirical-width}: one for the ideal setting ($\tilde{\Omega}=\Omega$), where optimization over the infinite domain $\tilde{\Omega}$ is replaced by a non-convex optimization procedure, and another for the finite-sample case $\tilde{\Omega}=\{Z_1,\dots,Z_N\}$. We experimented with both implementations, and the corresponding results are reported in Appendices~\ref{overestimation}, \ref{analytical-NTK}, and~\ref{finite-width}.

\section{Experiments}
\label{expr}

\begin{table}
\small
\begin{center}
\begin{tabular}{ |p{3cm}|p{1cm}|p{1cm}|}
\hline
Fractal & $d_{\varrho}$ & $d^{\rm emp}_{K}$\\
\hline
Cantor set & 1.2618 & 1.2415\\
\hline
Weierstrass function & 3.0 & 2.7052 \\
\hline
Sierpiński carpet  & 3.7855 & 3.2896 \\
\hline
Menger sponge & 5.4536 & 4.2506 \\
\hline
Lorenz attractor & 4.12 & 3.2839\\
\hline
\end{tabular}
\end{center} 
\caption{\small Comparison of the Laplace kernel-based upper metric dimension and the empirical upper bound on the effective dimension, computed using Algorithm~\ref{empirical-width}, for several well-known fractal sets. The upper metric dimension corresponds to twice the Hausdorff dimension.}\label{empirical}
\end{table}

\begin{figure*}[htb]
\begin{minipage}[t]{.2\textwidth}
    \centering
    \includegraphics[width=1.0\textwidth]{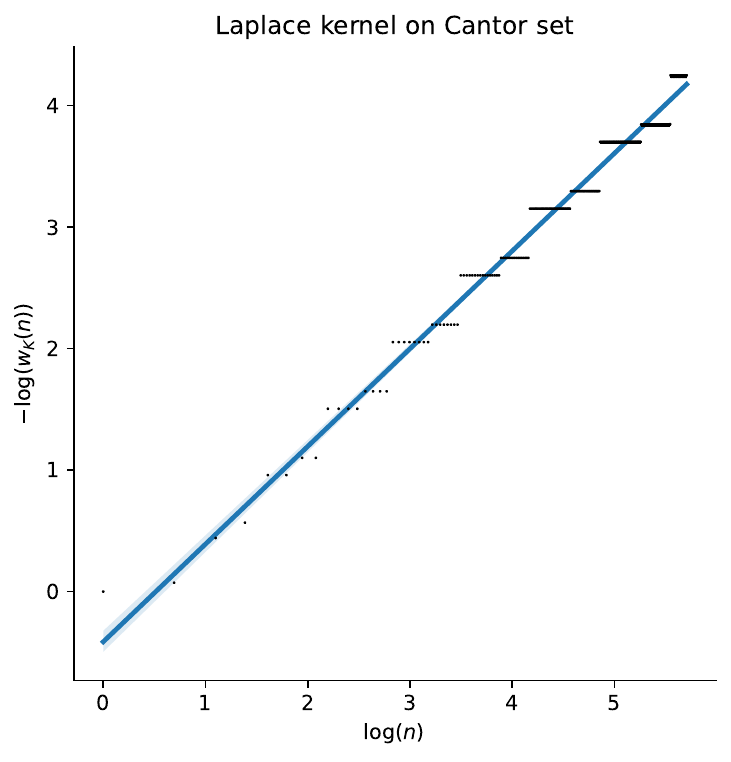}
\end{minipage}\hfill\begin{minipage}[t]{.2\textwidth}
    \centering
    \includegraphics[width=1.0\textwidth]{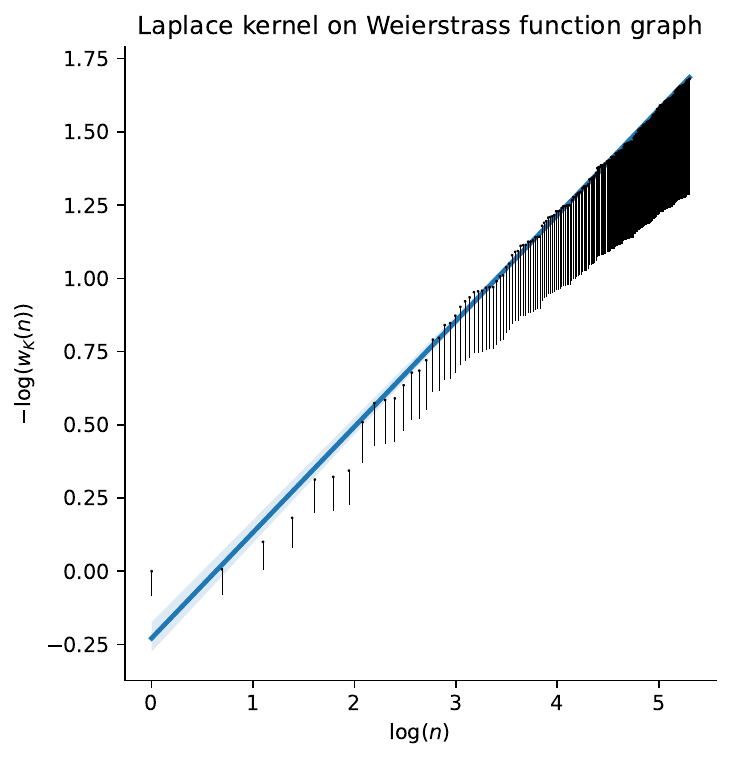}
\end{minipage}\hfill\begin{minipage}[t]{.2\textwidth}
    \centering
    \includegraphics[width=1.0\textwidth]{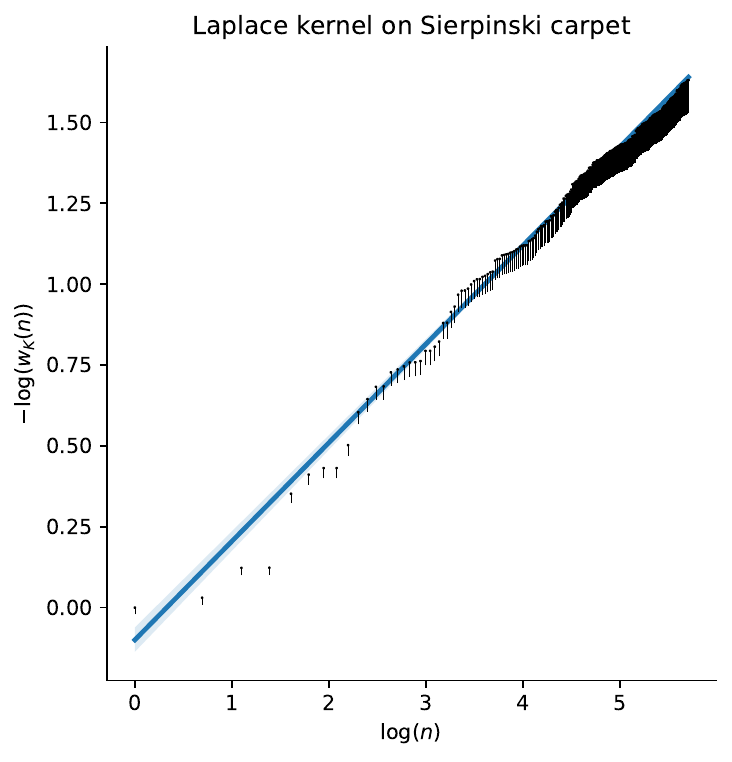}
\end{minipage}\hfill\begin{minipage}[t]{.2\textwidth}
    \centering
    \includegraphics[width=1.0\textwidth]{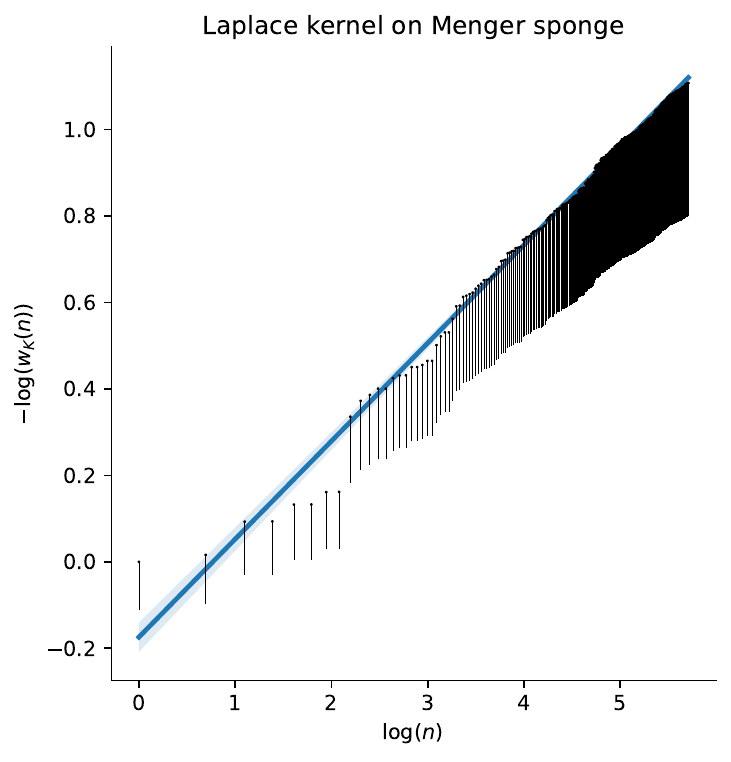}
\end{minipage}\hfill\begin{minipage}[t]{.2\textwidth}
    \centering
    \includegraphics[width=1.0\textwidth]{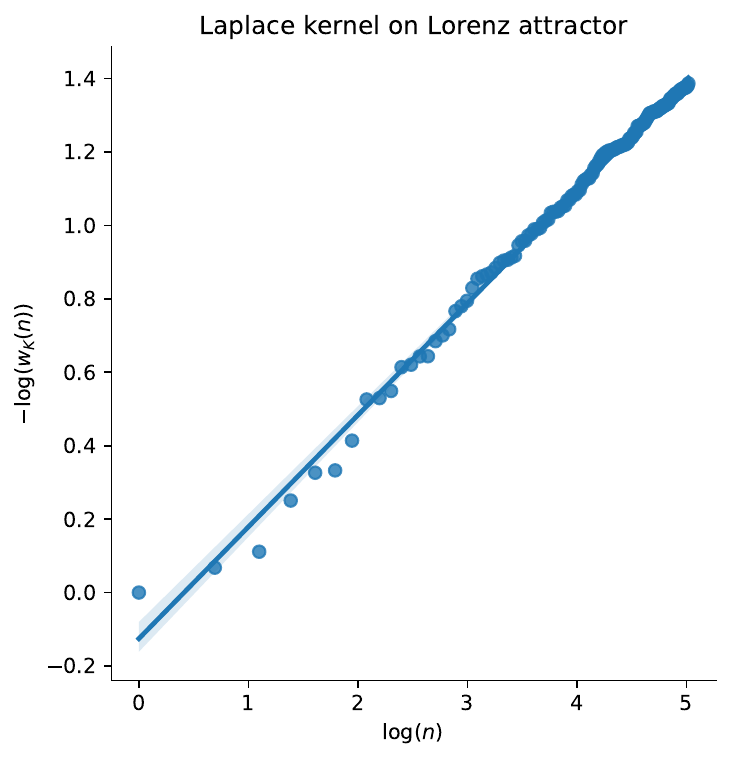}
\end{minipage}
\caption{\small Plot of $-\log(w_n)$ versus $\log n$, where $w_n$ is the output of Algorithm~\ref{empirical-width} for the Laplace kernel on various fractal sets. Black vertical bars represent uncertainty intervals of the form $[-\log(w_n + \varepsilon), -\log(w_n)]$, where $\varepsilon$ is the maximum possible deviation from the true value of $w_K(n)$, estimated using Theorem~\ref{deterministic} and known geometric properties of each fractal. For the Lorenz attractor, $\varepsilon$ is not reliably estimable. The slope of the fitted line, computed using the RANSAC algorithm~\cite{FISCHLER1987726}, is defined as inversely proportional to the empirically estimated effective dimension $d_K^{\rm emp}$.}\label{laplaceplots}    
\end{figure*}

\begin{figure*}[htb]
\begin{minipage}[t]{0.33\textwidth}
    \centering
    \includegraphics[width=1.0\textwidth]{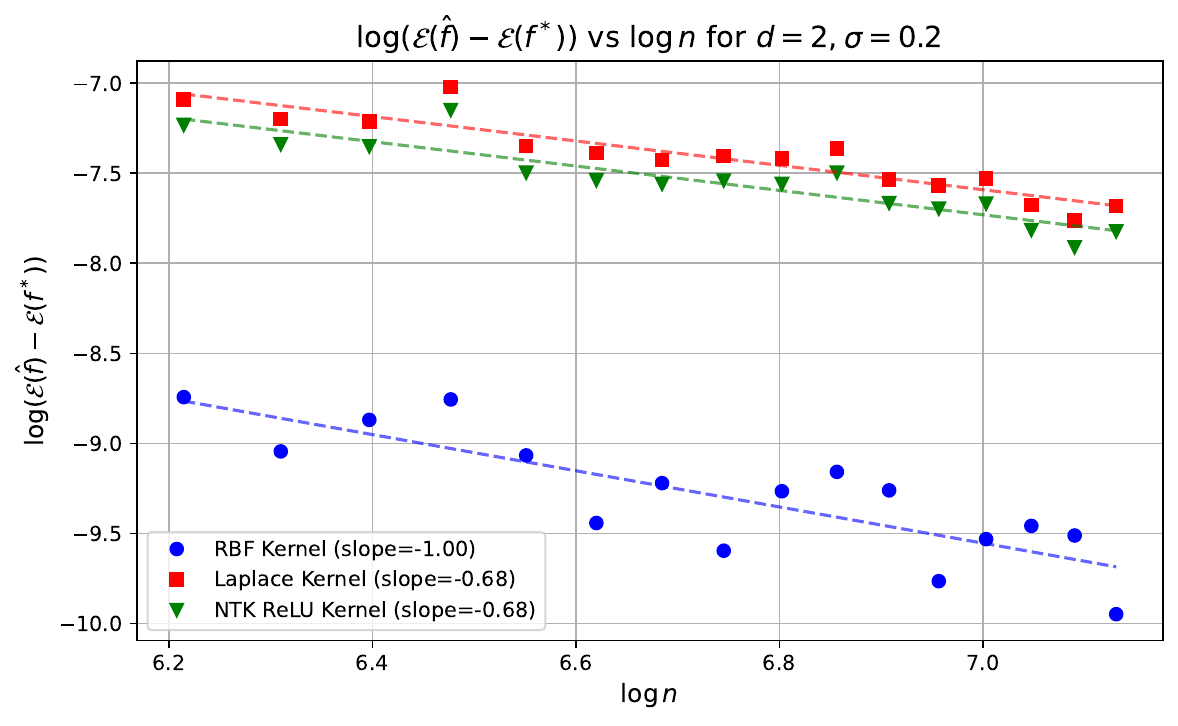}
\end{minipage}\hfill
\begin{minipage}[t]{0.33\textwidth}
    \centering
    \includegraphics[width=1.0\textwidth]{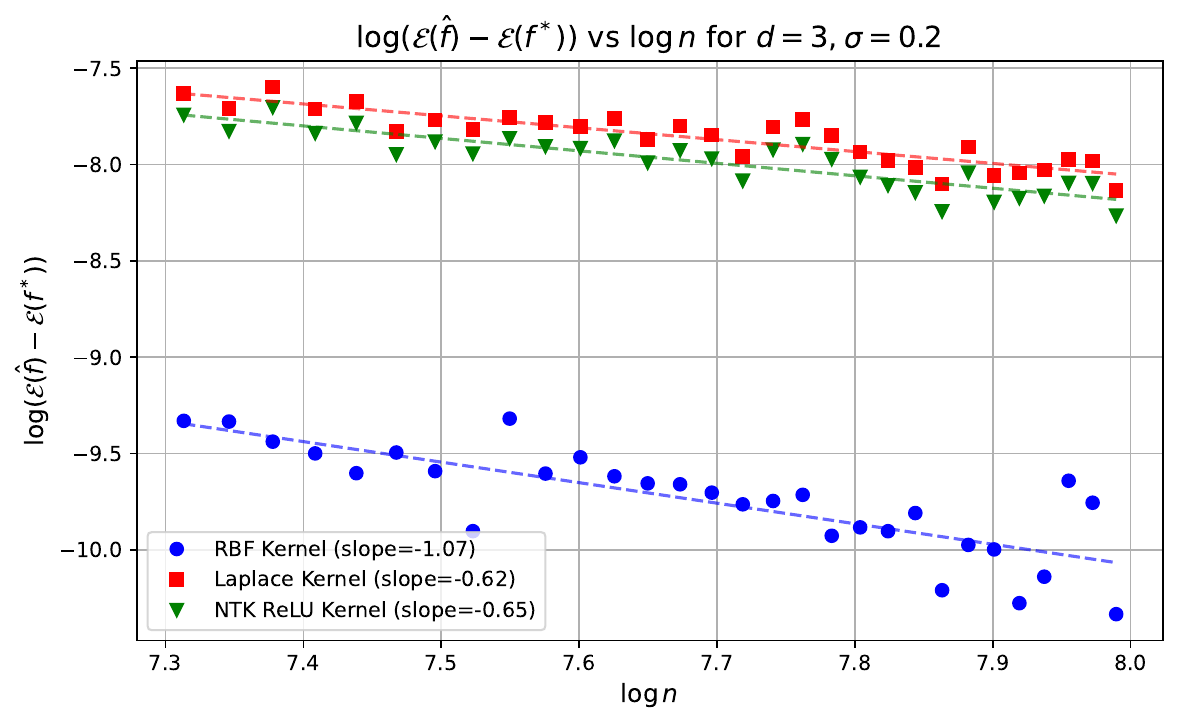}
\end{minipage}\hfill
\begin{minipage}[t]{0.33\textwidth}
    \centering
    \includegraphics[width=1.0\textwidth]{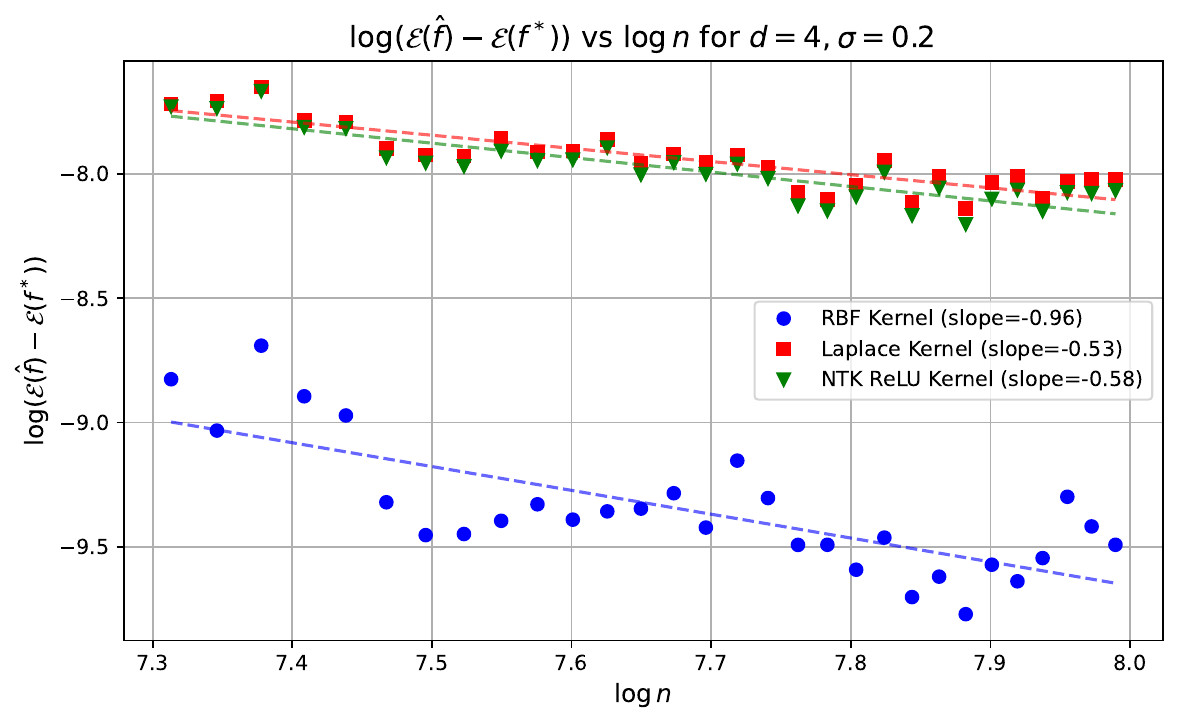}
\end{minipage}
\caption{\small Scatter plots of $\log$-excess risk against $\log$-sample size, with fitted linear regression lines, for $d = 2,3,4$.}\label{excess-risk}
\end{figure*}

{\bf Experiments with fractals}. Our first set of experiments focused on the Laplace kernel, which is closely related --- both theoretically and empirically --- to the neural tangent kernel (NTK) associated with ReLU activation functions~\cite{chen2021deep}. A remarkable property of the Laplace kernel is that $d_\varrho = d_K$ for any Riemann measurable $\Omega\subseteq {\mathbb R}^{d}$ with non-zero volume and for $\Omega = {\mathbb S}^{d-1}$. One might conjecture that this equality extends to less regular subsets of $\mathbb{R}^d$, such as fractals --- but our results show this is not the case. 

We carried out a series of numerical experiments estimating upper bounds on $d_K$ for various well-known fractal sets whose Hausdorff dimensions are explicitly known. 
As shown in Table~\ref{empirical} and plots~\ref{laplaceplots}, the computed upper bounds on $d_K$ are consistently smaller than the corresponding kernel-based metric dimensions $d_\varrho$. This strongly suggests that, for typical fractal sets, the inequality $d_K < d_\varrho$ holds.

{\bf Verification of decay rates of Theorem~\ref{krr-main}.} According to Theorem~\ref{krr-main}, the excess risk of constrained KRR satisfies
$$
\mathcal{E}(\hat{f}) - \mathcal{E}(f^\ast) = \mathcal{O}(n^{-\frac{2+d_K+\varepsilon}{2(1+d_K+\varepsilon)}})
\quad\text{as } n \to \infty.
$$
To empirically verify this rate, we performed numerical experiments for the empirical risk minimization problem with the squared loss $\ell(y,y') = (y-y')^2$
over the hypothesis class $B_{\mathcal{H}_K}$.

We considered three kernels:
(a) Gaussian: $k(x,y) = e^{-0.1\|x-y\|^2}$;
(b) Laplace: $k(x,y) = e^{-\|x-y\|}$;
(c) ReLU NTK: $k_{\mathrm{NTK}}(x,y) = r(u)$, where $r(u) = \frac{1}{\pi} \big[ (2u+1)(\pi - \arccos u) + \sqrt{1-u^2} \big] + 1$ with $u = \langle x, y \rangle$.
We set $\Omega = \mathbb{S}^{d-1}$ and chose the input distribution to be the rotation-invariant measure on $\mathbb{S}^{d-1}$. For each input $X$, the output $Y$ was drawn independently from $0.2$U$([-1, 1])$ where U$([-1, 1])$ denotes the uniform distribution. This corresponds to a pure-noise target with optimal regression function $f^\ast \equiv 0 \in B_{\mathcal{H}_K}$. 
From Table~\ref{rates}, $d_K = 0$ for the Gaussian kernel and $d_K = 2d - 2$ for the Laplace and ReLU NTK kernels. Scatter plots of $\log$-excess-risk versus $\log n$ (with fitted linear regression lines) for $d = 2,3,4$ are shown in Figure~\ref{excess-risk}.

Theorem~\ref{krr-main} predicts that, for large $n$, the slopes of these plots should be bounded above by:
$-1$ for the Gaussian kernel; $-\frac{2}{3}$ ($d=2$), $-0.6$ ($d=3$), $-\frac{4}{7}$ ($d=4$) for the Laplace and ReLU NTK kernels. The experimental results confirm these bounds, with all kernels exhibiting decay rates close to theoretical.

Details of the experimental setup, additional results for NTK kernels, and their discussion are provided in the Appendix.
\section{Conclusions and future research}

Unlike the eigenvalues of the kernel operator, Kolmogorov $n$-widths depend solely on the domain and the kernel, not on the underlying distribution. Their role in quantifying the separation between the approximation power of neural networks and random feature models is well established~\cite{10.5555/3586589.3586708}.
Our excess risk bound reveals that $n$-widths also govern the generalization behavior of KRR. A fundamental trade-off emerges: slower decay of $n$-widths implies greater approximation capacity but weaker generalization, while faster decay leads to stronger generalization but reduced approximation power.
Thus, $n$-widths serve as a meaningful complexity measure, and a deeper investigation of their properties forms a natural direction for future work.

Another promising direction for future research is to investigate kernels for which the effective dimension $d_K$ coincides with the Minkowski dimension $d_\varrho$ on regular domains, and to understand the extent to which this equality breaks down on irregular or fractal domains. Notably, this question is particularly relevant for ReLU-based neural networks, whose associated NTKs belong to this class. Gaining insight into how the equality $d_K = d_\varrho$ breaks down could help clarify the types of domains on which ReLU networks are likely to generalize effectively.

\bibliographystyle{apalike}
\newcommand{\noopsort}[1]{}

\clearpage
\appendix
\thispagestyle{empty}

\onecolumn
\aistatstitle{Supplementary Materials}

\section{Proofs for Section~\ref{eigenvalues}}\label{lower-ismagilov}
\begin{proof}[Proof of Lemma~\ref{schur}] Let $L_n\subseteq \mathcal{H}_K$ be the span of $\{K(x_i,\cdot)\}_1^n$. The squared distance between $K(x,\cdot)$ and $L_n$ is
\begin{equation*}
\begin{split}
&\inf_{f\in L_n}\|K(x,\cdot)-f\|^2_{\mathcal{H}_K} = \min_{[c_i]_1^n}\|K(x,\cdot)-\sum_{i=1}^n c_i K(x_i, \cdot)\|^2_{\mathcal{H}_K} = \\
&\min_{[c_i]_1^n} K(x,x)-2\sum_{i=1}^n K(x,x_i)c_i+\sum _{i,j=1}^n K(x_i,x_j)c_ic_j.
\end{split}
\end{equation*}
The latter minimum is attained at $[c_i]_1^n = K[X_n,X_n]^{-1}K[X_n,x]$ and is equal to
\begin{equation*}
\begin{split}
K(x,x)-K[X_n,x]^\top K[X_n,X_n]^{-1}K[X_n,x],
\end{split}
\end{equation*}
i.e. to the Schur complement $(K[(X_n,x),(X_n,x)]\mid K[X_n,X_n])$.
Thus, by the definition of $w_K(n)$, we have
\begin{equation*}
\begin{split}
w_K(n)^2\leq \sup_{x\in \Omega}(K[(X_n,x),(X_n,x)]\mid K[X_n,X_n]).
\end{split}
\end{equation*}
Since $(K[(X_n,x),(X_n,x)]\mid K[X_n,X_n]) = \frac{{\rm det}(K[(X_n,x),(X_n,x)])}{{\rm det}(K[X_n,X_n])}$ the latter supremum is attained an $x_{n+1} = f_n(X_n)$ and, therefore,
\begin{equation*}
\begin{split}
w_K(n)^2\leq (K[(X_n,f_n(X_n)),(X_n,f_n(X_n))]\mid K[X_n,X_n]).
\end{split}
\end{equation*}
Lemma proved.
\end{proof}

\begin{lemma}\label{measures} The sequence $\{x_i\}_1^\infty\subseteq \Omega$ defined by~\eqref{sequence} satisfies 
\begin{equation*}
\begin{split}
\prod_{i=1}^n\lambda_i({\rm O}_{K,\mu_n})\geq \prod_{i=1}^n\frac{w_K(i-1)^2}{e i},
\end{split}
\end{equation*}
where $\mu_n = \frac{1}{n}\sum_{i=1}^n \delta_{x_i}$. 
\end{lemma}
\begin{proof} Note that $L_2(\Omega, \mu_n)$ is $n$-dimensional and eigenvalues of ${\rm O}_{K,\mu_n}$ are equal to eigenvalues of the empirical kernel matrix $[\frac{1}{n}K(x_i,x_j)]_{i,j=1}^n$. Let again denote $ (x_1, \cdots, x_n)$ by $X_n$. Using Lemma~\ref{schur}, for any natural $k$ we obtain
\begin{equation*}
\begin{split}
&\frac{\prod_{i=1}^k\lambda_i({\rm O}_{K,\mu_{k}})}{\prod_{i=1}^{k-1}\lambda_i({\rm O}_{K,\mu_{k-1}})} = \frac{k^{-k}\prod_{i=1}^k\lambda_i(K[X_k,X_k])}{(k-1)^{-(k-1)}\prod_{i=1}^{k-1}\lambda_i(K[X_{k-1},X_{k-1}])} =\\
&\frac{1}{k}(1-\frac{1}{k})^{k-1}(K[X_k,X_k]\mid K[X_{k-1},X_{k-1}])
 \geq \frac{1}{ek} w_K(k-1)^2.
\end{split}
\end{equation*}
Therefore,
\begin{equation*}
\begin{split}
\prod_{i=1}^n\lambda_i({\rm O}_{K,\mu_{n}})\geq \prod_{k=1}^n \frac{w_K(k-1)^2}{ek} .
\end{split}
\end{equation*}
\end{proof}

\begin{proof}[Proof of the second part of Theorem~\ref{from-isma}]
Let us assume the opposite, i.e. $\limsup\limits_{n\to+\infty}\sup\limits_{\mu} \frac{n\lambda_n({\rm O}_{K,\mu})}{w(n)^2}<\frac{1}{e}$, and obtain a contradiction. The latter implies  $\lambda_n({\rm O}_{K,\mu})\leq \frac{\alpha(n)w(n)^2}{n}$ for any $\mu$ and $n$, where $\alpha(n)$ is decreasing and $\lim\limits_{n\to+\infty} \alpha(n)<\frac{1}{e}$. Therefore, for any $\mu$ and $n$, we have
\begin{equation*}
\begin{split}
\prod_{i=1}^n\lambda_i({\rm O}_{K,\mu})\leq  \prod_{i=1}^n \frac{\alpha(i)w(i)^2}{i}.
\end{split}
\end{equation*}
Let us now set $\mu = \mu_n$ where $\mu_n$ is defined as in Lemma~\ref{measures}. From that lemma we conclude
\begin{equation*}
\begin{split}
\prod_{i=1}^n \frac{\alpha(i)w(i)^2}{i}\geq \prod_{i=1}^n\frac{w_K(i-1)^2}{e i},
\end{split}
\end{equation*}
i.e. $w_{K}(n)^2\prod_{i=1}^n e\alpha(i)\geq w_{K}(0)^2$. Since $\prod_{i=1}^n e\alpha(i)\overset{n\to+\infty}{\to}0$ and $w_{K}(n)^2\overset{n\to+\infty}{\to}0$, we obtain the contradiction $w_{K}(0)^2\leq 0$.
\end{proof}

\section{Proofs for Section~\ref{KRR-dimension}}
In this section, the probability measure $\mu$ is fixed and eigenvalues $\{\lambda_i({\rm O}_{K,\mu})\}$ are simply denoted by $\{\lambda_i\}$. 
Given $S=(X_1,\cdots, X_n)\sim \mu^{n}$, the empirical local Rademacher compexity is defined by
\begin{equation*}
\begin{split}
R_n [\mathcal{F}](S, r) = {\mathbb E}_{\sigma}[\sup_{f\in \mathcal{F}, \|f\|_{L_2(\Omega,\mu)}^2\leq r}\frac{1}{n}\sum_{i=1}^n \sigma_i f(X_i)],
\end{split}
\end{equation*}
where $\sigma=(\sigma_1, \cdots, \sigma_n)$ are independent and uniformly distributed over $\{-1,1\}$ (also, independent from $S$). Then, $R_n [\mathcal{F}](r) = {\mathbb E}_{S}\big[R_n [\mathcal{F}](S, r)\big]$ is called the local Rademacher compexity. It was shown in~\cite{10.1007/3-540-45435-7_3} that 
\begin{equation}\label{mendelson}
\begin{split}
R_n [B_{\mathcal{H}_K}](r) \leq \big(\frac{2}{n}\sum_{i=1}^\infty \min(\lambda_i,r)\big)^{\frac{1}{2}},
\end{split}
\end{equation}
for every $r>0$.

Recall that a function $\psi :[0,\infty) \to [0,\infty)$ is called sub-root if it is nonnegative, nondecreasing and if $r\to \frac{\psi(r)}{\sqrt{r}}$ is nonincreasing for $r > 0$.
\begin{lemma}\label{sub-root} For any $\varepsilon>0$, there exists a sub-root function $\psi_\varepsilon$ and $N_\varepsilon\in {\mathbb N}$ such that $\forall r>0, \big(\frac{2}{n}\sum_{i=1}^\infty \min(\lambda_i,r)\big)^{\frac{1}{2}}\leq \psi_\varepsilon(r)$ and for any natural $N\geq 2N_\varepsilon$ and $r_N = (N/2)^{-1-\frac{2}{d_K+\varepsilon}}$ we have
\begin{equation*}
\begin{split}
\psi_\varepsilon(r_N) \leq 
\sqrt{2(3+d_K+\varepsilon)} r_N^{\frac{1}{2+d_K+\varepsilon}}n^{-\frac{1}{2}}.
\end{split}
\end{equation*}
\end{lemma}
\begin{proof}
Since $d_{K} = \limsup_{i\to+\infty}\frac{\log i}{\log (1/w_K(i))}$ and $\lim_{i\to+\infty}w_K(i)=0$, for any $\varepsilon>0$ there is $N_\varepsilon\in {\mathbb N} $ such that $\frac{\log i}{\log (1/w_K(i))}< d_K+\varepsilon$ and $w_K(i)<1$ whenever $i>N_\varepsilon$. Thus, $w_K(i)<i^{-\frac{1}{d_K+\varepsilon}}$ for $i>N_\varepsilon$. The corollary~\ref{from-isma} gives us $\lambda_i\leq \frac{w_K(\lfloor i/2\rfloor)^2}{(i-1)/2}$. Thus, for $i>2N_\varepsilon+1$ we obtain $\lambda_i\leq \frac{((i-1)/2)^{-\frac{2}{d_K+\varepsilon}}}{(i-1)/2}\leq ((i-1)/2)^{-1-\frac{2}{d_K+\varepsilon}}$. 

The RHS of the inequality~\eqref{mendelson} can be bounded as
\begin{equation*}
\begin{split}
\big(\frac{2}{n}\sum_{i=1}^\infty \min(\lambda_i,r)\big)^{\frac{1}{2}} \leq \left(\frac{2 N r}{n}+\frac{2 }{n}\sum_{i=N+1}^\infty ((i-1)/2)^{-1-\frac{2}{d_K+\varepsilon}}\right)^{\frac{1}{2}},
\end{split}
\end{equation*}
for any $r>0$ and natural $N\geq 2N_\varepsilon+1$. 
Since $\sum_{i=N+1}^\infty ((i-1)/2)^{-1-\frac{2}{d_K+\varepsilon}}<2\int_{(N-1)/2}^\infty x^{-1-\frac{2}{d_K+\varepsilon}}dx = 2^{\frac{2}{d_K+\varepsilon}}(d_K+\varepsilon)(N-1)^{-\frac{2}{d_K+\varepsilon}}$, 
we have
\begin{equation*}
\begin{split}
\big(\frac{2}{n}\sum_{i=1}^\infty \min(\lambda_i,r)\big)^{\frac{1}{2}} \leq \psi_\varepsilon(r),
\end{split}
\end{equation*}
where $\psi_\varepsilon(r) = \min_{N\in {\mathbb N}\cap [2N_\varepsilon+1,\infty)}\big(\frac{2 N r}{n}+\frac{2^{1+\frac{2}{d_K+\varepsilon}}}{n}(d_K+\varepsilon)(N-1)^{-\frac{2}{d_K+\varepsilon}}\big)^{\frac{1}{2}}$.

Let us prove that $\psi_\varepsilon(r)$ is sub-root. The fact that it is nonnegative is obvious. Let us prove that $r\to \frac{\psi_\varepsilon(r)}{\sqrt{r}}$ is nonincreasing. Indeed, for any $r_1>r_2>0$ we  have
\begin{equation*}
\begin{split}
\frac{2 N }{n}+\frac{2^{1+\frac{2}{d_K+\varepsilon}}}{r_1 n}(d_K+\varepsilon)(N-1)^{-\frac{2}{d_K+\varepsilon}}\leq \frac{2 N }{n}+\frac{2^{1+\frac{2}{d_K+\varepsilon}}}{r_2 n}(d_K+\varepsilon)(N-1)^{-\frac{2}{d_K+\varepsilon}},
\end{split}
\end{equation*}
for any natural $N\geq 2N_\varepsilon+1$. Therefore, $\frac{\psi_\varepsilon(r_1)}{\sqrt{r_1}}\leq \frac{\psi_\varepsilon(r_2)}{\sqrt{r_2}}$. Analogously, it is proved that $\psi_\varepsilon(r_1)\geq \psi_\varepsilon(r_2)$.

Let us set $r_{N-1} = 2^{1+\frac{2}{d_K+\varepsilon}} (N-1)^{-1-\frac{2}{d_K+\varepsilon}}$ for some natural $N\geq 2N_\varepsilon+1$. By construction, 
\begin{equation*}
\begin{split}
&\psi(r_{N-1}) \leq \big(\frac{2^{1+\frac{2}{d_K+\varepsilon}}}{n}(\frac{2}{N-1}+2+d_K+\varepsilon)(N-1)^{-\frac{2}{d_K+\varepsilon}}\big)^{\frac{1}{2}}\leq \\
&\sqrt{2(3+d_K+\varepsilon)} r_{N-1}^{\frac{1}{2+d_K+\varepsilon}}n^{-\frac{1}{2}}.
\end{split}
\end{equation*}
\end{proof}

The following theorem is a direct consequence of Corollary 5.3. from~\cite{Rademacher}. 
\begin{theorem}[\cite{Rademacher}]\label{bartlett} Let $\max_{x\in \Omega}K(x,x)\leq 1$ and $l:{\mathbb R}\times {\mathbb R}\to {\mathbb R}$ be a loss function, that satisfies condition (1)-(3). Let $\hat{f} = \arg\min_{f\in B_{\mathcal{H}_K}}\sum_{i=1}^n l(f(X_i), Y_i)$ and $f^\ast = \arg\min_{f\in B_{\mathcal{H}_K}}{\mathbb E}[l(f(X), Y)]$. Assume $\psi$ is a sub-root function for which $\psi(r)\geq BL \cdot R_n [2B_{\mathcal{H}_K}](\frac{r}{L^2})$.
Then, for any $x > 0$ and any $r \geq \psi(r)$, with probability at least $1-e^{-x}$,
\begin{equation*}
\begin{split}
{\mathbb E}[l(\hat{f}(X), Y)]-{\mathbb E}[l(f^\ast(X), Y)]\leq 705\frac{r}{B}+\frac{(11L+27B)x}{n}.
\end{split}
\end{equation*}
\end{theorem}

\begin{proof}[Proof of Theorem~\ref{krr-main}]
Let us fix $\varepsilon>0$ and consider a sub-root function $\psi_\varepsilon$ whose existence is guaranteed by Lemma~\ref{sub-root}. Using the inequality~\eqref{mendelson}, we have
\begin{equation*}
\begin{split}
BL\cdot R_n [2B_{\mathcal{H}_K}](\frac{r}{L^2})\leq BL \big(\frac{2}{n}\sum_{i=1}^\infty \min(4\lambda_i,\frac{r}{L^2})\big)^{\frac{1}{2}}\leq 2BL\psi_\varepsilon(\frac{r}{4L^2}).
\end{split}
\end{equation*}
Therefore, let us define $\psi(r) = 2BL\psi_\varepsilon(\frac{r}{4L^2})$. By construction, $\psi$ is a sub-root function that satisfies $\psi(r)\geq BL \cdot R_n [2B_{\mathcal{H}_K}](\frac{r}{L^2})$. 
According to Lemma~\ref{sub-root}, for any natural $N\geq 2N_\varepsilon$ and $r_N = (N/2)^{-1-\frac{2}{d_K+\varepsilon}}$ we have
\begin{equation*}
\begin{split}
2BL\psi_\varepsilon(r_N) \leq 
2BL \sqrt{2(3+d_K+\varepsilon)} r_N^{\frac{1}{2+d_K+\varepsilon}}n^{-\frac{1}{2}}.
\end{split}
\end{equation*}
Let us assume $\tilde{r} = 4L^2 r_N$ and find such a natural $N\geq 2N_\varepsilon$ that $2BL \sqrt{2(3+d_K+\varepsilon)} r_N^{\frac{1}{2+d_K+\varepsilon}}n^{-\frac{1}{2}}\leq \tilde{r}$. 
In other words,
\begin{equation*}
\begin{split}
2BL \sqrt{2(3+d_K+\varepsilon)} r_N^{\frac{1}{2+d_K+\varepsilon}}n^{-\frac{1}{2}}\leq 4L^2 r_N,
\end{split}
\end{equation*}
or
\begin{equation*}
\begin{split}
\frac{B}{2L} \sqrt{2(3+d_K+\varepsilon)} n^{-\frac{1}{2}}\leq r_N^{\frac{1+d_K+\varepsilon}{2+d_K+\varepsilon}}=(N/2)^{-\frac{1+d_K+\varepsilon}{d_K+\varepsilon}}.
\end{split}
\end{equation*}
If $2\left(\frac{2n L^2}{B^2(3+d_K+\varepsilon)} \right)^{\frac{d_K+\varepsilon}{2(1+d_K+\varepsilon)}}\geq 2N_\varepsilon$,
we can set $N=\lceil 2\left(\frac{2n L^2}{B^2(3+d_K+\varepsilon)} \right)^{\frac{d_K+\varepsilon}{2(1+d_K+\varepsilon)}} \rceil$. Thus,
\begin{equation*}
\begin{split}
&\psi(\tilde{r})\leq \tilde{r} = 4L^2 r_N \leq 4L^2 \left(\frac{2\left(\frac{2n L^2}{B^2(3+d_K+\varepsilon)} \right)^{\frac{d_K+\varepsilon}{2(1+d_K+\varepsilon)}}}{2}\right)^{-1-\frac{2}{d_K+\varepsilon}}\leq \\
&4L^2 \left(\frac{2n L^2}{B^2(3+d_K+\varepsilon)} \right)^{-\frac{2+d_K+\varepsilon}{2(1+d_K+\varepsilon)}}.
\end{split}
\end{equation*}
From Theorem~\ref{bartlett} we obtain
\begin{equation*}
\begin{split}
&{\mathbb E}[l(\hat{f}(X), Y)]-{\mathbb E}[l(f^\ast(X), Y)]\leq 705\frac{\tilde{r}}{B}+\frac{(11L+27B)x}{n}\leq \\
&CB^{\frac{1}{1+d_K+\varepsilon}}L^{1-\frac{1}{1+d_K+\varepsilon}} n^{-\frac{2+d_K+\varepsilon}{2(1+d_K+\varepsilon)}}+\frac{(11L+27B)x}{n}.
\end{split}
\end{equation*}
with probability at least $1-e^{-x}$, where $C = 705\cdot 2^{1.5}<1995$.
\end{proof}

\section{Proof for Section~\ref{upper-bound-algo}}\label{proofs-upper-bound-algo}
\begin{proof}[Proof of Lemma~\ref{lipshits}]
Let us denote $S_0(x,y) = K(x,y)$ and $S_n(x,y) = K(x,y)- K[X_n,x]^\top K[X_n,X_n]^{-1}K[X_n,y]$ for any $n\in {\mathbb N}$. Note that $S_i(x)=S_i(x,x)$.

The projection of the function $K(x,\cdot)$ onto the span of $\{K(x_1,\cdot), \cdots, K(x_n, \cdot)\}$ in $\mathcal{H}_K$, denoted by ${\rm pr}_{{\rm span}(X_n)}K(x,\cdot)$, equals
\begin{equation*}
\begin{split}
[{\rm pr}_{{\rm span}(X_n)}K(x,\cdot)](y) = K[X_n,x]^\top K[X_n,X_n]^{-1}K[X_n,y].
\end{split}
\end{equation*}
Therefore,
\begin{equation*}
\begin{split}
[{\rm pr}_{{\rm span}(X_n)^\perp}K(x,\cdot)](y) = K(x,y)-K[X_n,x]^\top K[X_n,X_n]^{-1}K[X_n,y].
\end{split}
\end{equation*}
and
\begin{equation*}
\begin{split}
&\langle {\rm pr}_{{\rm span}(X_n)^\perp}K(x,\cdot), {\rm pr}_{{\rm span}(X_n)^\perp}K(y,\cdot)\rangle_{\mathcal{H}_K} = \\
&\langle K(x,\cdot)-K(x,X_n)K(X_n,X_n)^{-1}K(X_n,\cdot), \\
&K(y,\cdot)-K(y,X_n)K(X_n,X_n)^{-1}K(X_n,\cdot)\rangle_{\mathcal{H}_K}=\\
&K(x,y)-K(x,X_n)K(X_n,X_n)^{-1}K(X_n,y)=S_n(x,y).
\end{split}
\end{equation*}
Let $d_n(x,y) = \sqrt{S_n(x,x)+S_n(y,y)-2S_n(x,y)}$.
Since $\langle {\rm pr}_{{\rm span}(X_n)^\perp}f, {\rm pr}_{{\rm span}(X_n)^\perp}g\rangle_{\mathcal{H}_K}$ is a semi-definite bilinear form on $\mathcal{H}_K$, it satisfies the triangle inequality, i.e.
\begin{equation*}
\begin{split}
&d_n(x,y) = \langle {\rm pr}_{{\rm span}(X_n)^\perp}K(x,\cdot)-K(y,\cdot), {\rm pr}_{{\rm span}(X_n)^\perp}K(x,\cdot)-K(y,\cdot)\rangle_{\mathcal{H}_K}^{\frac{1}{2}}\geq \\
&\langle {\rm pr}_{{\rm span}(X_n)^\perp}K(x,\cdot), {\rm pr}_{{\rm span}(X_n)^\perp}K(x,\cdot)\rangle_{\mathcal{H}_K}^{\frac{1}{2}}-\langle {\rm pr}_{{\rm span}(X_n)^\perp}K(y,\cdot), {\rm pr}_{{\rm span}(X_n)^\perp}K(y,\cdot)\rangle_{\mathcal{H}_K}^{\frac{1}{2}}.
\end{split}
\end{equation*}
Thus,
\begin{equation*}
\begin{split}
\sqrt{S_n(x,x)}-\sqrt{S_n(y,y)}\leq d_n(x,y) \leq d_1(x,y)=\varrho(x,y).
\end{split}
\end{equation*}

\end{proof}

\begin{lemma}\label{sample-eps-net} Let $Z_1, \cdots, Z_N\sim^{\rm iid} \mu$ and let
\begin{equation*}
\begin{split}
f(\varepsilon) = \min_{x\in \Omega} \mu(B(x,\varepsilon)\cap \Omega).
\end{split}
\end{equation*}
Then, for any $\varepsilon>0$ we have
\begin{equation*}
\begin{split}
{\mathbb P}[\{Z_1, \cdots, Z_N\}{\rm \,\,is\,\, a\,\,} 2\varepsilon-{\rm net\,\, in\,\,} (\Omega, \varrho) ]\geq 
1-\mathcal{N}(\varepsilon,\Omega,\varrho) e^{-Nf(\varepsilon)}.
\end{split}
\end{equation*}
\end{lemma}
\begin{proof}
Let $\varepsilon >0$ and $M=\mathcal{N}(\varepsilon,\Omega,\varrho)$. By construction, there are $z_1,z_2, \cdots, z_M\in \Omega$ satisfying $\bigcup_{i=1}^M B(z_i,\varepsilon)\supseteq \Omega$. If $\{Z_1, \cdots, Z_N\}\cap B(z_i, \varepsilon)\ne \emptyset$ for all $i: 1\leq i\leq M$, then $\{Z_1, \cdots, Z_N\}$ is the $2\varepsilon$-net in $(\Omega, \varrho)$.

The probability that there is at least one $i$ such that $\{Z_1, \cdots, Z_N\}\cap B(z_i, \varepsilon)= \emptyset$ is not more than $\sum_{i=1}^M (1-\mu(B(z_i, \varepsilon)\cap \Omega))^N$, which does not exceed $M(1-\min_{x\in \Omega}\mu(B(x, \varepsilon)\cap \Omega))^N$. Thus,
\begin{equation*}
\begin{split}
&{\mathbb P}[\{Z_1, \cdots, Z_N\}{\rm \,\,is\,\, a\,\,} 2\varepsilon-{\rm net\,\, in\,\,} (\Omega, \varrho) ]\geq \\
&1-M(1-f(\varepsilon))^N\geq 1-M e^{-Nf(\varepsilon)}.
\end{split}
\end{equation*}
Lemma proved.
\end{proof}

\begin{corollary}\label{prev-cor} Suppose that the measure $\mu$ is $C$-uniform over $\Omega$ and $Z_1, \cdots, Z_N\sim^{\rm iid} \mu$. For any sequence $\{x_t\}_{t=1}^\infty \subseteq \Omega$ and any $\delta\in (0,1)$, we have
\begin{equation*}
\begin{split}
{\mathbb P}[\{Z_1, \cdots, Z_N\}{\rm \,\,is\,\, a\,\,} 2\varepsilon-{\rm net\,\, in\,\,} (\Omega, \varrho) ]\geq 1-\delta,
\end{split}
\end{equation*}
provided that $N\geq \frac{\log \mathcal{N}(\varepsilon,\Omega,\varrho)+\log \frac{1}{\delta}}{C\varepsilon^{d_\varrho}}$.
\end{corollary}

\begin{proof} Since $f(\varepsilon)\geq C\varepsilon^{d_\varrho}$,  $\mathcal{N}(\varepsilon,\Omega,\varrho) e^{-Nf(\varepsilon)}\leq \mathcal{N}(\varepsilon,\Omega,\varrho) e^{-NC\varepsilon^{d_\varrho}}$. Then, $\delta\geq \mathcal{N}(\varepsilon,\Omega,\varrho)e^{-NC\varepsilon^{d_\varrho}}$ is equivalent to $N\geq \frac{\log \mathcal{N}(\varepsilon,\Omega,\varrho)+\log \frac{1}{\delta}}{C\varepsilon^{d_\varrho}}$.  Therefore
\begin{equation*}
\begin{split}
{\mathbb P}[\{Z_1, \cdots, Z_N\}{\rm \,\,is\,\, a\,\,} 2\varepsilon-{\rm net\,\, in\,\,} (\Omega, \varrho) ]\geq 1-\delta,
\end{split}
\end{equation*}
if $N\geq \frac{\log \mathcal{N}(\varepsilon,\Omega,\varrho)+\log \frac{1}{\delta}}{C\varepsilon^{d_\varrho}}$.
\end{proof}

\begin{proof}[Proof of Theorem~\ref{finite-omega}]
We only to use Corollary~\ref{prev-cor} in combination with Theorem~\ref{deterministic}. 
\end{proof}

\subsection{Sample size for smaller $t$}\label{proofs-smaller-t}
Let us call the pair $(\mu, K)$ non-degenerate if for $X_1, \cdots, X_n\sim^{\rm iid}\mu,n\geq 3$ the probability of the event
\begin{equation*}
\begin{split}
&K(X_1,X_1)-K([X_j]_{j=3}^n,X_1)^\top K([X_j]_{j=3}^n,[X_j]_{j=3}^n)^{-1} K([X_j]_{j=3}^n,X_1)=\\
&K(X_2,X_2)-K([X_j]_{j=3}^n,X_2)^\top K([X_j]_{j=3}^n,[X_j]_{j=3}^n)^{-1} K([X_j]_{j=3}^n,X_2)
\end{split}
\end{equation*} 
is zero.
\begin{lemma}\label{small-t}
Suppose that $(\mu,K)$ is non-degenerate, $Z_1, \cdots, Z_N\sim^{\rm iid} \mu$ and
\begin{equation*}
\begin{split}
f(\varepsilon) = \min_{x\in \Omega} \mu(B(x,\varepsilon)\cap \Omega).
\end{split}
\end{equation*} 
Let $\{w_t\}_{t=0}^{T-1}$ be an output of Algorithm~\ref{empirical-width} for $\tilde{\Omega}=\{Z_1, \cdots, Z_N\}$. For any $t\in \{1,\cdots, T\}$ and $\varepsilon>0$, we have
\begin{equation*}
\begin{split}
{\mathbb P}[w_{t-1} < w_K(t-1)-\varepsilon]\leq {N \choose t} e^{-f(\varepsilon)(N-t)}.
\end{split}
\end{equation*}
\end{lemma}
\begin{proof} Let $\mu^N$ be the product measure $\mu \times \cdots \times \mu$ on $\Omega^N$. We will treat elements $\{x_t\}_{t=1}^T$, $\{w_t\}_{t=0}^{T-1}$ generated by the Algorithm~\ref{empirical-width} as random variables $\Omega^N\to {\mathbb R}$ depending on the random input $\tilde{\Omega}=\{Z_1, \cdots, Z_N\}$ where $Z_1, \cdots, Z_N\sim^{\rm iid}\mu$. 

A sequence of elements $a_1, \cdots, a_t\in \Omega$ defines the sequence of functions $\{S_i^{\mathbf a}\}_{i=0}^{t}$ by $S_i^{\mathbf a}(x) = K(x,x) - K([a_j]_{j=1}^i, x)^\top K([a_j]_{j=1}^i,[a_j]_{j=1}^i)^{-1}K([a_j]_{j=1}^i, x)$.  
Let $A_t$ be a set of tuples ${\mathbf a} = (a_1, \cdots, a_t)\in \Omega^t$ such that $S_{i-1}^{\mathbf a}(a_i)\geq S_{i-1}^{\mathbf a}(a_j)$ for any $i,j\in \{1,\cdots,t\}$. We denote the event $x_1=a_1, \cdots, x_t=a_t$ by $E_{{\mathbf a}}$. Note that for $E_{{\mathbf a}}=\emptyset$ if ${\mathbf a}\notin A_t$.

Let $P(N,t)$ denote the set of all $t$-element arrangements of $\{1,\cdots, N\}$. The event $E_{a_1, \cdots, a_t}$ can be represented as
\begin{equation*}
\begin{split}
E_{a_1, \cdots, a_t} = \bigcup_{(i_1, \cdots, i_t)\in P(N,t)} C_{i_1, \cdots, i_t}^{a_1, \cdots, a_t},
\end{split}
\end{equation*}
where $C_{i_1, \cdots, i_t}^{a_1, \cdots, a_t}$ denotes the event $Z_{i_1}=a_1,\cdots,Z_{i_t}=a_t$, $S_{u-1}^{\mathbf a}(a_u)\geq S_{u-1}^{\mathbf a}(Z_{j})$, $\forall u\in \{1,\cdots, t\}$, $j\in \{1,\cdots, N\}$. Due to symmetry, all probabilities ${\mathbb P}[C_{i_1, \dots, i_t}^{a_1, \dots, a_t}]$ are equal for a fixed ${\mathbf a}$.
 Therefore, for any Borel set $B\subseteq \Omega^t$, we have
\begin{equation*}
\begin{split}
{\mathbb P}[(x_1, \cdots, x_t)\in B]=
{\mathbb P}[\bigcup_{(a_1, \cdots, a_t)\in B} E_{a_1, \cdots, a_t}] 
\leq |P^N_t|\cdot {\mathbb P}[ \bigcup_{(a_1, \cdots, a_t)\in B} C_{1, \cdots, t}^{a_1, \cdots, a_t}].
\end{split}
\end{equation*}
We are specifically interested in the case of
\begin{equation*}
\begin{split}
B = \{(a_1, \cdots, a_t)\in A_t \mid S_{t-1}^{\mathbf a}(a_t) < w_K(t-1)-\varepsilon\},
\end{split}
\end{equation*}
due to ${\mathbb P}[w_{t-1} < w_K(t-1)-\varepsilon]= {\mathbb P}[(x_1, \cdots, x_t)\in B]$.
We will estimate ${\mathbb P}[ \bigcup_{(a_1, \cdots, a_t)\in B} C_{1, \cdots, t}^{a_1, \cdots, a_t}]$ for this specific choice of $B$. Let us denote
\begin{equation*}
\begin{split}
\Omega_{a_1, \cdots, a_t} = \{x\in \Omega\mid S_{u-1}^{\mathbf a}(a_u)\geq S_{u-1}^{\mathbf a}(x), \forall u\in \{1,\cdots, t\}\}.
\end{split}
\end{equation*}
The event $C_{1, \cdots, t}^{a_1, \cdots, a_t}$, as an element of sigma algebra over $\Omega^N$, satisfies 
\begin{equation*}
\begin{split}
C_{1, \cdots, t}^{a_1, \cdots, a_t} = \{(a_1, \cdots, a_t)\}\times \Omega_{a_1, \cdots, a_t}^{N-t}.
\end{split}
\end{equation*}
By Fubini's theorem, we have
\begin{equation*}
\begin{split}
&{\mathbb P}[(x_1, \cdots, x_t)\in B] 
\leq \frac{N!}{(N-t)!}\int_B \Big(\int_{\Omega_{{\mathbf a}}^{N-t}}d\mu^{N-t}\Big)d\mu^t({\mathbf a}) = \\
&\frac{N!}{(N-t)!}\int_B \mu^{N-t}(\Omega_{{\mathbf a}}^{N-t})d\mu^t({\mathbf a}) = \frac{N!}{(N-t)!}\int_B \mu(\Omega_{{\mathbf a}})^{N-t}d\mu^t({\mathbf a}). 
\end{split}
\end{equation*}

Let $a_t^\ast \in \arg\max_{x\in \Omega}\sqrt{S^{{\mathbf a}}_{t-1}(x)}$. For any $z\in \Omega_{{\mathbf a}}$ we have (a) $\sqrt{S_{t-1}(a_t)}\geq \sqrt{S_{t-1}(z)}$, (b) $w_K(t-1)\leq \sqrt{S_{t-1}(a_t^\ast)}$. Since $\sqrt{S_{t-1}}$ is 1-Lipschitz (Lemma~\ref{lipshits}), we have $w_K(t-1)-\sqrt{S_{t-1}(a_t)}\leq \sqrt{S_{t-1}(a_t^\ast)}-\sqrt{S_{t-1}(z)}\leq \varrho(a_t^\ast, z)$. Thus, $\sqrt{S_{t-1}(a_t)} < w_K(t-1)-\varepsilon$ implies $\varrho(a_t^\ast, z)> \varepsilon$. This leads to 
\begin{equation*}
\begin{split}
\mu(\Omega_{{\mathbf a}})\leq \mu(\{z\in \Omega\mid \varrho(a_t^\ast, z)> \varepsilon\})\leq 1-f(\varepsilon),
\end{split}
\end{equation*}
for any ${\mathbf a}\in B$.
\
Thus, we proved
\begin{equation*}
\begin{split}
{\mathbb P}[(x_1, \cdots, x_t)\in B] 
\leq \frac{N!}{(N-t)!} (1-f(\varepsilon))^{N-t}\int_{A_t}  d\mu^t({\mathbf a}).
\end{split}
\end{equation*}
Due to non-degeneracy condition, $t!\int_{A_t}  d\mu^t({\mathbf a}) = \int_{\Omega^t}  d\mu^t({\mathbf a})=1$, and we obtain
\begin{equation*}
\begin{split}
{\mathbb P}[w_{t-1} < w_K(t-1)-\varepsilon]\leq {N \choose t}  (1-f(\varepsilon))^{N-t}\leq {N \choose t} e^{-f(\varepsilon)(N-t)}.
\end{split}
\end{equation*}
\end{proof}

\begin{theorem} Suppose that the measure $\mu$ is $C$-uniform over $\Omega$, $(\mu,K)$ is non-degenerate, and $Z_1, \cdots, Z_N\sim^{\rm iid} \mu$.  Let $\{w_t\}_{t=0}^{T-1}$ be an output of Algorithm~\ref{empirical-width} for $\tilde{\Omega}=\{Z_1, \cdots, Z_N\}$.  Let $\delta\in (0,1)$ and $\varepsilon>0$. If $10\leq T\leq \frac{NC\varepsilon^{d_\varrho}+\log \delta}{\log N}$, then
$$
w_{t-1} \geq w_K(t-1)-\varepsilon,t=1,\cdots, T
$$
with probability at least $1-\delta$.
\end{theorem}
\begin{proof} Using the previous lemma we obtain
\begin{equation*}
\begin{split}
{\mathbb P}[w_{t-1} \geq w_K(t-1)-\varepsilon,t=1,\cdots, T]\geq 1-\sum_{t=1}^T {N \choose t} e^{-f(\varepsilon)(N-t)}.
\end{split}
\end{equation*}
The latter sum can be bounded by
\begin{equation*}
\begin{split}
\sum_{t=1}^T {N \choose t} e^{-f(\varepsilon)(N-t)}\leq T {N \choose T} e^{-f(\varepsilon)(N-T)}\leq T(\frac{Ne^{2}}{T})^Te^{-f(\varepsilon)N}.
\end{split}
\end{equation*}
Thus, we need to have $T(\frac{Ne^{2}}{T})^Te^{-f(\varepsilon)N}\leq\delta$ to satisfy $w_{t-1} \geq w_K(t-1)-\varepsilon,t=1,\cdots, T$ with probability at least $1-\delta$. This condition is equivalent to 
\begin{equation*}
\begin{split}
-(T-1)\log T+T\log N+2T\leq f(\varepsilon)N+\log\delta.
\end{split}
\end{equation*}
Since $-(T-1)\log T+2T\leq 0$ for $T=10,11,\cdots$, it is sufficient to satisfy
\begin{equation*}
\begin{split}
f(\varepsilon)N+\log\delta\geq T\log N,
\end{split}
\end{equation*}
i.e. $T\leq \frac{NC\varepsilon^{d_\varrho}+\log \delta}{\log N}$.
\end{proof}

\subsection{Implementation of the Algorithm~\ref{empirical-width} and its complexity}\label{implementation-complexity}
{\bf The finite sample implementation.} The actual implementation of the algorithm differs slightly from its formal presentation. Let $|\tilde{\Omega}|=N$. At the $(t+1)$-st iteration, given the inverse matrix $K[X_{t}, X_{t}]^{-1}$, the algorithm selects the next point $x_{t+1} \in \tilde{\Omega}$ by maximizing the expression
$$
K(x,x) - K[X_{t}, x]^\top K[X_{t}, X_{t}]^{-1} K[X_{t}, x],
$$
which requires $\mathcal{O}(t^2 N)$ arithmetic operations (including multiplications, additions, and comparisons) and $\mathcal{O}(tN)$ kernel evaluations.

The inverse matrix $K[X_t, X_t]^{-1}$ can be updated iteratively using the formula:
\begin{equation*}
\begin{split}
K[X_{t},X_{t}]^{-1} = 
\begin{bmatrix}K[X_{t-1},X_{t-1}]^{-1}+ A_tA_t^\top s_t^{-1} & -A_t s_t^{-1}\\
 -A_t^\top s_t^{-1} & s_t^{-1}
\end{bmatrix},
\end{split}
\end{equation*}
where $A_t = K[X_{t-1},X_{t-1}]^{-1} K(X_{t-1},x_t)$ and the Schur complement
$$
s_t = K(x_t, x_t) - K[X_{t-1}, x_t]^\top K[X_{t-1}, X_{t-1}]^{-1} K[X_{t-1}, x_t]
$$
is equal to $w(t-1)^2$, and therefore, has already been computed at a previous iteration.

Thus, the total computational complexity of the algorithm is bounded by $\mathcal{O}(T^3 N)$ arithmetic operations and $\mathcal{O}(T^2 N)$ kernel evaluations. In practice, the latter often dominates the runtime, especially when kernel evaluations are costly --- as is the case with NNGP and NTK kernels (given non-analytically).

{\bf The L-BFGS--based implementation. } Now let $ \tilde{\Omega} = \Omega$. To select the next point $x_{t+1} \in \Omega$, any nonconvex optimization method can be employed. In our experiments, we used the Limited-memory Broyden–Fletcher–Goldfarb–Shanno (L-BFGS) algorithm to maximize
$$
K(x,x) - K[X_{t}, x]^\top K[X_{t}, X_{t}]^{-1} K[X_{t}, x].
$$
The inverse matrix $K[X_{t}, X_{t}]^{-1}$ was updated in the same manner as in the finite-sample implementation. The main drawback of this approach is that, at the $(t+1)$-st iteration, the optimized objective is non-convex, so we cannot guarantee that the computed value at the point $x_{t+1}$ is within $\varepsilon$ of the true maximum. Consequently, the resulting sequence $\{w(t)\}$ may fail to be $\varepsilon$-close to any upper bound on the actual $n$-widths. As shown in Figure~\ref{exponential-type}, the L-BFGS--based implementation produces results that are very close to those of the finite-sample approach, while having the advantage of being completely independent of $N$ (which must be large even for $T \approx 300$ in the finite sample setting). We also report cases in which the output sequence $\{w(t)\}$ substantially underestimates the actual $n$-widths. A possible remedy is to invoke L-BFGS multiple times with independent initializations to obtain a more accurate estimate of the optimal $x_{t+1}$.


\section{Additional experiments}

\subsection{Details of experiments with fractals} 
In experiments with the Laplace kernel on fractals we used the following values of parameters:
\begin{itemize}
\item Cantor set:  $N=2^{15}$, $T=300$;
\item Weierstrass function graph: $N=2^{24}$, $T=200$;
\item Sierpi{\'n}ski carpet:  $N=8^{8}$, $T=200$;
\item Menger sponge: $N=20^{5}$, $T=300$;
\item Lorenz attractor: $N=10^{6}$, $T=300$.
\end{itemize}

Our empirical estimation of the effective dimension has two main limitations.
First, the relation $w_K(n) \asymp n^{-1/d_K}$ holds only asymptotically, whereas we estimate $d_K$ using $n$-width approximations for $n \le 300$.
Second, the estimators we employ are upper bounds on the true $n$-widths and are therefore not unbiased.

To demonstrate that our estimates are nevertheless quite accurate, we also computed lower bounds on the $n$-widths using Ismagilov’s theorem:
$$
w_K(n) \ge \sqrt{\sum_{i>n}\lambda_i(\mathcal{O}_{K,\mu})}.
$$
We calculated these lower bounds for the first 400 $n$-widths based on the eigenvalues ${\hat{\lambda}_i}$ of the empirical kernel matrix $[K(x_i, x_j)]_{i,j=1}^{M}$ with $M = 10{,}000$.

As shown in Figure~\ref{laplaceplots2}, Ismagilov’s lower bounds differ from the upper bounds produced by Algorithm~\ref{empirical-width} only by a constant factor asymptotically, and their log-log plots have nearly identical slopes. This provides additional evidence that our empirical estimates of the effective dimension are quite accurate.
\begin{figure*}[htb]
\begin{minipage}[t]{.2\textwidth}
    \centering
    \includegraphics[width=1.0\textwidth]{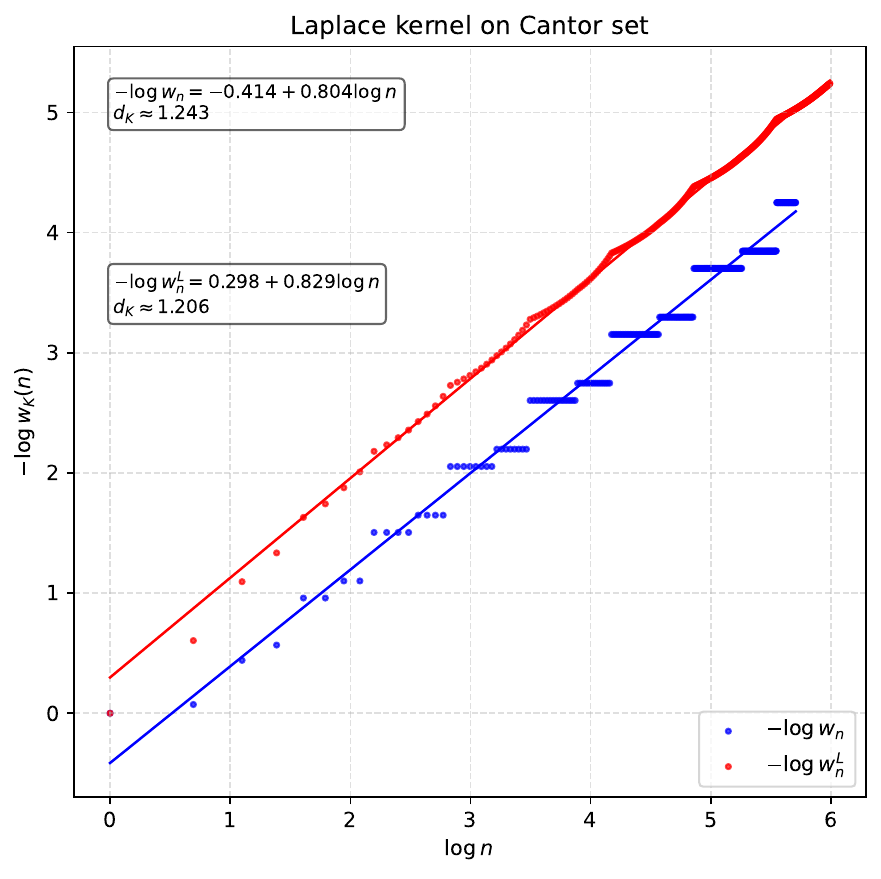}
\end{minipage}\hfill\begin{minipage}[t]{.2\textwidth}
    \centering
    \includegraphics[width=1.0\textwidth]{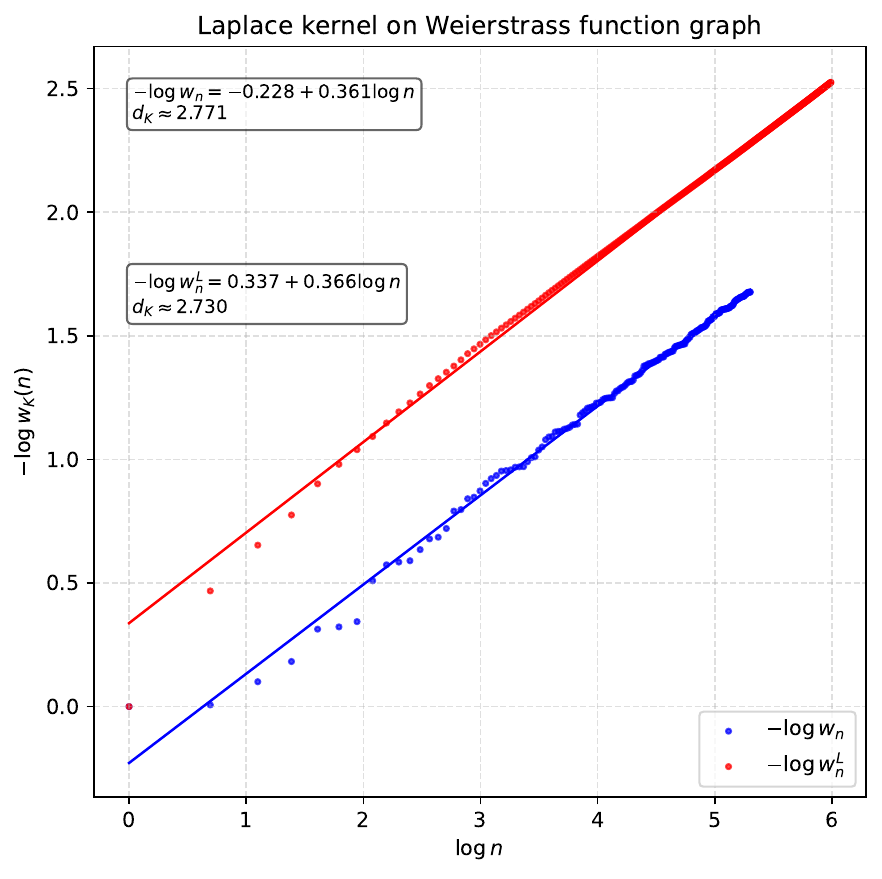}
\end{minipage}\hfill\begin{minipage}[t]{.2\textwidth}
    \centering
    \includegraphics[width=1.0\textwidth]{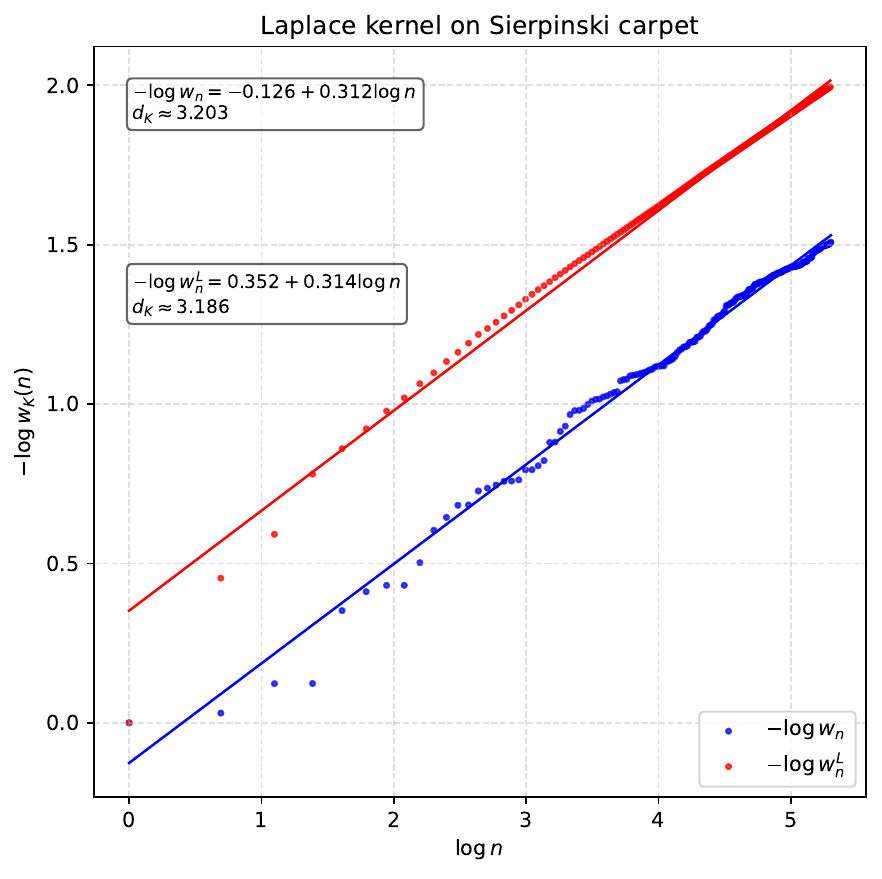}
\end{minipage}\hfill\begin{minipage}[t]{.2\textwidth}
    \centering
    \includegraphics[width=1.0\textwidth]{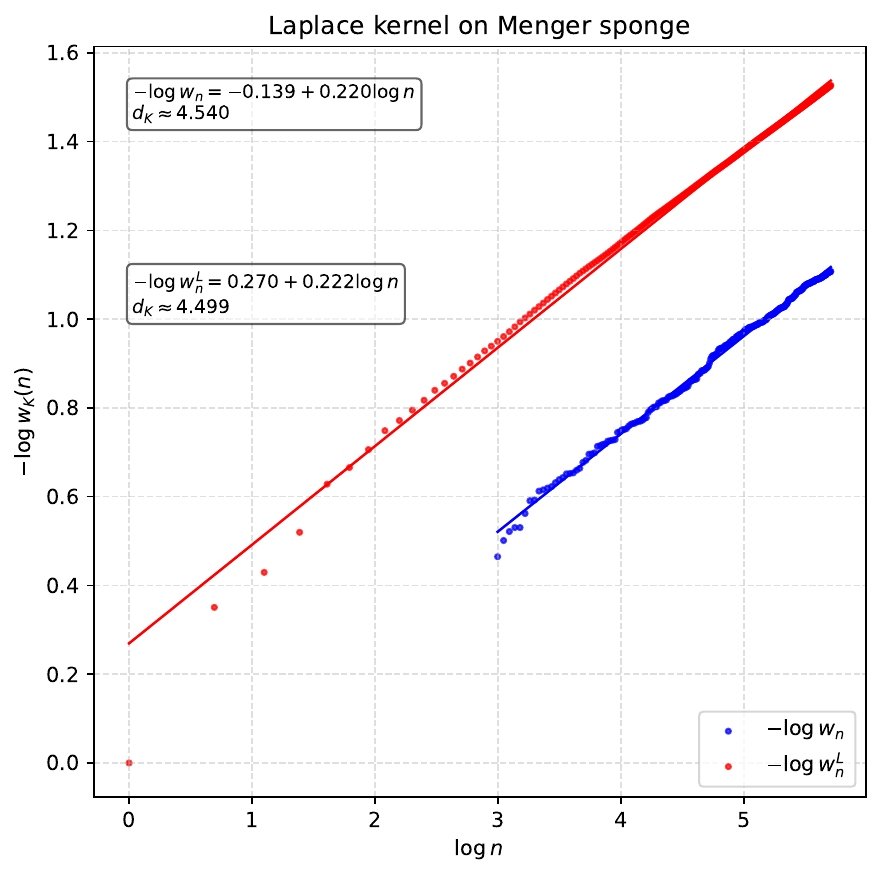}
\end{minipage}\hfill\begin{minipage}[t]{.2\textwidth}
    \centering
    \includegraphics[width=1.0\textwidth]{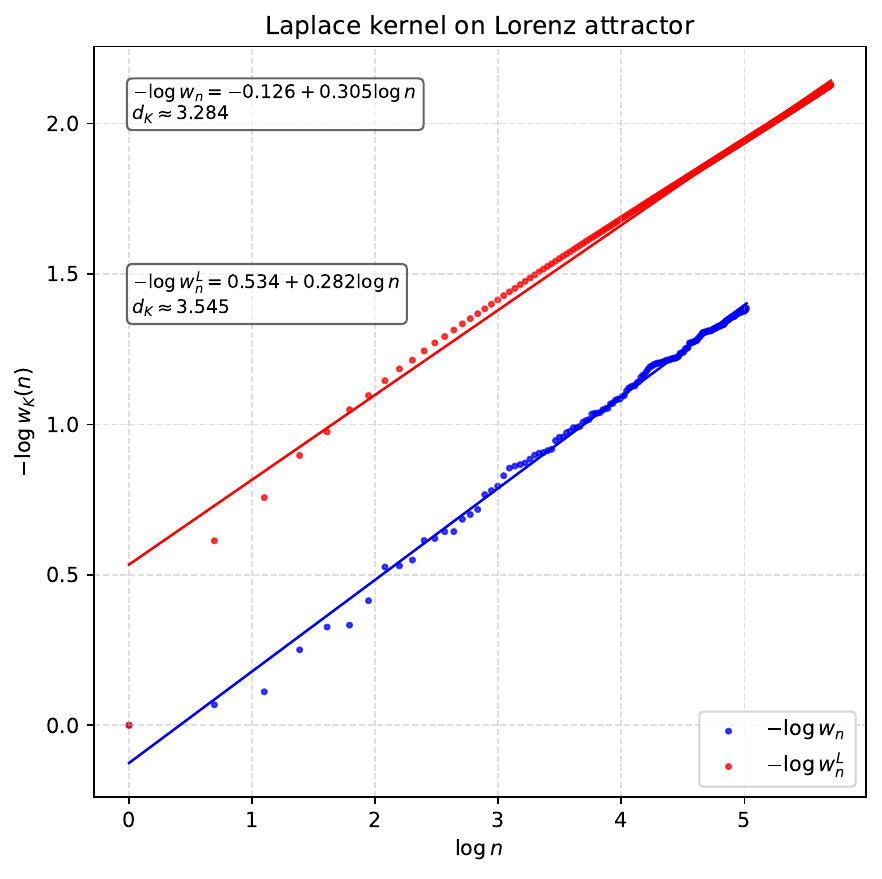}
\end{minipage}
\caption{\small Plot of $-\log(w^L_n)$ (and $-\log(w_n)$) versus $\log n$, where $w^L_n=\sqrt{\sum_{i>n}\hat\lambda_i}$ for the Laplace kernel on various fractal sets. The slope of the fitted line is computed using standard linear regression. Estimated effective dimension is defined as inversely proportional to the slope.}\label{laplaceplots2}    
\end{figure*}

\subsection{Details of experiments with the excess risk decay rates} 
In our experiments, we reduced the constrained KRR problem (which involves optimization over $B_{\mathcal{H}_K}$) to an unconstrained KRR with a regularization term $\lambda \|f\|^2_{\mathcal{H}_K}$, denoted by KRR($\lambda$). 
The optimal value of $\lambda$ was determined via a half-interval search with 30 iterations, ensuring that the solution $f_\lambda$ of KRR($\lambda$) satisfies $\|f_\lambda\|_{\mathcal{H}_K} = 1$. 
This approach is roughly 30 times more expensive than standard KRR, yet remains computationally feasible.
For each kernel, we trained $\hat{f}$ on a training set of size $n$ and evaluated the excess risk on a test set of size $n_{\mathrm{test}} = 10000$, averaging the results over 10 independent trials. 

\subsection{Overestimation and underestimation of $n$-widths}\label{overestimation}
Algorithm~\ref{empirical-width} can significantly overestimate the Kolmogorov $n$-widths. Figure~\ref{exponential-type} shows the estimated upper bounds on the effective dimension obtained using this algorithm. Notably, a substantial overestimation is observed for the Gaussian kernel. The most plausible explanation is that, in this case, the optimal $n$-dimensional subspace $L_n \subseteq \mathcal{H}_K$, which minimizes the expression in equation~\eqref{n-width-def}, cannot be expressed as the span of $n$ kernel sections of the form $K(x, \cdot)$. This representability condition is essential for the output of Algorithm~\ref{empirical-width} to accurately approximate the true $n$-widths.

In the first two plots (corresponding to $a = \frac{1}{2}$ and $a = 1$), we observe that the estimated upper bounds on the effective dimension fall below the true effective dimension as the ambient dimension increases (e.g., for $d = 8$). This underestimation arises from an insufficient sample size --- specifically, a small value of the parameter $|\tilde{\Omega}|=N$ in the finite sample implementation of Algorithm~\ref{empirical-width}.
As follows from Remark~\ref{dim-eps} following Theorem~\ref{finite-omega}, achieving uniform $\varepsilon$-accuracy requires a sample size of order $\mathcal{O}\left(\varepsilon^{-\frac{2(d-1)}{a}} \log \frac{1}{\varepsilon}\right)$
for the kernel $e^{-\|x - y\|^a}$ on the domain $\Omega = \mathbb{S}^{d-1}$. Consequently, as the ambient dimension $d$ increases or the kernel becomes less smooth (i.e., $a$ decreases), the required sample size grows rapidly. This explains the observed decline in the accuracy of Algorithm~\ref{empirical-width} under such conditions. 

In experiments with exponential type kernels we used the following values of parameters: $N=1000000$, $T=1000$.

\begin{figure}[htb]
\begin{minipage}[t]{.33\textwidth}
    \centering
    \includegraphics[width=1.0\textwidth]{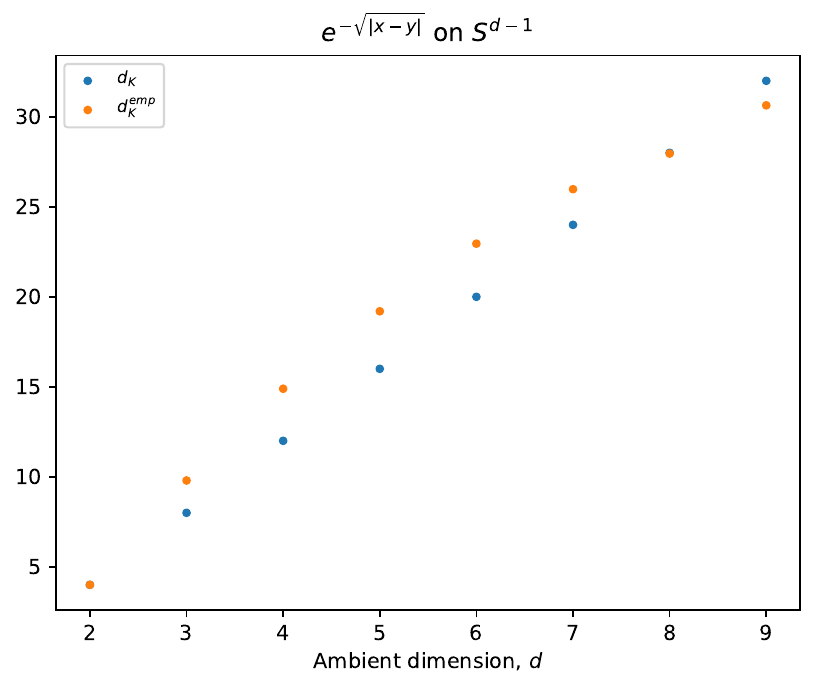}
\end{minipage}\hfill\begin{minipage}[t]{.33\textwidth}
    \centering
    \includegraphics[width=1.0\textwidth]{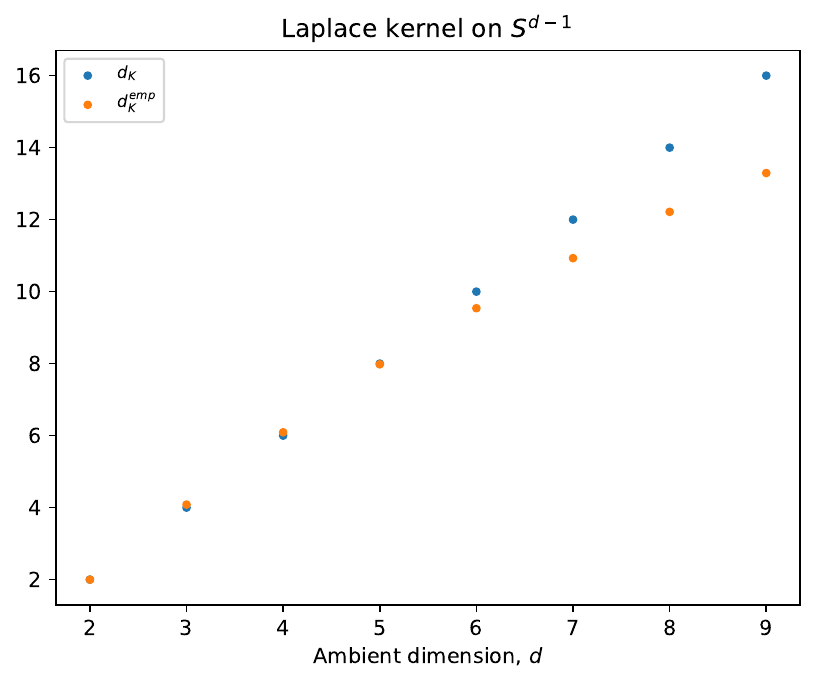}
\end{minipage}\hfill\begin{minipage}[t]{.33\textwidth}
    \centering
    \includegraphics[width=1.0\textwidth]{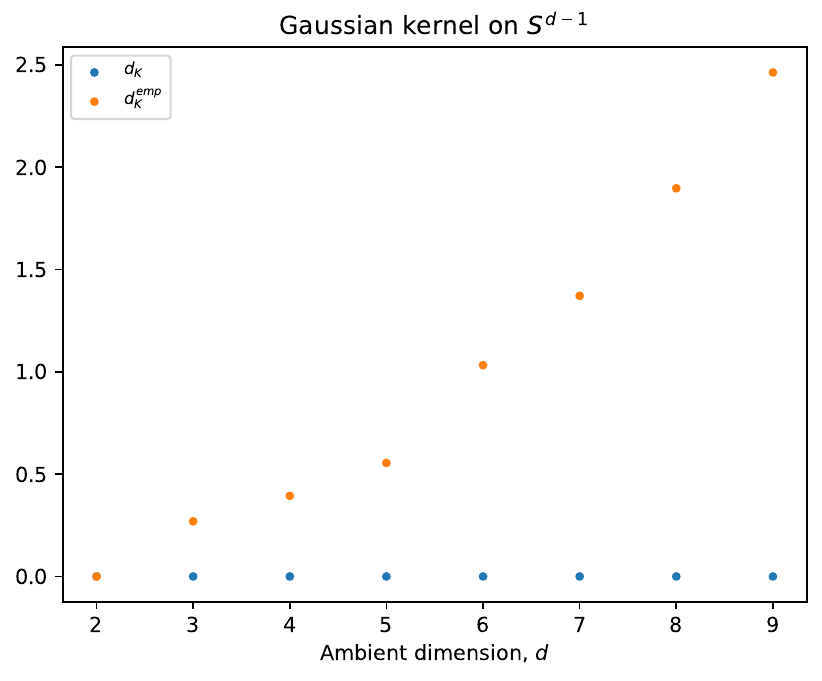}
\end{minipage}\\
\begin{minipage}[t]{.33\textwidth}
    \centering
    \includegraphics[width=1.0\textwidth]{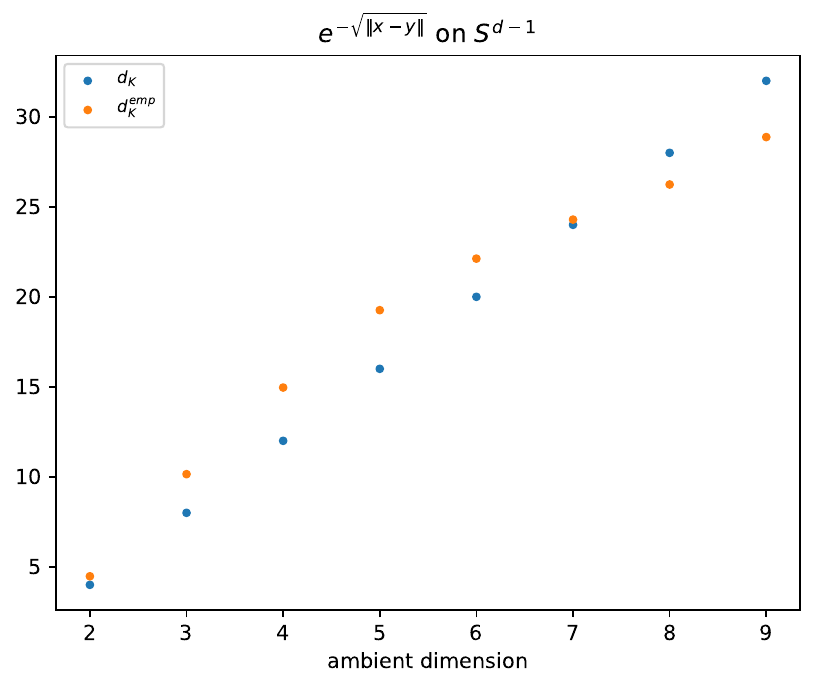}
\end{minipage}\hfill\begin{minipage}[t]{.33\textwidth}
    \centering
    \includegraphics[width=1.0\textwidth]{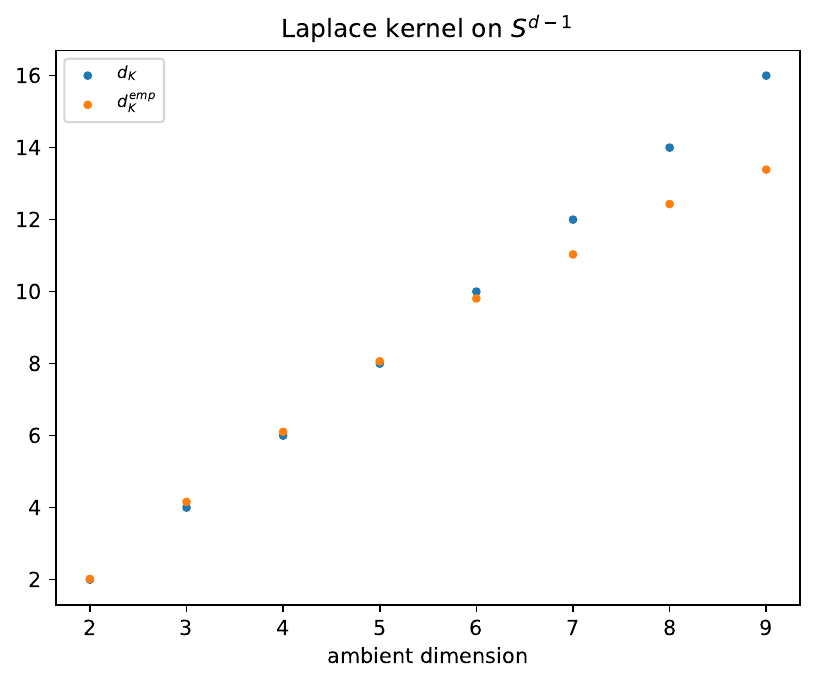}
\end{minipage}\hfill\begin{minipage}[t]{.33\textwidth}
    \centering
    \includegraphics[width=1.0\textwidth]{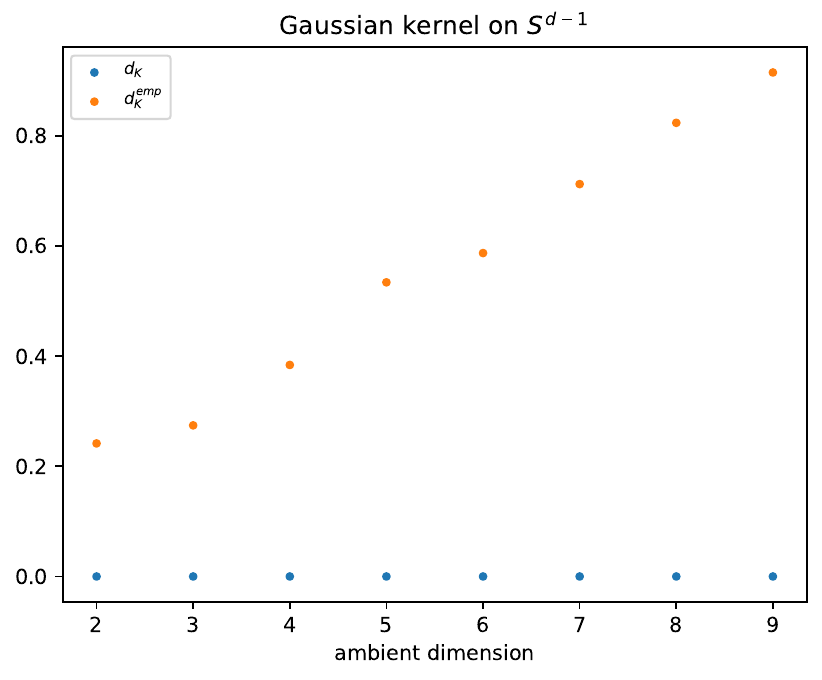}
\end{minipage}
\caption{\small The effective dimension and the empirical upper bound on the effective dimension for exponential type kernels $e^{-\|x-y\|^a}$ on $\Omega={\mathbb S}^{d-1}$ :  $a=\frac{1}{2}$, $a=1$ (the Laplace kernel), $a=2$ (the Gaussian kernel). The first row displays plots computed with the finite sample implementation of Algorithm \ref{empirical-width}, whereas the second row displays plots computed using the L-BFGS--based implementation.}\label{exponential-type}
\end{figure}

\subsection{Experiments with analytically given NNGP and NTK kernels}\label{analytical-NTK}
Analytical expressions for the NNGP and NTK kernels are available for the following activation functions: $\mathds{1}_{x>0}$ (step function), $\max(0,x)$ (ReLU), $\mathrm{erf}$, and $\alpha x\mathds{1}_{x\leq 0} + x\mathds{1}_{x>0}$ (Leaky ReLU)~\cite{pmlr-v244-takhanov24a}.
For each of these kernels (without bias), defined on the domains $\Omega = \mathbb{S}^{d-1}$ with $2 \leq d \leq 6$, we computed the empirical effective dimensions $d_K^{\rm emp}$ using the L-BFGS–based implementation of Algorithm~\ref{empirical-width} with $T = 500$. In every optimization step we invoke L-BFGS 10 or 20 times with independent initializations and output the best point.

After obtaining the empirical $n$-widths $\{w(t)\}_{t=0}^T$, the quantity $d_K^{\rm emp}$ is defined as the reciprocal of $a$, where $a$ is the slope in the dataset $\{(\log t,-\log w(t))\}_{t=300}^{500}$
estimated via the RANSAC algorithm~\cite{FISCHLER1987726}.

Figure~\ref{nngp-relu-effective} shows the computed empirical effective dimensions as functions of $d$. As the plots indicate, $d_K^{\rm emp}$ is recovered with high accuracy for $2 \leq d \leq 5$, but is underestimated for $d=6$. This underestimation arises from the limited value of $T$ in higher dimensions, where the asymptotic relation $w_K(n) \asymp n^{-1/d_K}$ becomes valid only for $n > T$.

\begin{figure}[htb]
\begin{minipage}[t]{.2\textwidth}
    \centering
    \includegraphics[width=1.0\textwidth]{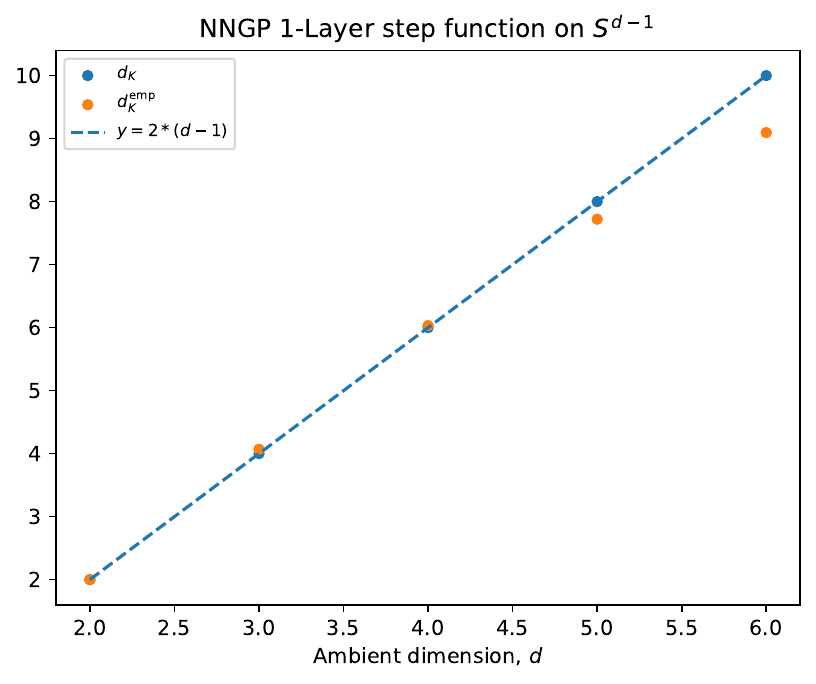}
\end{minipage}\hfill\begin{minipage}[t]{.2\textwidth}
    \centering
    \includegraphics[width=1.0\textwidth]{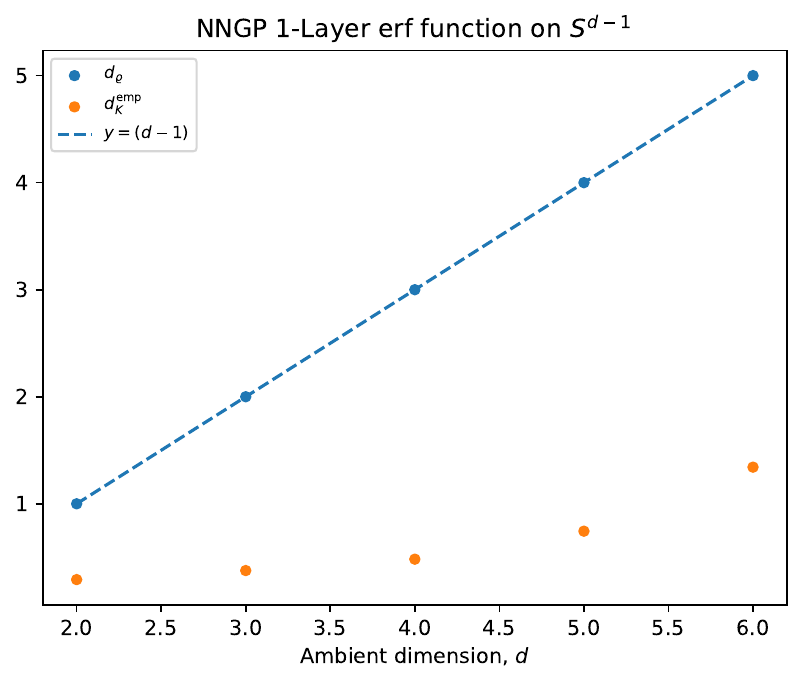}
\end{minipage}\hfill\begin{minipage}[t]{.2\textwidth}
    \centering
    \includegraphics[width=1.0\textwidth]{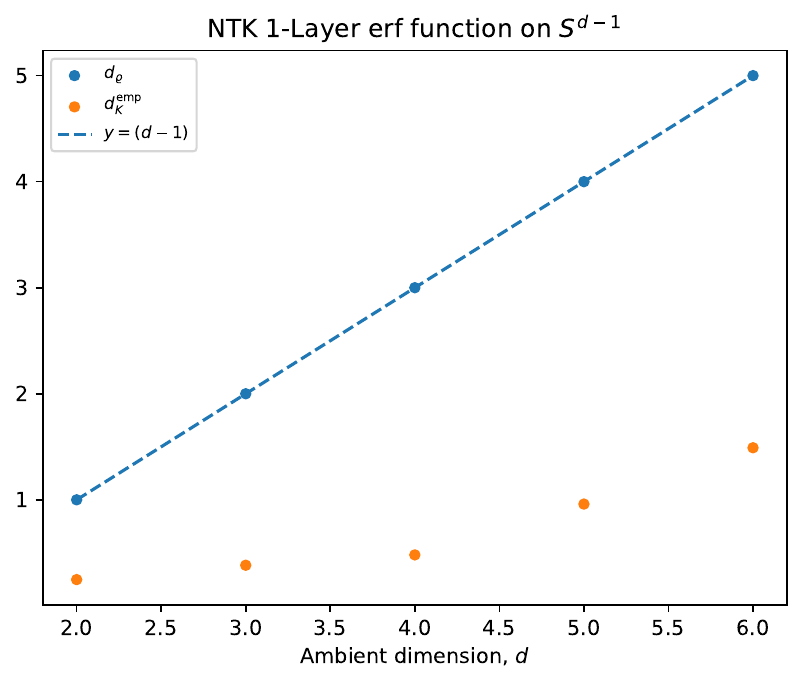}
\end{minipage}\hfill
\begin{minipage}[t]{.2\textwidth}
    \centering
    \includegraphics[width=1.0\textwidth]{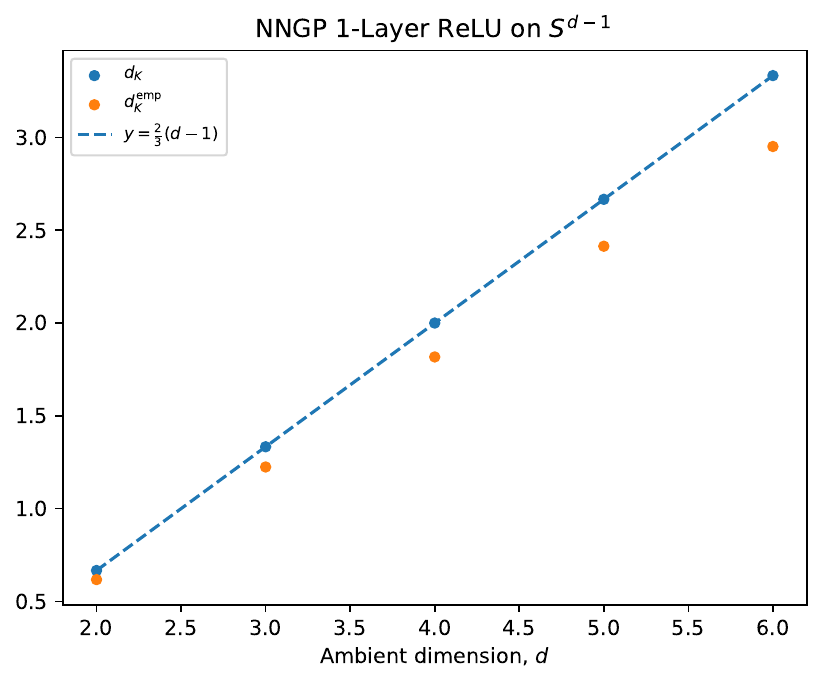}
\end{minipage}\hfill\begin{minipage}[t]{.2\textwidth}
    \centering
    \includegraphics[width=1.0\textwidth]{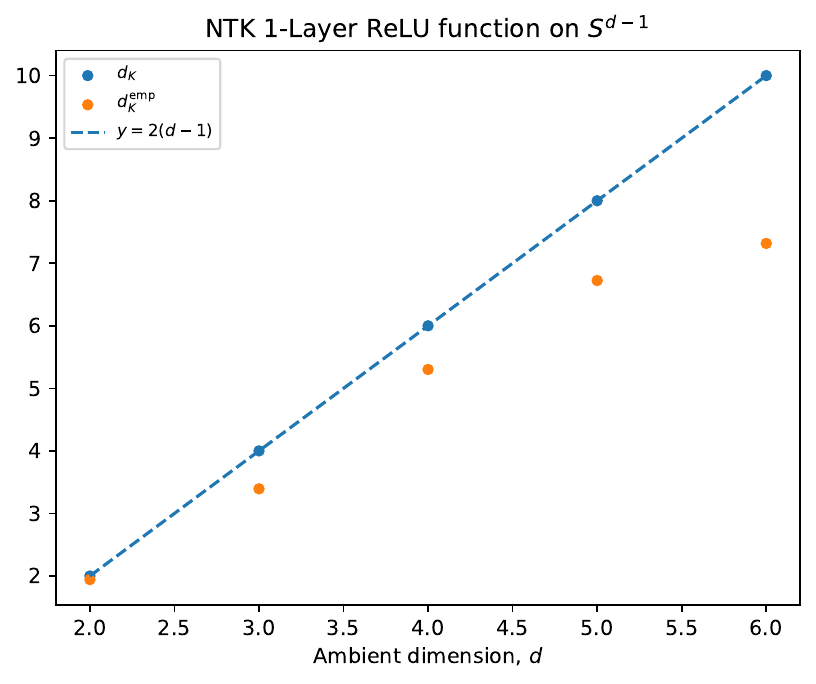}
\end{minipage}
\caption{\small Empirical effective dimensions for infinite-width NNGP and NTK kernels with various activation functions, computed using the L-BFGS–based algorithm. For NNGP ReLU, we have $d_K = \frac{2}{3}(d-1)$, as Table~\ref{rates} indicates. We do not show plots for the Leaky ReLU activation function ($\alpha<1$) due to their similarity to the ReLU case.}\label{nngp-relu-effective}
\end{figure}

On some pictures, along with empirical $n$-widths,  we draw a linear plot for $d_\varrho$. The kernel-based upper Minkowski dimension $d_\varrho$ can be expressed as $d_\varrho = c\, d_H$, where $d_H$ denotes the Euclidean upper Minkowski dimension (which equals $d-1$ in our setup). We call the coefficient $c$ the \emph{metric dimension factor}. If the kernel $K$ satisfies 
$\sqrt{K(x,x)+K(y,y)-2K(x,y)} \asymp \|x-y\|^a$,
then $c=\tfrac{1}{a}$. We refer to kernels with $c=1$ as \emph{type I kernels}, and to those with $c=2$ as \emph{type II kernels}.  

For all activation functions considered, the corresponding NNGP kernels are of type I. Among NTK kernels, those induced by $\tanh$, sigmoid, $\operatorname{erf}$, and Gaussian activations are of type I, whereas those induced by ReLU and Leaky ReLU activations are of type II.


\subsection{Experiments with NNGP and NTK kernels of finite width}\label{finite-width}
Let $\sigma: \mathbb{R} \to \mathbb{R}$ be an activation function. Define the output of a one-hidden-layer neural network as
$$
f_1(x \mid \mathbf{w}, W^{(1)}) = \frac{1}{\sqrt{n_1}}\mathbf{w}^\top \sigma(W^{(1)} x),
$$
where $W^{(1)} \in \mathbb{R}^{n_1 \times d}$ and $\mathbf{w} \in \mathbb{R}^{n_1}$.

The random 1-layer NNGP kernel of width $n_1$ on the unit sphere $\mathbb{S}^{d-1}$ is defined as
\begin{equation*}
\begin{split}
&\mathrm{NNGP}_{\sigma}(x, y \mid W^{(1)}) = \mathbb{E}_{\mathbf{w} \sim \mathcal{N}(0, I_{n_1})} f_1(x \mid \mathbf{w}, W^{(1)}) f_1(y \mid \mathbf{w}, W^{(1)}) = \\
&\frac{1}{n_1} \sigma(W^{(1)} x)^\top \sigma(W^{(1)} y),
\end{split}
\end{equation*}
where the entries of $W^{(1)}$ are sampled i.i.d. from $\mathcal{N}(0, 1)$. This kernel serves as a Monte Carlo approximation of the infinite-width NNGP kernel (without bias):
$$
\mathbb{E}_{\omega \sim \mathcal{N}(0, I_d)}[\sigma(\omega^\top x)\sigma(\omega^\top y)].
$$
The random 1-layer NTK kernel (without bias) of width $n_1$ is defined by the expected inner product of the gradients of the network output with respect to all parameters:
$$
\mathrm{NTK}_{\sigma}(x, y \mid W^{(1)}) = \mathbb{E}_{\mathbf{w} \sim \mathcal{N}(0,  I_{n_1})} \left\langle \nabla_{(\mathbf{w}, W^{(1)})} f_1(x), \nabla_{(\mathbf{w}, W^{(1)})} f_1(y) \right\rangle,
$$
with the same initialization for $W^{(1)}$ as above.

Now define the 2-layer network:
$$
f_2(x \mid \mathbf{w}, W^{(1)}, W^{(2)}) = \frac{1}{\sqrt{n_1}}\mathbf{w}^\top \sigma(\frac{1}{\sqrt{n_1}} W^{(2)} \sigma(W^{(1)} x)),
$$

where $W^{(2)} \in \mathbb{R}^{n_1 \times n_1}$ has i.i.d. entries sampled from $\mathcal{N}(0, 1)$.

The random 2-layer NNGP kernel (without bias) of width $n_1$ is given by
\begin{equation*}
\begin{split}
&\mathrm{NNGP}_{\sigma}(x, y \mid W^{(1)}, W^{(2)}) = \mathbb{E}_{\mathbf{w} \sim \mathcal{N}(0, I_{n_1})} f_2(x \mid \mathbf{w}, W^{(1)}, W^{(2)}) f_2(y \mid \mathbf{w}, W^{(1)}, W^{(2)}) \\
&= \frac{1}{n_1} \sigma(\frac{1}{\sqrt{n_1}} W^{(2)} \sigma(W^{(1)} x))^\top \sigma(\frac{1}{\sqrt{n_1}} W^{(2)} \sigma(W^{(1)} y)).
\end{split}
\end{equation*}
The random 2-layer NTK kernel (without bias) of width $n_1$ is defined analogously as
\begin{equation*}
\begin{split}
\mathrm{NTK}_{\sigma}(x, y \mid W^{(1)}, W^{(2)}) = \mathbb{E}_{\mathbf{w} \sim \mathcal{N}(0, I_{n_1})} \left\langle \nabla_{(\mathbf{w}, W^{(1)}, W^{(2)})} f_2(x), \nabla_{(\mathbf{w}, W^{(1)}, W^{(2)})} f_2(y) \right\rangle.
\end{split}
\end{equation*}
Our definition of NTK kernels is generally consistent with the original formulation~\cite{Jacot}, with the addition of an explicit expectation over the weights of the final layer. This modification simplifies kernel computation and slightly reduces the variance arising from the random sampling of the remaining weights.
We applied the Algorithm~\ref{empirical-width} to NNGP and NTK kernels of width $n_1=1000$ on ${\mathbb S}^{d-1}$ for different activation functions. 
The parameter values that we used was $N=500000$, $T=500$. 
Results for $n_1=1000$, $N=500000$, $T=500$ are shown on Figure~\ref{ntk-plots}. 

For slightly larger parameter settings ($N=1000000$, $T=1000$), Algorithm~\ref{empirical-width} applied to the Laplace kernel began to underestimate $d_K$ starting from dimension $d=6$. Therefore, we present results only for dimensions up to 6 in the plots. As expected, the empirical upper metric dimension is consistently greater than the empirical effective dimension.

For the infinite-width NTK with ReLU activation, it is known that $d_\varrho = d_K = 2d - 2$, similarly to the Laplace kernel. However, in our experiments, the empirical estimate $d_K^{\rm emp}$ lies significantly below $d_\varrho^{\rm emp}$. This discrepancy is likely due to two factors: the finite network width ($n_1 = 1000$), and the underestimation of $n$-widths caused by the limited sample size $N$.

Based on these observations, we conjecture that for finite-width NTK ReLU kernels, the inequality $d_K < d_\varrho$ holds. As expected, the empirical effective dimensions of NTK ReLU kernels are approximately twice those of the corresponding NNGP ReLU kernels. 
Results for other activation functions align with known behavior: the $n$-widths of NTK kernels are typically only slightly larger than those of the corresponding NNGP kernels. These experiments are intended solely to demonstrate the efficiency of Algorithm~\ref{empirical-width} when applied to NNGP and NTK kernels. 
From these preliminary experiments, we observe that the main bottleneck in applying Algorithm~\ref{empirical-width} lies in the difficulty of computing such kernels exactly, due to the absence of closed-form expressions. Even powerful tools such as the \texttt{neural-tangents} library~\cite{neuraltangents2020,han2022fast} provide only low-rank approximations (typically of rank 30–40), which leads to a collapse of the computed empirical $n$-widths.
A comprehensive study of the $n$-widths associated with these kernels is left for future work. 

\begin{figure}[htb]
\begin{minipage}[t]{.2\textwidth}
    \centering
    \includegraphics[width=1.0\textwidth]{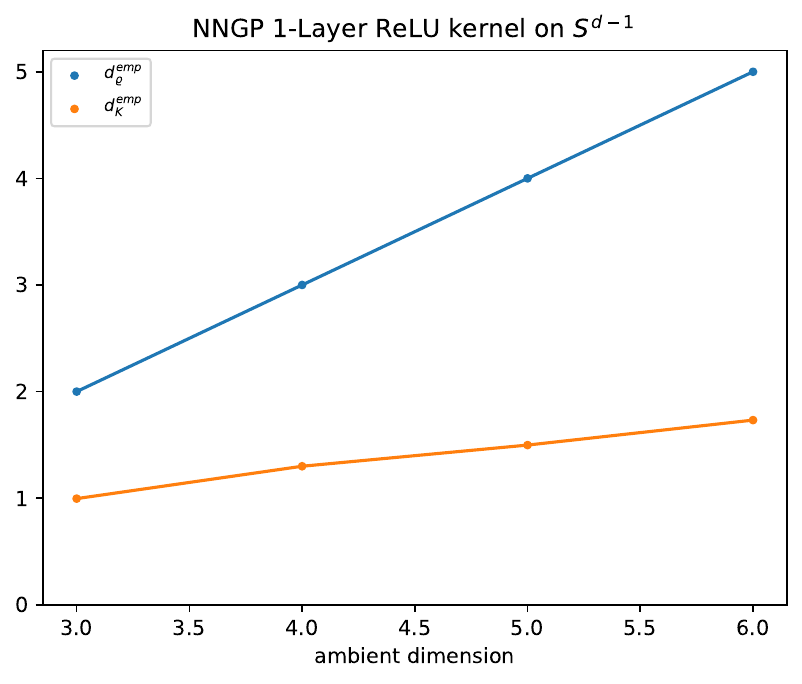}
\end{minipage}\hfill\begin{minipage}[t]{.2\textwidth}
    \centering
    \includegraphics[width=1.0\textwidth]{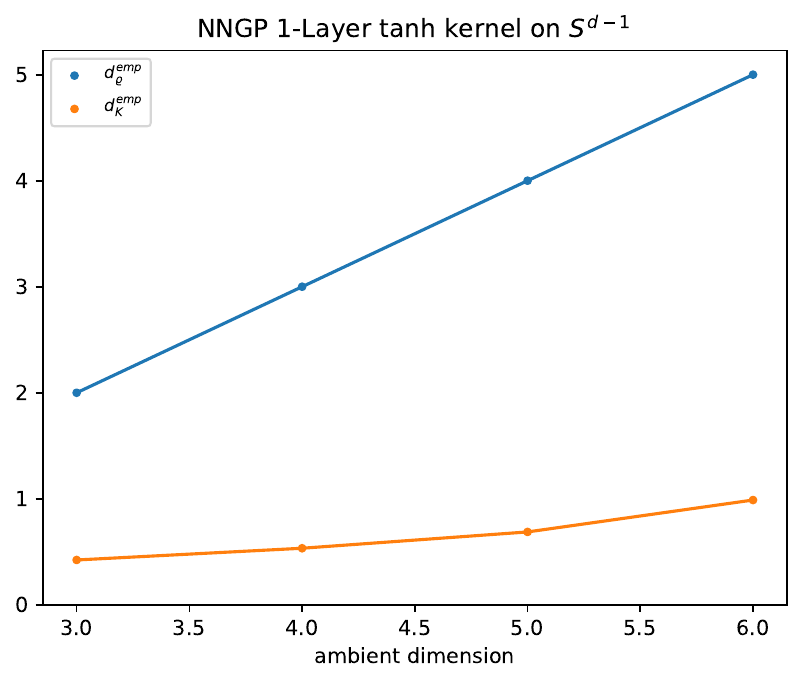}
\end{minipage}\hfill\begin{minipage}[t]{.2\textwidth}
    \centering
    \includegraphics[width=1.0\textwidth]{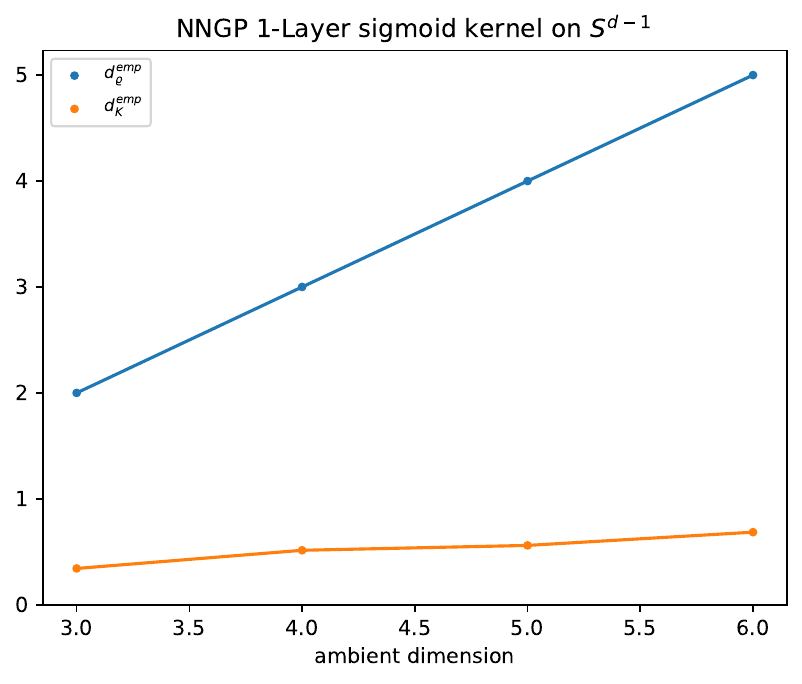}
\end{minipage}\hfill\begin{minipage}[t]{.2\textwidth}
    \centering
    \includegraphics[width=1.0\textwidth]{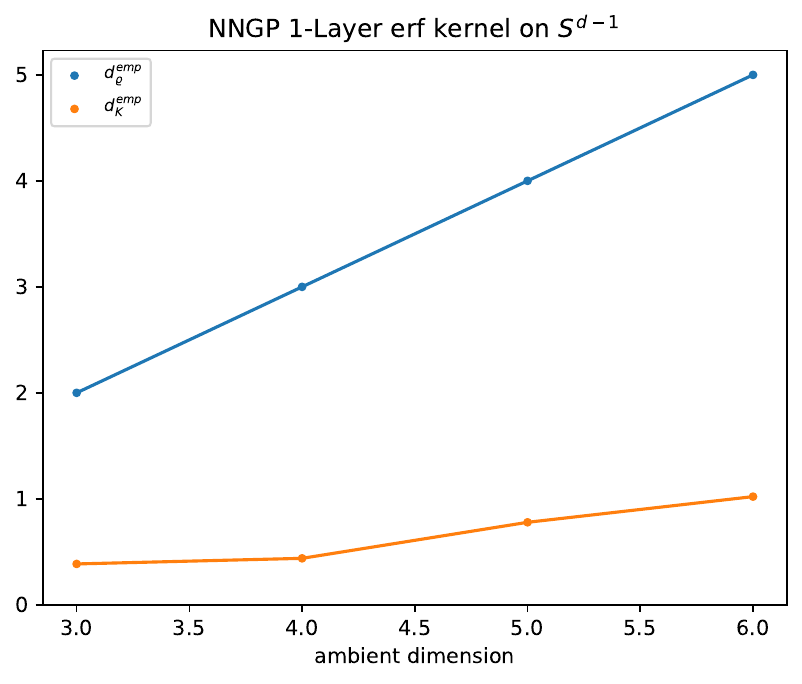}
\end{minipage}\hfill\begin{minipage}[t]{.2\textwidth}
    \centering
    \includegraphics[width=1.0\textwidth]{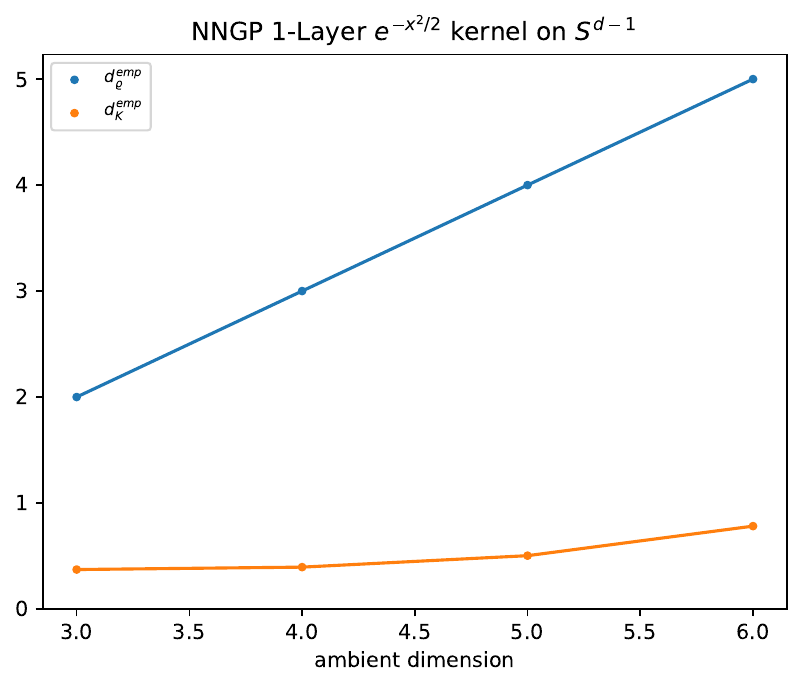}
\end{minipage}\\
\begin{minipage}[t]{.2\textwidth}
    \centering
    \includegraphics[width=1.0\textwidth]{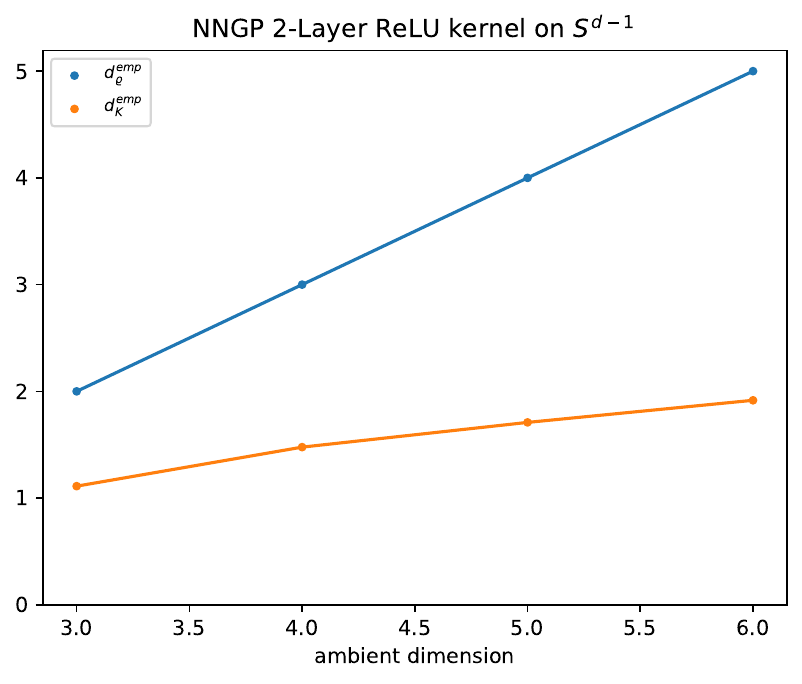}
\end{minipage}\hfill\begin{minipage}[t]{.2\textwidth}
    \centering
    \includegraphics[width=1.0\textwidth]{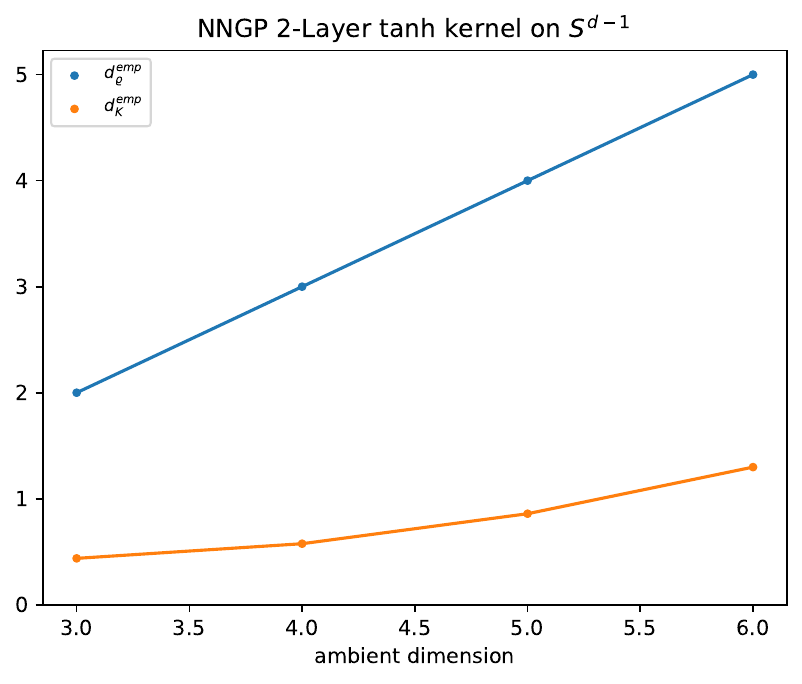}
\end{minipage}\hfill\begin{minipage}[t]{.2\textwidth}
    \centering
    \includegraphics[width=1.0\textwidth]{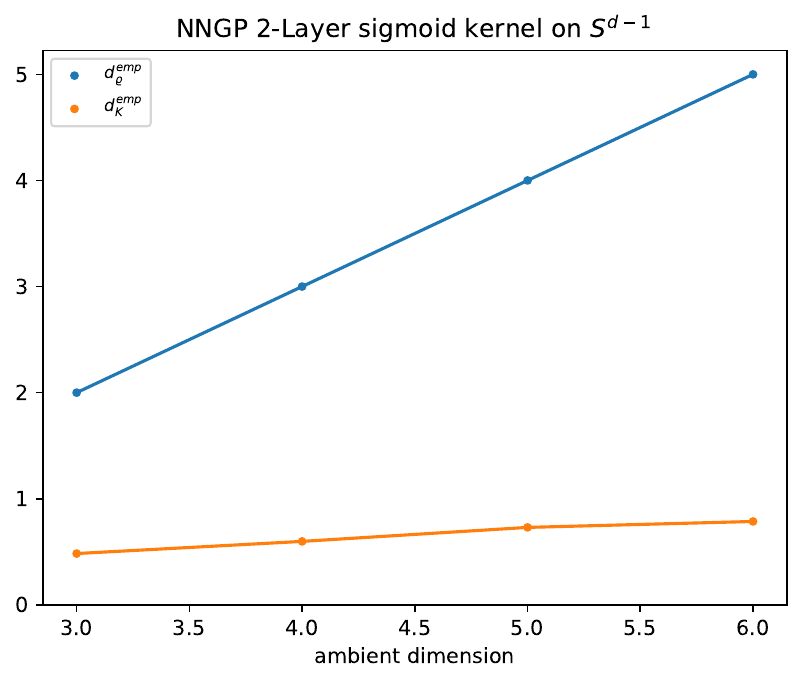}
\end{minipage}\hfill\begin{minipage}[t]{.2\textwidth}
    \centering
    \includegraphics[width=1.0\textwidth]{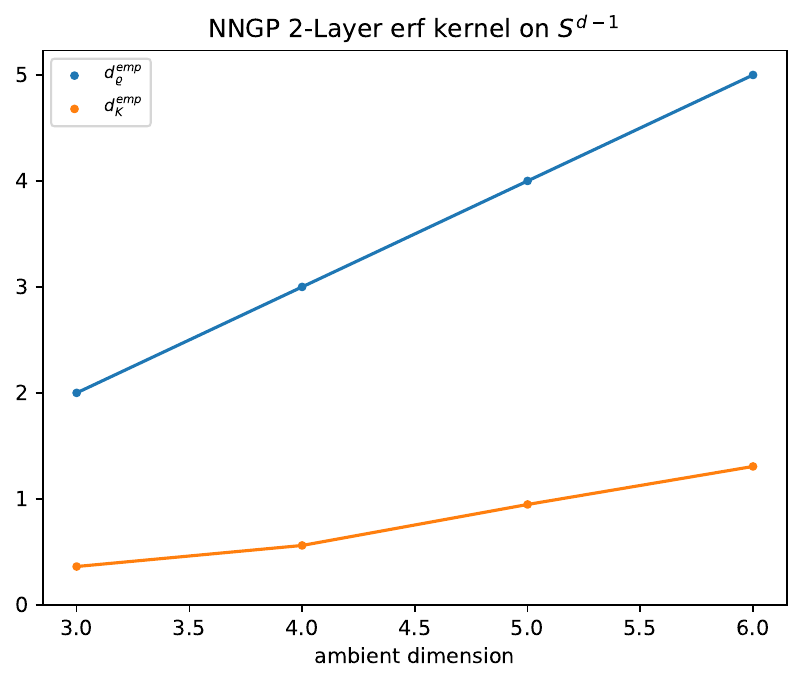}
\end{minipage}\hfill\begin{minipage}[t]{.2\textwidth}
    \centering
    \includegraphics[width=1.0\textwidth]{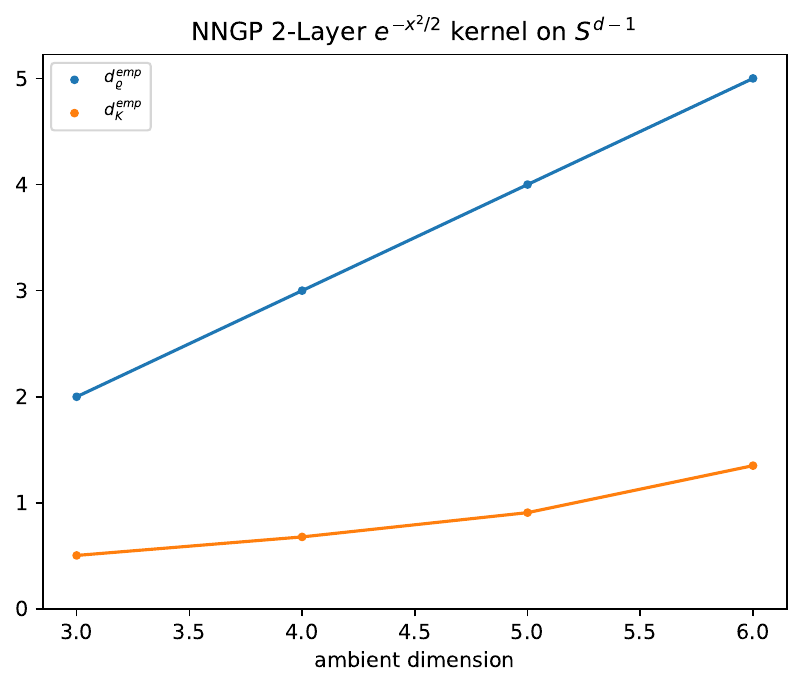}
\end{minipage}\\
\begin{minipage}[t]{.2\textwidth}
    \centering
    \includegraphics[width=1.0\textwidth]{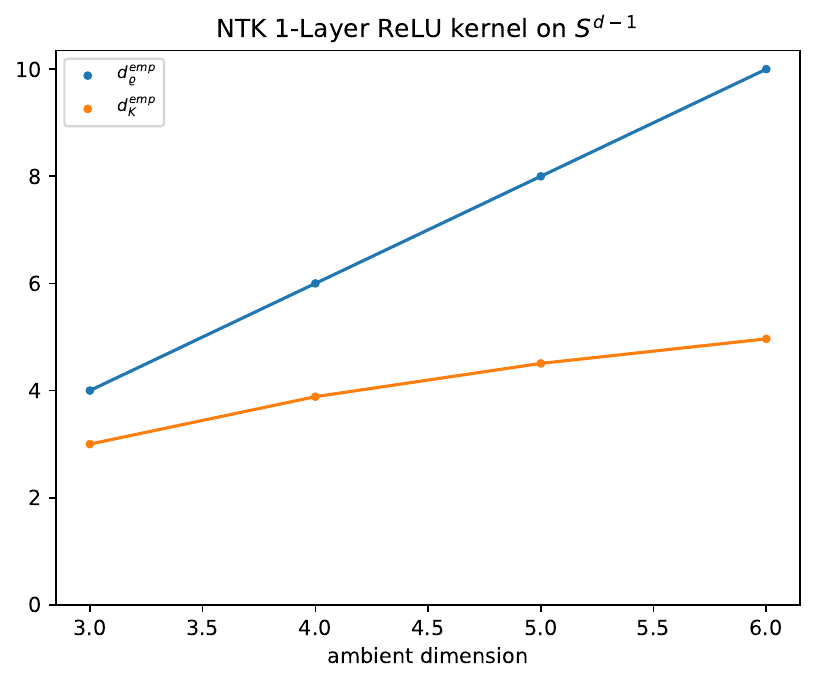}
\end{minipage}\hfill\begin{minipage}[t]{.2\textwidth}
    \centering
    \includegraphics[width=1.0\textwidth]{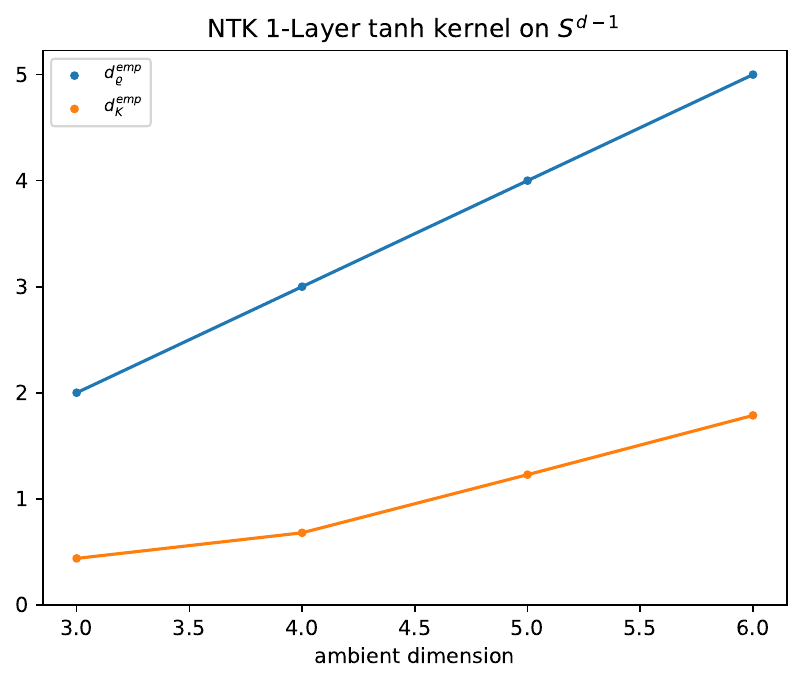}
\end{minipage}\hfill\begin{minipage}[t]{.2\textwidth}
    \centering
    \includegraphics[width=1.0\textwidth]{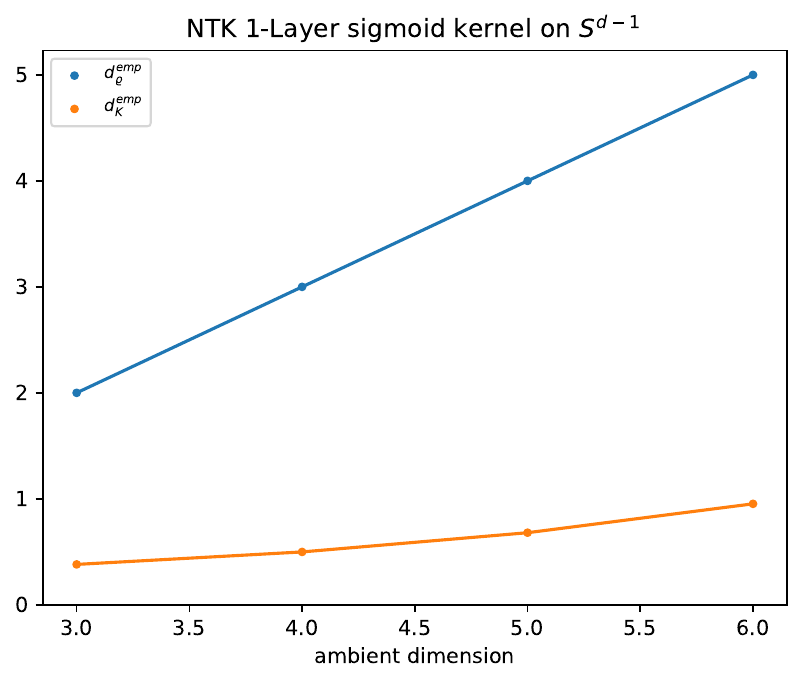}
\end{minipage}\hfill\begin{minipage}[t]{.2\textwidth}
    \centering
    \includegraphics[width=1.0\textwidth]{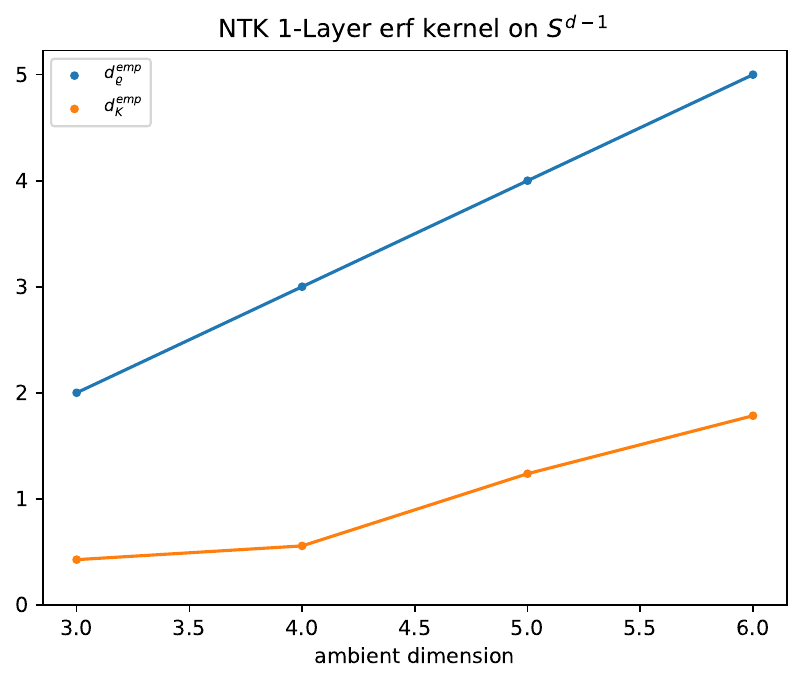}
\end{minipage}\hfill\begin{minipage}[t]{.2\textwidth}
    \centering
    \includegraphics[width=1.0\textwidth]{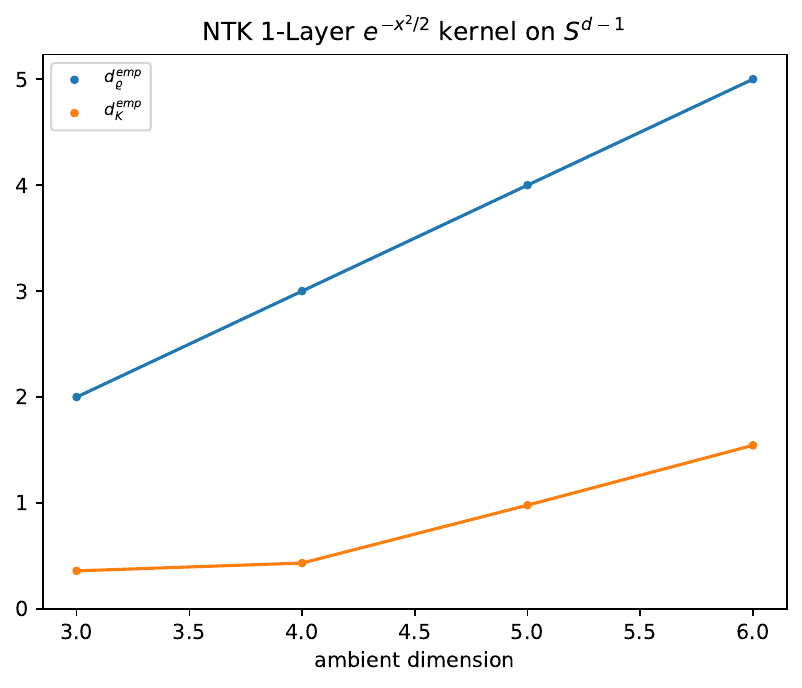}
\end{minipage}\\
\begin{minipage}[t]{.2\textwidth}
    \centering
    \includegraphics[width=1.0\textwidth]{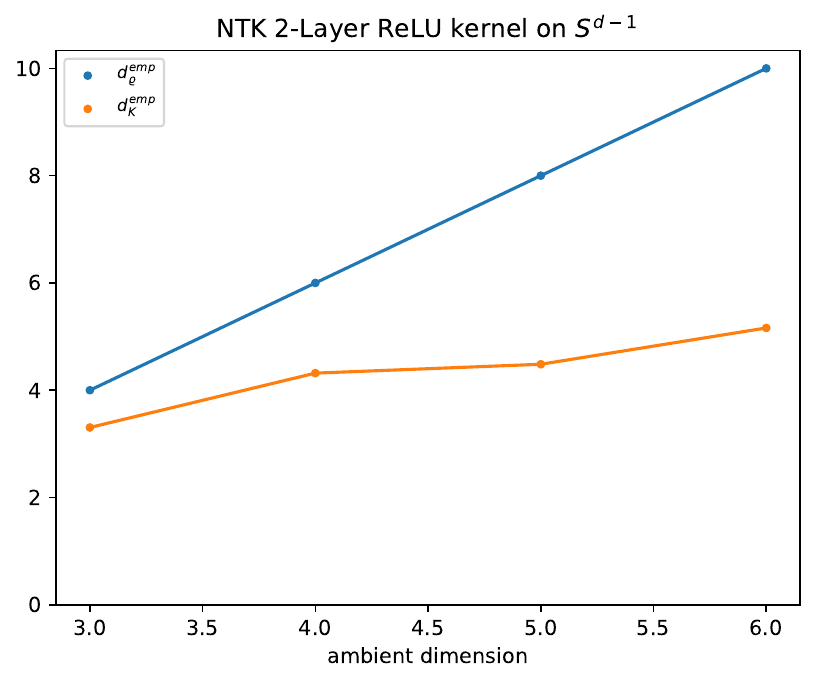}
\end{minipage}\hfill\begin{minipage}[t]{.2\textwidth}
    \centering
    \includegraphics[width=1.0\textwidth]{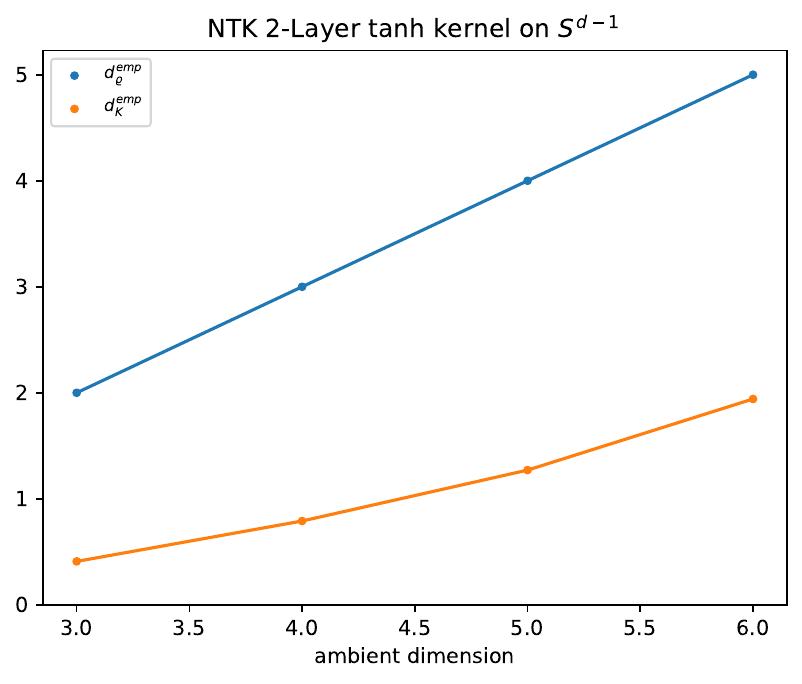}
\end{minipage}\hfill\begin{minipage}[t]{.2\textwidth}
    \centering
    \includegraphics[width=1.0\textwidth]{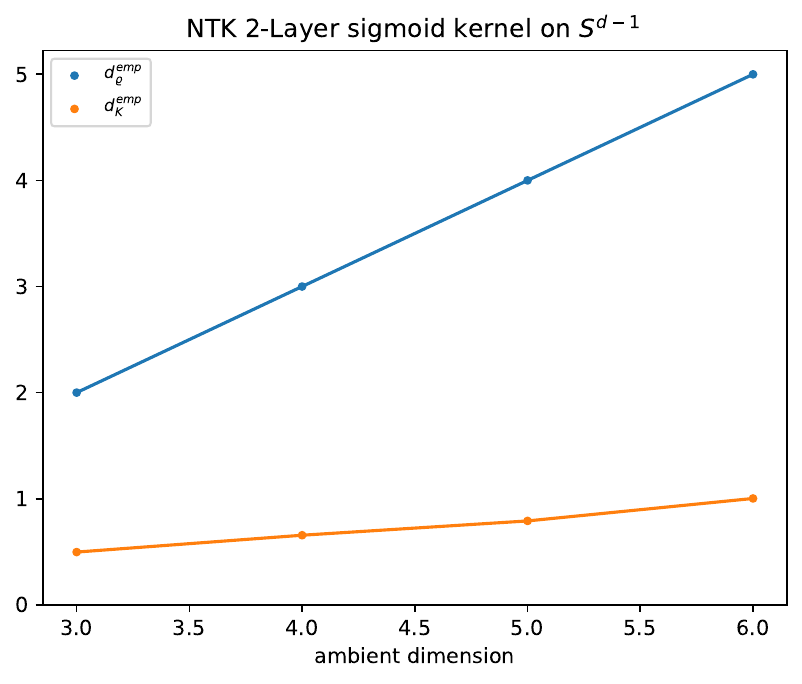}
\end{minipage}\hfill\begin{minipage}[t]{.2\textwidth}
    \centering
    \includegraphics[width=1.0\textwidth]{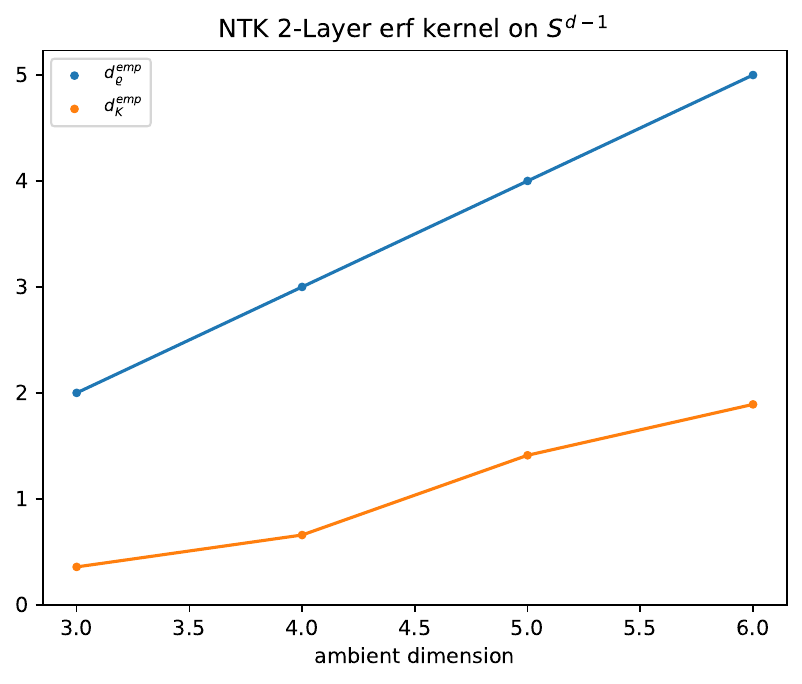}
\end{minipage}\hfill\begin{minipage}[t]{.2\textwidth}
    \centering
    \includegraphics[width=1.0\textwidth]{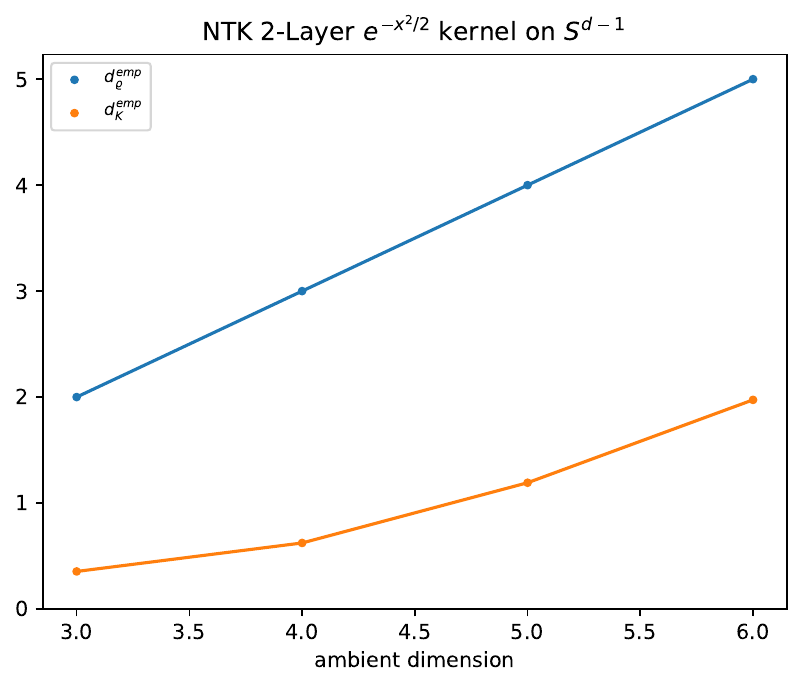}
\end{minipage}
\caption{Empirical upper metric dimensions and empirical effective dimensions for NNGP and NTK kernels of width $n_1=1000$ on ${\mathbb S}^{d-1}$.}\label{ntk-plots}    
\end{figure}

The computational experiments were performed on a system equipped with an Intel Core i9-10900X CPU operating at 3.70 GHz, running Ubuntu 22.04.1 LTS. The software environment included pandas version 2.0.0, numpy version 1.26.0, and seaborn version 0.13.0. Our code is available on \href{https://github.com/anonaistats2026/intrinsic}{github} to facilitate the reproducibility of the results.

\section{Proofs for Section~\ref{dim-intro}: cases of $d_K = d_\varrho$}
We begin with a simple lemma that provides a sufficient condition for the equality $d_\varrho = d_K$.
\begin{lemma}\label{isma-simple} To establish $d_K = d_\varrho$, it suffices to construct a Borel probability measure $\mu$ on $\Omega$ such that $\lambda_{n}({\rm O}_{K,\mu})\gg n^{-1-\frac{2}{d_\varrho}}$. 
\end{lemma}
\begin{proof}
From the definition of $d_K$, for any $\varepsilon > 0$, there exists $N_\varepsilon$ such that
$$
w_K(n) < n^{-\frac{1}{d_K + \varepsilon}} \quad \text{for all } n > N_\varepsilon.
$$
Since $d_K \leq d_\varrho$, it follows that
$$
w_K(n) < n^{-\frac{1}{d_\varrho + \varepsilon}}.
$$
By Ismagilov's theorem, the eigenvalue tail of the integral operator satisfies
$$
\sum_{k = n+1}^{\infty} \lambda_k(\mathcal{O}_{K,\mu}) \leq w_K(n)^2 \leq n^{-\frac{2}{d_\varrho + \varepsilon}}.
$$
On the other hand, if we assume $\lambda_n(\mathcal{O}_{K,\mu}) \gg n^{-1 - \frac{2}{d_\varrho}}$, then summing the tail gives
$$
\sum_{k = n+1}^{\infty} \lambda_k(\mathcal{O}_{K,\mu}) \gg \int_n^\infty x^{-1 - \frac{2}{d_\varrho}} dx \asymp n^{-\frac{2}{d_\varrho}}.
$$
Combining the upper and lower bounds,
$$
n^{-\frac{2}{d_\varrho}} \ll w_K(n)^2 \leq n^{-\frac{2}{d_\varrho + \varepsilon}},
$$
which is only possible if $d_K = d_\varrho$.
\end{proof}

We now show that the desired eigenvalue decay holds for regular domains and certain translation-invariant kernels under the uniform measure.

\begin{lemma}\label{translation-invariant} 
Let $\Omega \subseteq \mathbb{R}^d$ be a Riemann measurable set with positive volume, and consider a translation-invariant kernel $K(\mathbf{x}, \mathbf{y}) = k(\mathbf{x} - \mathbf{y})$ such that its Fourier transform
$$
\widehat{k}(\boldsymbol{\omega}) = \int_{\mathbb{R}^d} k(\mathbf{x}) e^{-i \boldsymbol{\omega}^\top \mathbf{x}} \, d\mathbf{x}
$$
is bounded, nonnegative, and satisfies
$$
\liminf_{\|\boldsymbol{\omega}\| \to \infty} \|\boldsymbol{\omega}\|^{d+a} \widehat{k}(\boldsymbol{\omega}) > 0
$$
for some $a > 0$. Then,
$$
\lambda_n(\mathcal{O}_{K,\mu}) \gg n^{-1 - \frac{a}{d}},
$$
where $\mu$ is the uniform probability measure on $\Omega$.
\end{lemma}
We will use the following classical result due to Widom.

\begin{theorem}[\cite{922bec1d8af6}] 
Let $\Omega \subseteq \mathbb{R}^d$ be Riemann measurable with nonzero volume, $\mu$ the uniform distribution on $\Omega$, and $K(\mathbf{x}, \mathbf{y}) = k(\mathbf{x} - \mathbf{y})$ with $\widehat{k}$ bounded, nonnegative, and vanishing at infinity. Then
$$
\lambda_n(\mathcal{O}_{K, \mu}) \asymp \frac{1}{\mathrm{vol}(\Omega)} \, \phi_0\left( \frac{(2\pi)^d}{\mathrm{vol}(\Omega)} n \right),
$$
where $\phi_0 : [0,\infty) \to [0,\infty)$ is a nonincreasing function equimeasurable with $\widehat{k}$.
\end{theorem}
This result plays an important role in the theory of translation-invariant kernels, e.g. it can be used to determine the rate of uniform convergence of a truncated Mercer series to the kernel~\cite{ARXIV2205}.

\begin{proof}[Proof of Lemma~\ref{translation-invariant}] 
Let $\phi_0^{-1}(t) := \sup\{x > 0 \mid \phi_0(x) > t\}$. Since $\phi_0$ is equimeasurable with $\widehat{k}$, we have
$$
\mathrm{vol}\left( \{ \boldsymbol{\omega} \mid \widehat{k}(\boldsymbol{\omega}) > t \} \right) = \phi_0^{-1}(t).
$$
From the assumption $\widehat{k}(\boldsymbol{\omega}) \geq c \|\boldsymbol{\omega}\|^{-d - a}$ for $\|\boldsymbol{\omega}\| > C$, it follows that
$$
\mathrm{vol}\left( \{ \boldsymbol{\omega} \mid \widehat{k}(\boldsymbol{\omega}) > t \} \right) + \omega_d C^d \geq \mathrm{vol}\left( \left\{ \boldsymbol{\omega} \mid \frac{c}{\|\boldsymbol{\omega}\|^{d+a}} > t \right\} \right) = \omega_d \left( \frac{c}{t} \right)^{\frac{d}{d+a}},
$$
where $\omega_d = \frac{\pi^{\frac{d}{2}}}{\Gamma(\frac{d}{2}+1)}$ is the volume of the unit ball in $\mathbb{R}^d$. Hence,
$$
\phi_0^{-1}(t) \geq \omega_d \left( \frac{c}{t} \right)^{\frac{d}{d+a}} - \omega_d C^d \gg t^{-\frac{d}{d+a}} \quad \text{as } t \to 0+.
$$
This implies $\phi_0(t) \gg t^{-\frac{d+a}{d}}$ as $t \to \infty$. Substituting into Widom’s result gives
$$
\lambda_n(\mathcal{O}_{K, \mu}) \gg n^{-\frac{d+a}{d}} = n^{-1 - \frac{a}{d}},
$$
as claimed.
\end{proof}

\subsection{Proof of Theorem~\ref{width-cases}: the Laplace kernel case}
The Fourier transform of $k({\mathbf x}) = e^{-\gamma\|{\mathbf x}\|}$ is $\widehat{k}(\boldsymbol{\omega}) = c_d\frac{\gamma}{(\gamma^2+\|\boldsymbol{\omega}\|^2)^{\frac{d+1}{2}}}$ where $c_d=2^d\pi^{\frac{d-1}{2}} \Gamma(\frac{d+1}{2})$. A nonincreasing function $\phi_0$ equimeasurable with $\widehat{k}$ satisfies
\begin{equation*}
\begin{split}
{\rm vol}(\{\boldsymbol{\omega}\mid c_d\frac{\gamma}{(\gamma^2+\|\boldsymbol{\omega}\|^2)^{\frac{d+1}{2}}}>t\}) = \phi_0^{-1}(t).
\end{split}
\end{equation*}
Therefore, $\phi_0^{-1}(t) = \omega_d\left((\frac{c_d\gamma}{t})^{\frac{2}{d+1}}-\gamma^2\right)^{\frac{d}{2}}$ where $\omega_d = \frac{\pi^{\frac{d}{2}}}{\Gamma(\frac{d}{2}+1)}$. For $t\to +0$ we have $\phi_0^{-1}(t) \asymp t^{-\frac{d}{d+1}}$, and therefore, $\phi_0(t)\asymp t^{-\frac{d+1}{d}}$ as $t\to +\infty$. From Widom's result we obtain $\lambda_n ({\rm O}_{K, \mu})\asymp n^{-\frac{d+1}{d}}$. Since $d_\varrho = 2d$, we obtain $\lambda_n ({\rm O}_{K, \mu})\asymp n^{-1-\frac{2}{d_\varrho}}$. Lemma~\ref{isma-simple} implies $d_K = d_\varrho$.

\subsection{Proof of Theorem~\ref{width-cases}: the kernel $K({\mathbf x}, {\mathbf y}) = e^{-\gamma\|{\mathbf x}-{\mathbf y}\|^a}$ for $a\in (0,1)$}
Let $k({\mathbf x}) = e^{-\gamma\|{\mathbf x}\|^a}$.
In~\cite{bams/1183509728} the following integral representation was derived
\begin{equation*}
\begin{split}
e^{-r^a} = \int_0^\infty e^{-r t}  g_a(t)  dt,
\end{split}
\end{equation*}
where $r>0$, $0<a<1$ and $g_a$ is a probability density function of a distribution supported in $[0,+\infty)$ (sometimes called the L{\'e}vy stable distribution) which has the following representation:
\begin{equation*}
\begin{split}
 g_a(t)  =
\frac{1}{\pi}\sum_{n=1}^\infty\frac{-\sin(n(a+1)\pi)}{n!}\left(\frac{1}{x}\right)^{a n+1}\Gamma(a n+1).
\end{split}
\end{equation*}
We will only need the property $ g_a(t) \asymp \frac{1}{x^{a+1}}$ as $x\to +\infty$.
Since the Fourier transform in $d$ dimensions of $e^{-t\|{\mathbf x}\|}$ is $c_d\frac{t}{(t^2+\|\boldsymbol{\omega}\|^2)^{\frac{d+1}{2}}}$ we conclude 
\begin{equation*}
\begin{split}
&\widehat{k}(\boldsymbol{\omega}) = \mathcal{F}[\int_0^\infty e^{-\gamma^{1/a} \|{\mathbf x}\| t}  g_a(t)  dt] = \int_0^\infty \mathcal{F}[e^{-\gamma^{1/a} \|{\mathbf x}\| t}]  g_a(t)  dt = \\
&\int_0^\infty \frac{c_d\gamma^{1/a} t}{(\gamma^{2/a} t^2 + \|\boldsymbol{\omega}\|^2)^{\frac{d+1}{2}}} g_a(t)  dt.
\end{split}
\end{equation*}
Since $\int_A^{2A} t^a g_a(t) dt \gg \int_A^{2A}\frac{dt}{t}=\log 2\gg 1$, we conclude
\begin{equation*}
\begin{split}
&\|\boldsymbol{\omega}\|^{d+a}\widehat{k}(\boldsymbol{\omega}) = c_d\gamma^{1/a} \int_0^\infty  \frac{ t^{1-a}\|\boldsymbol{\omega}\|^{d+a} }{(\gamma^{2/a} t^2 + \|\boldsymbol{\omega}\|^2)^{\frac{d+1}{2}}} t^a g_a(t) dt\geq\\
&c_d\gamma^{1/a} \int_{\|\boldsymbol{\omega}\|/\gamma^{1/a}}^{2\|\boldsymbol{\omega}\|/\gamma^{1/a}}  \frac{ t^{1-a}\|\boldsymbol{\omega}\|^{d+a} }{(\gamma^{2/a} t^2 + \|\boldsymbol{\omega}\|^2)^{\frac{d+1}{2}}} t^a g_a(t) dt\geq \\
&c_d\gamma^{1/a}\min_{x\in [\|\boldsymbol{\omega}\|/\gamma^{1/a},2\|\boldsymbol{\omega}\|/\gamma^{1/a}]}\frac{ x^{1-a}\|\boldsymbol{\omega}\|^{d+a} }{(\gamma^{2/a} x^2 + \|\boldsymbol{\omega}\|^2)^{\frac{d+1}{2}}} \int_{\|\boldsymbol{\omega}\|/\gamma^{1/a}}^{2\|\boldsymbol{\omega}\|/\gamma^{1/a}}   t^a g_a(t) dt\gg 1.
\end{split}
\end{equation*}
Thus, $\liminf_{\|\boldsymbol{\omega}\|\to +\infty}\|\boldsymbol{\omega}\|^{d+a}\widehat{k}(\boldsymbol{\omega}) >0$. Using Lemma~\ref{translation-invariant} we conclude that for any Riemann measurable $\Omega\subseteq {\mathbb R}^d$ and a uniform distribution $\mu$ on $\Omega$, we have $\lambda_n ({\rm O}_{K, \mu})\gg n^{-1-\frac{a}{d}}$.

Finally, we show that $d_\varrho = \frac{2d}{a}$. Indeed, observe that the kernel-induced metric satisfies $\varrho({\mathbf x}, {\mathbf y}) = \sqrt{2k({\mathbf 0}) - 2k({\mathbf x} - {\mathbf y})} \asymp \|{\mathbf x} - {\mathbf y}\|^{\frac{a}{2}}$.
This implies that the entropy number $\varepsilon(n)$ with respect to $\varrho$ scales as $\varepsilon(n) \asymp \varepsilon'(n)^{\frac{a}{2}}$, where $\varepsilon'(n)$ denotes the entropy number under the Euclidean metric. Since the Minkowski dimension of any Riemann measurable set $\Omega \subseteq \mathbb{R}^d$ with nonzero volume is $d$, we conclude that $d_\varrho = \frac{2d}{a}$. Thus, we have $\lambda_n ({\rm O}_{K, \mu})\gg n^{-1-\frac{2}{d_\varrho}}$. Lemma~\ref{isma-simple} gives us $d_\varrho = d_K$.

\subsection{Proof of Theorem~\ref{width-cases}: the Matérn kernel case}
The Matérn kernel is defined as $k_{\rm M}(\|{\mathbf x}-{\mathbf y}\|)$ where
\begin{equation*}
\begin{split}
k_{\rm M}(r)  = \frac{2^{1-\nu}}{\Gamma(\nu)}(\frac{\sqrt{2\nu}r}{l})^\nu K_\nu (\frac{\sqrt{2\nu}r}{l}) ,
\end{split}
\end{equation*}
with positive parameters $\nu$ and $l$, where $K_\nu$ is a modified Bessel function~\cite{10.7551/mitpress/3206.001.0001}.
Its Fourier transform in $d$ dimensions is
\begin{equation*}
\begin{split}
\widehat{k_{\rm M}}(\boldsymbol{\omega})  = \frac{2^d \pi^{\frac{d}{2}}\Gamma(\nu+\frac{d}{2})(2\nu)^\nu}{\Gamma(\nu)l^{2\nu}}\left(\frac{2\nu}{l^2}+4\pi \|\boldsymbol{\omega}\|^2\right)^{-(\nu+\frac{d}{2})}.
\end{split}
\end{equation*}

By construction, $\liminf_{\|\boldsymbol{\omega}\|\to +\infty}\|\boldsymbol{\omega}\|^{d+2\nu}\widehat{k}(\boldsymbol{\omega})> 0$ if $\nu>0$. From Lemma~\ref{translation-invariant} we conclude that if $\nu >0$, then $\lambda_n ({\rm O}_{K, \mu})\gg n^{-1-\frac{2\nu}{d}}$ for a uniform distribution $\mu$ on a Riemann measurable domain with non-zero volume $\Omega\subseteq {\mathbb R}^d$.

For $\nu\in (0,1)$, the modified Bessel function satisfies $K_\nu (x) =  \frac{\Gamma(\nu)}{2}(\frac{2}{x})^\nu-C_2 x^{\nu} +o(x^{\nu})$ for small $x>0$. Therefore, $\varrho({\mathbf x},{\mathbf y})= \sqrt{2k_{\rm M}({\mathbf 0})-2k_{\rm M}({\mathbf x}-{\mathbf y})}\asymp \|{\mathbf x}-{\mathbf y}\|^\nu$. As in the case of exponential type kernels, we have $\varepsilon(n)\asymp \varepsilon'(n)^\nu$, where $\varepsilon'(n)$ is the entropy number w.r.t. to the euclidean metrics. Thus, for any Riemann measurable with non-zero volume $\Omega\subseteq {\mathbb R}^d$, we have $d_{\varrho} = \frac{d}{\nu}$. This gives us $\lambda_n ({\rm O}_{K, \mu})\gg n^{-1-\frac{2}{d_\varrho}}$ and, using Lemma~\ref{isma-simple}, we obtain $d_\varrho = d_K$. 


\subsection{Remarks on Table~\ref{rates}}
Table~\ref{rates} presents the Kolmogorov $n$-widths for zonal kernels defined on the domain $\Omega = \mathbb{S}^{d-1}$, the unit sphere in $\mathbb{R}^d$. Let $\mu_{\mathbb{S}^{d-1}}$ denote the rotation-invariant probability measure on $\mathbb{S}^{d-1}$. Zonal kernels on the sphere admit the following expansion:
\begin{equation*}
\begin{split}
K(x, y) = \sum_{l=0}^\infty a(l)\sum_{i=1}^{\mathcal{N}(d,l)}Y_{l,i}(x)Y_{l,i}(y),
\end{split}
\end{equation*}
where $\{a(l)\}$ are the eigenvalues of the integral operator $\mathcal{O}_{\mathbb{S}^{d-1}, \mu_{\mathbb{S}^{d-1}}}$, and $\{Y_{l,i}\}_{i=1}^{\mathcal{N}(d,l)}$ is an orthonormal basis of the space of real-valued spherical harmonics of degree $l$. The integer $\mathcal{N}(d,l)$ denotes the dimension of that space.

Using the identity
$$
\sum_{i=1}^{\mathcal{N}(d, l)} Y_{l,i}(x)^2 = \mathcal{N}(d, l),
$$
and applying Ibragimov's theorem, we obtain the following two-sided estimate:
$$
\sqrt{\sum_{l = l_0 + 1}^\infty a(l) \mathcal{N}(d, l)} \leq w_K(n) \leq \sqrt{\sum_{l = l_0 + 1}^\infty a(l) \mathcal{N}(d, l)},
$$
where $n = \sum_{l = 0}^{l_0} \mathcal{N}(d, l)$. Hence, we have the explicit expression:
$$
w_K(n) = \sqrt{\sum_{l = l_0 + 1}^\infty a(l) \mathcal{N}(d, l)}.
$$
This characterization allows for precise asymptotic estimates of $w_K(n)$ for zonal kernels, provided the decay rate of the eigenvalues $\{a(l)\}$ is known.

For the kernel $K(x,y)=e^{-\gamma\|x-y\|}$ and for the NTK-ReLU, it was shown in~\cite{NEURIPS2020_1006ff12} that $a(l)\asymp l^{-d}$ as $l\to +\infty$. Using $\mathcal{N}(d,l)\asymp l^{d-2}$, we conclude that $w_K(\sum_{l=0}^{l_0}\mathcal{N}(d,l))\asymp l_0^{-1/2}$ as $l_0\to +\infty$. Since $n=\sum_{l=0}^{l_0}\mathcal{N}(d,l)\asymp l_0^{d-1}$, we conclude $w_K(n)\asymp n^{-\frac{1}{2d-2}}$. Therefore, $d_K = 2d-2$. The Minkowski dimension of  ${\mathbb S}^{d-1}$ in ${\mathbb R}^{d}$ equals $d-1$. Since $\varrho(x,y)\asymp \|x-y\|^{\frac{1}{2}}$, we conclude $d_\varrho = \frac{d-1}{1/2}$, i.e. $d_\varrho =  d_K$.

For the kernel $K(x,y)=e^{-\frac{\|x-y\|^2}{\sigma^2}}$, it was shown in~\cite{10.1007/11776420_14} that $a(l)\asymp (\frac{2e}{\sigma^2})^{l}(2l+d-2)^{-(l+\frac{d-1}{2})}$ if $\sigma > (\frac{2}{d})^{\frac{1}{2}}$. Analogously to the previous kernel, this gives $w_K(n) \asymp \sqrt{\int_{l_0}^\infty (\frac{2e}{\sigma^2})^x \frac{x^{d-2}dx}{(2x+d-2)^{x+\frac{d-1}{2}}}}$ where $n = \sum_{l = 0}^{l_0} \mathcal{N}(d, l)$. That is, $w_K(n) \ll \sqrt{\int_{\mathcal{O}(n^{1/(d-1)})}^\infty (\frac{2e}{\sigma^2})^x \frac{x^{d-2}dx}{(2x+d-2)^{x+\frac{d-1}{2}}}}$. The latter expression decays exponentially fast. Therefore, $d_K=0$.

\subsubsection{Neural Network Gaussian Process kernels}

Given the activation function $\sigma: {\mathbb R}\to {\mathbb R}$, the Neural Network Gaussian Process (NNGP) kernel without the bias term (in the large width limit), denoted $k_{\rm NNGP,\sigma}: {\mathbb R}^d\times {\mathbb R}^d\to {\mathbb R}$, is defined by 
\begin{equation*}
\begin{split}
k_{\rm NNGP,\sigma}(x, y) = {\mathbb E}_{w\sim \mathcal{N}({\mathbf 0},I_d)}[\sigma(w^\top x)\sigma(w^\top y)],
\end{split}
\end{equation*}
where $\mathcal{N}({\mathbf 0},I_d)$ denotes the standard Gaussian distribution in $d$ dimensions. The NNGP kernel, restricted to ${\mathbb S}^{d-1}$, is zonal.

{\bf The case of smooth activations.} If $\sigma$ is infinitely differentiable and $\max_{x\in {\mathbb R}}|\sigma^{(k)}(x)|\ll k!$,  \cite{10.5555/3586589.3586708} proved that
\begin{equation*}
\begin{split}
\sum_{l=l_0+1}^\infty a(l)\mathcal{N}(d,l)\ll \frac{1}{n},
\end{split}
\end{equation*}
where $n = \sum_{l = 0}^{l_0} \mathcal{N}(d, l)$.
Therefore,
\begin{equation*}
\begin{split}
w_K(n)\ll n^{-\frac{1}{2}},
\end{split}
\end{equation*}
and $d_K\leq 2$.

As the following lemma demonstrates, for activation functions with a bounded third derivative, the upper metric dimension $d_\varrho$ is typically either $d - 1$ or $\frac{d - 1}{2}$, depending on the behavior of certain expectations:

\begin{lemma}
Let $\sup_{x \in \mathbb{R}} |\sigma'''(x)| < \infty$. Then:
\begin{itemize}
\item If $
  \mathbb{E}_{w_1 \sim \mathcal{N}(0,1)}\left[\sigma(w_1)\sigma'(w_1) w_1 - \sigma''(w_1)\right] \ne 0$,  it follows that $d_\varrho = d - 1$.
\item If $\mathbb{E}_{w_1 \sim \mathcal{N}(0,1)}\left[\sigma(w_1)\sigma'(w_1) w_1 - \sigma''(w_1)\right] = 0$
 and $\mathbb{E}_{w_1 \sim \mathcal{N}(0,1)}\left[\sigma''(w_1)w_1^2 - \sigma''(w_1)\right] \ne 0$,
  then $d_\varrho = \frac{d - 1}{2}$.
\end{itemize}
\end{lemma}

\begin{proof}
Let $x, y \in \mathbb{S}^{d-1}$ be such that $y = x + \delta x$. By the rotational invariance of the kernel $k_{\rm NNGP,\sigma}(x, y)$, we may assume without loss of generality that $x = e_1 = (1, 0, \ldots, 0)$. Then the perturbation $\delta x = (\delta x_1, \ldots, \delta x_d)$ satisfies:

$$
\delta x_1 = x^\top y - 1, \quad \sum_{i \ne 1} \delta x_i^2 = 2 - 2 x^\top y - (x^\top y - 1)^2 = 1 - (x^\top y)^2.
$$

Define the following constants:

$$
\begin{aligned}
C_0 &= \mathbb{E}_{w_1 \sim \mathcal{N}(0,1)}[\sigma(w_1)^2], \\
C_1 &= \mathbb{E}_{w_1 \sim \mathcal{N}(0,1)}[\sigma(w_1) \sigma'(w_1) w_1], \\
C_2 &= \mathbb{E}_{w_1 \sim \mathcal{N}(0,1)}[\sigma''(w_1) w_1^2], \\
D_2 &= \mathbb{E}_{w_1 \sim \mathcal{N}(0,1)}[\sigma''(w_1)].
\end{aligned}
$$

These will be used to determine the local behavior of the kernel near the diagonal, which governs the scaling of the induced metric.
 We have
\begin{equation*}
\begin{split}
&k_{\rm NNGP,\sigma}(x, y) = {\mathbb E}_{w\sim \mathcal{N}({\mathbf 0},I_d)}[\sigma(w^\top x)\sigma(w^\top (x+\delta x))] = \\
&{\mathbb E}_{w\sim \mathcal{N}({\mathbf 0},I_d)}[\sigma(w_1)\sigma(w_1+w^\top\delta x)] = \\
& {\mathbb E}_{w\sim \mathcal{N}({\mathbf 0},I_d)}[\sigma(w_1)(\sigma(w_1)+ \sigma '(w_1)w^\top\delta x+\frac{1}{2}\sigma ''(w_1)(w^\top\delta x)^2)]+o(\|\delta x\|^2) = \\
&{\mathbb E}_{w_1\sim \mathcal{N}(0,1)}[\sigma(w_1)^2]+{\mathbb E}_{w_1\sim \mathcal{N}(0,1)}[\sigma(w_1)\sigma '(w_1)w_1]\delta x_1+\\
&\frac{1}{2}{\mathbb E}_{w_1\sim \mathcal{N}(0,1)}[\sigma ''(w_1)w_1^2]\delta x_1^2+\frac{1}{2}\sum_{i\ne 1}{\mathbb E}_{w_1\sim \mathcal{N}(0,1)}[\sigma ''(w_1)]\delta x_i^2+o(\|\delta x\|^2).
\end{split}
\end{equation*}
Thus, 
\begin{equation*}
\begin{split}
k_{\rm NNGP,\sigma}(x, y) = C_0+C_1(x^\top y-1)+\frac{1}{2}C_2 (x^\top y-1)^2+\frac{1}{2}D_2 (1-(x^\top y)^2)+o(\|\delta x\|^2) .
\end{split}
\end{equation*}
Note that $1-(x^\top y)^2 = 2(1-(x^\top y))-(1-(x^\top y))^2$.
If $C_1-D_2\ne 0$, we have
\begin{equation*}
\begin{split}
k_{\rm NNGP,\sigma}(x, y) = C_0+(C_1-D_2)(x^\top y-1)+o(\|\delta x\|),
\end{split}
\end{equation*}
and 
\begin{equation*}
\begin{split}
\varrho(x,y)\asymp \sqrt{1-x^\top y}\asymp \|x-y\|.
\end{split}
\end{equation*}
Thus, $d_\varrho=d-1$.

If $C_1-D_2=0$, then
\begin{equation*}
\begin{split}
k_{\rm NNGP,\sigma}(x, y) = C_0+\frac{1}{2}(C_2-D_2) (x^\top y-1)^2+o(\|\delta x\|^2) .
\end{split}
\end{equation*}
When $C_2-D_2\ne 0$, we have
\begin{equation*}
\begin{split}
\varrho(x,y)\asymp 1-x^\top y\asymp \|x-y\|^2.
\end{split}
\end{equation*}
Thus, $d_\varrho=\frac{d-1}{2}$.
\end{proof}
\begin{example} Let us consider the case of $\sigma(x)=\cos(x)$. We have 
\begin{equation*}
\begin{split}
&{\mathbb E}_{w_1\sim \mathcal{N}(0,1)}[\sigma(w_1)\sigma '(w_1) w_1 -\sigma ''(w_1)]=\\
&\frac{1}{\sqrt{2\pi}}\int_{-\infty}^{\infty}(-\frac{\sin(2x)x}{2}+\cos(x))e^{-\frac{x^2}{2}}dx\approx 0.47\ne 0.
\end{split}
\end{equation*}
Therefore, $d_\varrho = d-1$. Analogously, for $\sigma(x)=\cos(x)$, we have
\begin{equation*}
\begin{split}
&{\mathbb E}_{w_1\sim \mathcal{N}(0,1)}[\sigma(w_1)\sigma '(w_1) w_1 -\sigma ''(w_1)]=\\
&\frac{1}{\sqrt{2\pi}}\int_{-\infty}^{\infty}(\frac{\sin(2x)x}{2}+\sin(x))e^{-\frac{x^2}{2}}dx\approx 0.13\ne 0,
\end{split}
\end{equation*}
and $d_\varrho = d-1$.
\end{example}

{\bf The case of $\sigma(x) = \max(0,x)^\alpha, \alpha\geq 0$. } For $\alpha\in \{0,1\}$ it was shown in~\cite{JMLR:v18:14-546} that $a(l)\asymp l^{-d-2\alpha}$. Therefore, $w_K(\sum_{l=0}^{l_0}\mathcal{N}(d,l))\asymp l_0^{-1/2-\alpha}$ as $l_0\to +\infty$ and $w_K(n)\asymp n^{-\frac{1+2\alpha}{2d-2}}$. Thus, $d_K = \frac{2d-2}{1+2\alpha}$. 

\cite{10.5555/3586589.3586708} proved that $w_K(n)\gg n^{-\frac{1+2\alpha}{2d-2}}$ for any $\alpha>0$ and gave some experimental evidence in favour of $w_K(n)\asymp n^{-\frac{1+2\alpha}{2d-2}}$. From this, we can claim $d_K\geq \frac{2d-2}{1+2\alpha}$ for any $\alpha\geq 0$.

For $\alpha=1$, we have $k_{\rm NNGP,\sigma}(x, y)\propto u(\pi -\arccos(u))+\sqrt{1-u^2}$ where $u=x^\top y$. Let $v=1-u=\frac{1}{2}\|x-y\|^2$. Thus, $k_{\rm NNGP,\sigma}(x, y)\propto (1-v) (\pi -\arccos(1-v))+v^{\frac{1}{2}}(2-v)^{\frac{1}{2}}=\pi -\pi v+o(v)$. Therefore, $\varrho(x,y)\asymp \|x-y\|$ and $d_\varrho=d-1$.

For $\alpha=0$, we have $k_{\rm NNGP,\sigma}(x, y)\propto \pi -\arccos(u)$ where $u=x^\top y$. Analogously, we show that $k_{\rm NNGP,\sigma}(x, y)\propto \pi -\sqrt{2v}+o(\sqrt{v})$. Therefore, $\varrho(x,y)\asymp \sqrt{\|x-y\|}$ and $d_\varrho=2(d-1)$.

\end{document}